\documentclass[journal]{IEEEtran}
\usepackage{comment}
\usepackage{cite}
\usepackage{amsmath,amssymb,amsfonts,mathrsfs}
\usepackage{algorithm,algorithmic,setspace}
\usepackage{graphicx}
\usepackage{textcomp}
\usepackage{xcolor}
\usepackage{comment}
\usepackage{mathtools}
\usepackage{dsfont}
\usepackage{subfig}
\usepackage{algorithm}
\usepackage{algorithmic}

\usepackage{epstopdf} 
\usepackage{graphics} 
\usepackage{subfig}
\usepackage{epstopdf} 
\usepackage{url}
\usepackage{color,soul}
\usepackage{amsthm}
\usepackage[colorlinks=true,linkcolor=blue, citecolor=green, urlcolor=magenta]{hyperref}

\DeclarePairedDelimiter{\norm}{\lVert}{\rVert}

\theoremstyle{definition}

\newtheorem{theorem}{Theorem}
\newtheorem{assumption}{Assumption}

\newtheorem{proposition}{Proposition}
\newtheorem{lemma}{Lemma}
\newtheorem{remark}{Remark}

\allowdisplaybreaks

\ifCLASSINFOpdf
\else
\fi

\title{
	Hybrid Feedback Control Design for Non-Convex Obstacle Avoidance 
} 

\author{Mayur Sawant, Ilia Polushin, and Abdelhamid Tayebi 
	\thanks{This work was supported by the National Sciences and Engineering Research Council of Canada (NSERC), under the grants RGPIN-2020-06270 and RGPIN-2020-0644.}
	\thanks{M. Sawant, I. Polushin and A. Tayebi are with the Department of Electrical and Computer Engineering, Western University, London, ON N6A 3K7, Canada. (e-mail: {\tt\small msawant2, ipolushi, atayebi@uwo.ca}).  A. Tayebi is also with the Department of Electrical Engineering, Lakehead University, Thunder Bay, ON P7B 5E1, Canada. (e-mail: {\tt\small atayebi@lakeheadu.ca}).}%
}

\begin{document}

\maketitle
\thispagestyle{empty}

\begin{abstract}
    We develop an autonomous navigation algorithm for a robot operating in two-dimensional environments containing obstacles, with arbitrary non-convex shapes, which can be in close proximity with each other, as long as there exists at least one safe path connecting the initial and the target location. An instrumental transformation that modifies (virtually) the non-convex obstacles, in a non-conservative manner, is introduced to facilitate the design of the \textit{obstacle-avoidance} strategy. The proposed navigation approach relies on a hybrid feedback that guarantees global asymptotic stabilization of a target location while ensuring the forward invariance of the modified obstacle-free workspace. The proposed hybrid feedback controller guarantees Zeno-free switching between the \textit{move-to-target} mode and the \textit{obstacle-avoidance} mode based on the proximity of the robot with respect to the modified obstacle-occupied workspace. Finally, we provide an algorithmic procedure for the sensor-based implementation of the proposed hybrid controller and validate its effectiveness via some numerical simulations.
\end{abstract}

\section{Introduction} 
Autonomous navigation, the problem of designing control strategies to guide a robot to its goal while avoiding obstacles, is one of the fundamental problems in robotics. One of the widely explored techniques in that regard is the artificial potential fields (APF) \cite{khatib1986real} in which interplay between an attractive field and a repulsive field allows the robot, in most of the cases, to safely navigate towards the target. However, for certain inter-obstacle arrangements, this approach is hampered by the existence of undesired local minima. The navigation function (NF) based approach \cite{koditschek1990robot}, which is directly applicable to sphere world environments \cite{koditschek1992exact,verginis2021adaptive}, mitigates this issue by restricting the influence of the repulsive field within a local neighbourhood of the obstacle by means of a properly tuned parameter, and ensures almost\footnote{Almost global convergence here refers to the convergence from all initial conditions except a set of zero Lebesgue measure.} global convergence of the robot towards the target location. 
To extend the applicability of the NF approach to environments containing more general convex and star-shaped obstacles, one can employ the diffeomorphic mappings developed in \cite{koditschek1992exact} and \cite{li2018navigation}. However, the application of these diffeomorphic mappings requires a global knowledge of the environment.

The NF-based approach was extended in \cite{paternain2017navigation} to environments containing curved obstacles. The authors established sufficient conditions on the eccentricity of the obstacles to guarantee almost global convergence to a neighborhood surrounding the \textit{a priori} unknown target location. This approach is applicable to convex obstacles with smooth and sufficiently curved boundaries.
In \cite{loizou2011navigation}, the authors presented a methodology for the design of a harmonic potential-based NF that ensures almost global convergence to the target location in \textit{a priori} known environments that are diffeomorphic to the point world. This work was subsequently extended in \cite{loizou2021correct} to unknown environments using a sensor-based approach. However, similar to the approach in \cite{paternain2017navigation}, it is assumed that the shapes of the obstacles become known when the robot visits their respective neighborhoods.


In \cite{berkane2021navigation}, the authors proposed a feedback controller  based on Nagumo's theorem \cite[Theorem 4.7]{blanchini2008set}, for autonomous navigation in environments with general convex obstacles. The forward invariance of the obstacle-free space is ensured by projecting the \textit{ideal} velocity control vector (pointing towards the target)  onto the tangent cone at the boundary of the obstacle whenever it points towards the obstacle. In \cite{reis2020control}, a control barrier function-based approach was used for robot navigation in an environment with a single spherical obstacle. It was shown that this approach does not guarantee global convergence to the target location due to the existence of an undesired equilibrium.

In \cite{arslan2019sensor}, a reactive power diagram-based approach was introduced for robots navigating in \textit{a priori} unknown environments. This approach guarantees almost global asymptotic stabilization of the target location, provided that the obstacles are sufficiently separated and strongly convex.
This approach was further extended in \cite{vasilopoulos2018reactive} to handle partially known non-convex environments, where it is assumed that the robot possesses geometrical information about the non-convex obstacles but lacks knowledge of their precise locations within the workspace. 
However, due to the topological obstruction induced by the motion space in the presence of obstacles for any continuous time-invariant vector fields \cite{koditschek1990robot}, the above-mentioned approaches provide at best almost global convergence guarantees.

In \cite{vrohidis2018prescribed}, a time-varying vector field planner was proposed for the navigation of a single robot in a sphere world. This planner leverages prescribed performance control techniques to achieve predetermined convergence to a neighborhood of the target location from any initial position, while ensuring obstacle avoidance. 
In \cite{sanfelice2006robust}, hybrid control techniques were employed to achieve robust global regulation to a target while avoiding a single spherical obstacle. This approach was extended in \cite{poveda2021robust} to the problem of steering a group of planar robots to a neighborhood of an unknown source (emitting a signal with measurable intensity), while avoiding a single obstacle.
In \cite{braun2020explicit}, the authors proposed a hybrid control law to globally asymptotically stabilize a class of linear systems while avoiding neighbourhoods of unsafe points.
In \cite{matveev2011method} and \cite{berkane2021obstacle}, hybrid control techniques were employed to enable the robot to operate in the \textit{obstacle-avoidance} mode when in close proximity to an obstacle or in the \textit{move-to-target} mode when located away from obstacles. These strategies bear resemblance to bug algorithms \cite{lumelsky1986dynamic}, which are commonly used for point robot path planning.
In \cite[Definition 2]{berkane2021obstacle}, the proposed hybrid controller is applicable for known $n-$dimensional environments with sufficiently disjoint elliptical obstacles. On the other hand, in \cite[Assumption 10]{matveev2011method}, the obstacles are assumed to be smooth and sufficiently separated from each other. In \cite{loizou2003closed}, the authors proposed a discontinuous feedback control law for autonomous robot navigation in partially known two-dimensional environments. When a known obstacle is encountered, the control vector aligns with the negative gradient of the Navigation Function (NF). However, when close to an unknown obstacle, the robot moves along its boundary, relying on the local curvature information of the obstacle. This method is limited to point robots and, similar to \cite{matveev2011method}, assumes smooth obstacle boundaries without sharp edges. In our earlier work \cite{sawant2021hybrid}, we proposed a hybrid feedback controller design to address the problem of autonomous robot navigation in planar environments with arbitrarily shaped convex obstacles.

In the present work, which has been initiated in our preliminary conference paper \cite{cdc_sawant}, we consider the autonomous robot navigation problem in a two-dimensional space with arbitrarily-shaped non-convex obstacles which can be in close proximity with each other. Unlike \cite{arslan2019sensor}, \cite{berkane2021obstacle} and \cite{sawant2021hybrid}, wherein the robot is allowed to pass between any pair of obstacles, we require the existence of a safe path joining the initial and the target location, as stated in Assumptions \ref{assumption:connected_interior} and \ref{Assumption:reach}. The main contributions of the present paper are as follows:
\begin{enumerate}
\item[1)]\textit{Asymptotic stability:} the proposed autonomous navigation solution ensures asymptotic stability of the target location for the robot operating in planar environments with arbitrary non-convex obstacles.
\item[2)]\textit{Arbitrarily-shaped obstacles:} there are no restrictions on the shape of the non-convex obstacles such as those mentioned in \cite{paternain2017navigation}, \cite{arslan2019sensor}, \cite{berkane2021obstacle}, and their proximity with respect to each other \textit{e.g.}, see \cite[Assumption 10]{matveev2011method}, \cite[Definition 2]{berkane2021obstacle}, except for the mild feasibility Assumptions \ref{assumption:connected_interior} and \ref{Assumption:reach}.
\item[3)]\textit{Applicable in \textit{a priori} unknown environments:} the proposed obstacle avoidance approach can be implemented using only range scanners (\textit{e.g.}, LiDAR) without an \textit{a priori} global knowledge of the environment which satisfies Assumptions \ref{assumption:connected_interior} and \ref{Assumption:reach}.
\end{enumerate}


The remainder of the paper is organized as follows. In Section \ref{notation_prelim}, we provide the notations and some preliminaries that will be used throughout the paper. The problem is formulated in Section \ref{Section:problem_formulation}. The design of the obstacle reshaping operator is provided in Section \ref{section:obstacle_reshaping}, and the proposed hybrid control algorithm is presented in Section \ref{sec:the_proposed_contrroller}. The stability and safety guarantees of the proposed navigation control scheme are provided in Section \ref{sec:stabilitY_analysis}. A sensor-based implementation of the proposed obstacle avoidance algorithm, using 2D range scanners (LiDAR), is given in Section \ref{section:sensor-based_implementation}. Simulation results are presented in Section \ref{sec:simulation_results}, and some final concluding remarks are given in Section \ref{sec:conclusion}.

\section{Notations and Preliminaries}\label{notation_prelim}
\subsection{Notations}
The sets of real and natural numbers are denoted by $\mathbb{R}$ and $\mathbb{N}$, respectively. We identify vectors using bold lowercase letters. The Euclidean norm of a vector $\mathbf{p}\in\mathbb{R}^n$ is denoted by $\norm{\mathbf{p}}$, and an Euclidean ball of radius $r>0$ centered at $\mathbf{p}$ is represented by $\mathcal{B}_{r}(\mathbf{p}) = \{\mathbf{q}\in\mathbb{R}^n|\norm{\mathbf{q} - \mathbf{p}} \leq r\}.$ Given two vectors $\mathbf{p}\in\mathbb{R}^2$ and $\mathbf{q}\in\mathbb{R}^2$, we denote by $\measuredangle(\mathbf{p}, \mathbf{q})$ the angle from $\mathbf{p}$ to $\mathbf{q}$. The angle measured in counter-clockwise manner is considered positive, and vice versa. Given two locations $\mathbf{p}, \mathbf{q}\in\mathbb{R}^n$, the notation $\mathbf{path}(\mathbf{p}, \mathbf{q})$ represents a continuous path in $\mathbb{R}^n$, which joins these locations. A set $\mathcal{A}\subset\mathbb{R}^n$ is said to be pathwise connected if for any two points $\mathbf{p},\mathbf{q}\in\mathcal{A}$, there exists a continuous path, joining $\mathbf{p}$ and $\mathbf{q}$, that belongs to the same set \textit{i.e.}, there exists a $\mathbf{path}(\mathbf{p}, \mathbf{q})\subset\mathcal{A}$ \cite[Definition 27.1]{willard2012general}. For two sets $\mathcal{A}, \mathcal{B}\subset\mathbb{R}^n$, the relative complement of $\mathcal{B}$ with respect to $\mathcal{A}$ is denoted by $\mathcal{A}\setminus\mathcal{B} =\{\mathbf{a}\in\mathcal{A}|\mathbf{a}\notin \mathcal{B}\}$. Given a set $\mathcal{A}\subset\mathbb{R}^n$, the symbols $\partial\mathcal{A}, \mathcal{A}^{\circ}$, $\mathcal{A}^c$ and $\bar{\mathcal{A}}$ represent the boundary, interior, complement and the closure of the set $\mathcal{A}$, respectively, where $\partial\mathcal{A} = \bar{\mathcal{A}}\backslash\mathcal{A}^{\circ}$. The number of elements in a set $\mathcal{A}$ is given by $\mathbf{card}(\mathcal{A})$. 
Let $\mathcal{A}$ and $\mathcal{B}$ be subsets of $\mathbb{R}^n$, then the dilation of $\mathcal{A}$ by $\mathcal{B}$ is denoted by $\mathcal{A}\oplus\mathcal{B}= \{ \mathbf{a} + \mathbf{b}|\mathbf{a}\in\mathcal{A}, \mathbf{b}\in\mathcal{B}\},$ and the erosion of $\mathcal{A}$ by $\mathcal{B}$ is denoted by $\mathcal{A}\ominus\mathcal{B} = \{\mathbf{x}\in\mathbb{R}^n|\mathbf{x} + \mathbf{b}\in\mathcal{A}, \forall\mathbf{b}\in\mathcal{B}\}$ \cite{haralick1987image}. Additionally, the set $\mathcal{B}$ is referred to as a structuring element.
{Given a positive scalar $r> 0$, the $r-$dilated version of the set $\mathcal{A}\subset\mathbb{R}^n$ is denoted by $\mathcal{D}_{r}(\mathcal{A}) =\mathcal{A}\oplus\mathcal{B}_{r}(\mathbf{0})$.
The $r-$neighbourhood of the set $\mathcal{A}$ is denoted by $\mathcal{N}_{r}(\mathcal{A}) = \mathcal{D}_{r}(\mathcal{A})\setminus\mathcal{A}^{\circ}$.

\subsection{Projection on a set}\label{section:metric_projection} Given a closed set $\mathcal{A}\subset\mathbb{R}^n$ and a point $\mathbf{x}\in\mathbb{R}^n$, the Euclidean distance of $\mathbf{x}$ from the set $\mathcal{A}$ is evaluated as
\begin{equation}
    d(\mathbf{x}, \mathcal{A}) = \underset{\mathbf{q}\in\mathcal{A}}{\min}\norm{\mathbf{x} - \mathbf{q}}.\label{definition:distance_from_set}
\end{equation}
The set $\mathcal{PJ}(\mathbf{x}, \mathcal{A})\subset\mathcal{A}$, which is defined as
\begin{equation}
    \mathcal{PJ}(\mathbf{x}, \mathcal{A}) = \{\mathbf{q}\in\mathcal{A}|\norm{\mathbf{x} - \mathbf{q}} = d(\mathbf{x}, \mathcal{A})\},\label{definition:set_of_projections}
\end{equation}
is the set of points in $\mathcal{A}$ that are at the distance of $d(\mathbf{x}, \mathcal{A})$ from $\mathbf{x}$. If $\mathbf{card}(\mathcal{PJ}(\mathbf{x}, \mathcal{A}))$ is one, then the element of the set $\mathcal{PJ}(\mathbf{x}, \mathcal{A})$ is denoted by $\Pi(\mathbf{x}, \mathcal{A})$. 

\subsection{Sets with positive reach}
\label{section:set_with_positive_reach}Given a closed set $\mathcal{A}\subset\mathbb{R}^n$, the set $\text{Unp}(\mathcal{A})$, which is defined as
\begin{equation}
\text{Unp}(\mathcal{A}) = \{\mathbf{x}\in\mathbb{R}^n|\mathbf{card}(\mathcal{PJ}(\mathbf{x}, \mathcal{A})) = 1\},
\end{equation}
denotes the set of all $\mathbf{x}\in\mathbb{R}^n$ for which there exists a unique point in $\mathcal{A}$ nearest to $\mathbf{x}$. Then, for any $\mathbf{x}\in\mathcal{A}$, the reach of set $\mathcal{A}$ at $\mathbf{x}$, denoted by $\textbf{r}(\mathcal{A}, \mathbf{x})$ \cite[Pg 55]{rataj2019curvature}, is defined as
\begin{equation}
\mathbf{r}(\mathcal{A}, \mathbf{x}) := \sup\{r \geq 0|\mathcal{B}_{r}^{\circ}(\mathbf{x})\subset\text{Unp}(\mathcal{A})\}.
\end{equation}
The reach of set $\mathcal{A}$ is then given by
\begin{equation}
\mathbf{reach}(\mathcal{A}) := \underset{\mathbf{x}\in\mathcal{A}}{\inf}\mathbf{ r}(\mathcal{A}, \mathbf{x}).
\end{equation} 
If a closed set has reach greater than or equal to $r > 0$, then any location less than $r$ distance away from the set will have a unique closest point on the set.

\subsection{Geometric sets}
\subsubsection{Line} Let $\mathbf{p}\in\mathbb{R}^n$ and $\mathbf{q}\in\mathbb{R}^n\setminus\{\mathbf{0}\}$, then a line passing through the point $\mathbf{p}$ in the direction of the vector $\mathbf{q}$ is defined as
\begin{equation}
    \mathcal{L}(\mathbf{p}, \mathbf{q}) := \left\{\mathbf{x}\in\mathbb{R}^n|\mathbf{x} = \mathbf{p} + \lambda\mathbf{q}, \lambda\in\mathbb{R}\right\}.
\end{equation}
\subsubsection{Line segment} Let $\mathbf{p}\in\mathbb{R}^n$ and $\mathbf{q}\in\mathbb{R}^n$, then a line segment joining $\mathbf{p}$ and $\mathbf{q}$ is denoted by
\begin{equation}
    \mathcal{L}_s(\mathbf{p}, \mathbf{q}):=\left\{\mathbf{x}\in\mathbb{R}^n|\mathbf{x} = \lambda\mathbf{p} + (1 - \lambda)\mathbf{q}, \lambda\in[0, 1]\right\}.
\end{equation}
\subsubsection{Hyperplane} Given $\mathbf{p}\in\mathbb{R}^n$ and $\mathbf{q}\in\mathbb{R}^n\setminus\{\mathbf{0}\}$, a hyperplane passing through $\mathbf{p}$ and orthogonal to $\mathbf{q}$ is given by
\begin{equation}
    \mathcal{P}(\mathbf{p}, \mathbf{q}):= \{\mathbf{x}\in\mathbb{R}^n|\mathbf{q}^\intercal(\mathbf{x} - \mathbf{p}) = 0\}.\label{definition:hyperplane}
\end{equation}
The hyperplane divides the Euclidean space $\mathbb{R}^n$ into two half-spaces \textit{i.e.}, a closed positive half-space $\mathcal{P}_{\geq}(\mathbf{p}, \mathbf{q})$ and a closed negative half-space $\mathcal{P}_{\leq}(\mathbf{p}, \mathbf{q})$ which are obtained by substituting `$=$' with `$\geq$' and `$\leq$' respectively, in the right-hand side of \eqref{definition:hyperplane}. We also use the notations $\mathcal{P}_{>}(\mathbf{p}, \mathbf{q})$ and $\mathcal{P}_{<}(\mathbf{p} ,\mathbf{q})$ to denote the open positive and the open negative half-spaces such that $\mathcal{P}_{>}(\mathbf{p}, \mathbf{q}) = \mathcal{P}_{\geq}(\mathbf{p}, \mathbf{q})\backslash\mathcal{P}(\mathbf{p}, \mathbf{q})$ and $\mathcal{P}_{<}(\mathbf{p} ,\mathbf{q})= \mathcal{P}_{\leq}(\mathbf{p}, \mathbf{q})\backslash\mathcal{P}(\mathbf{p}, \mathbf{q})$.
\subsubsection{Convex cone}Given $\mathbf{c}\in\mathbb{R}^n$, $\mathbf{p}_1\in\mathbb{R}^n\setminus\{\mathbf{c}\}$, and $\mathbf{p}_2\in\mathbb{R}^n\setminus\{\mathbf{c}\}$, a convex cone with its vertex at $\mathbf{c}$ and its edges passing through $\mathbf{p}_1$ and $\mathbf{p}_2$ is defined as
\begin{equation}
    \begin{aligned}
    \mathcal{C}(\mathbf{c}, \mathbf{p}_1, \mathbf{p}_2) := \{&\mathbf{x}\in\mathbb{R}^n|\mathbf{x} = \mathbf{c} + \lambda_1(\mathbf{p}_1 - \mathbf{c}) + \lambda_2(\mathbf{p}_2 - \mathbf{c}), \\
    &\forall\lambda_1\geq 0, \forall\lambda_2\geq0\}.
    \end{aligned}
\end{equation}

\subsubsection{Conic hull \cite[Section 2.1.5]{boyd2004convex}}Given a set $\mathcal{A}\subset\mathbb{R}^n$ and a point $\mathbf{x}\in\mathbb{R}^n$, the conic hull $\mathcal{CH}(\mathbf{x}, \mathcal{A})$ for the set $\mathcal{A}$, with its vertex at $\mathbf{x}$ is defined as
\begin{equation}
    \mathcal{CH}(\mathbf{x}, \mathcal{A}) := \bigcup_{\mathbf{p}, \mathbf{q}\in\mathcal{A}}\mathcal{C}(\mathbf{x}, \mathbf{p}, \mathbf{q}).\label{definition:conic_hull}
\end{equation}
The conic hull $\mathcal{CH}(\mathbf{x}, \mathcal{A})$ is the smallest convex cone with its vertex at $\mathbf{x}$ that contains the set $\mathcal{A}$ \textit{i.e.}, $\mathcal{A}\subset\mathcal{CH}(\mathbf{x}, \mathcal{A}).$ 

\subsection{Tangent cone and Normal cone} Given a closed set $\mathcal{A}\subset\mathbb{R}^n$, the tangent cone to $\mathcal{A}$ at a point $\mathbf{x}\in\mathbb{R}^n$ \cite[Def 4.6]{blanchini2008set} is defined by
\begin{equation}
    \mathbf{T}_{\mathcal{A}}(\mathbf{x}) := \left\{\mathbf{w}\in\mathbb{R}^n\bigg|\underset{\tau \rightarrow 0^+}{\text{lim inf }}\frac{d(\mathbf{x} + \tau\mathbf{w}, \mathcal{A})}{\tau} = 0\right\}.
\end{equation}
The tangent cone to the set $\mathcal{A}$ at $\mathbf{x}$ is the set that contains all the vectors whose directions point from $\mathbf{x}$ either inside or tangent to the set $\mathcal{A}$. Given a tangent cone to a set $\mathcal{A}$ at a point $\mathbf{x}$, the normal cone to the set $\mathcal{A}$ at the point $\mathbf{x}$, as defined in \cite[Pg 58]{rataj2019curvature}, is given by
\begin{equation}
    \mathbf{N}_{\mathcal{A}}(\mathbf{x}) := \left\{\mathbf{p}\in\mathbb{R}^n\big|\mathbf{p}^\intercal\mathbf{w} \leq 0, \forall\mathbf{w}\in\mathbf{T}_{\mathcal{A}}(\mathbf{x})\right\}.
\end{equation}
\label{section:normal_and_tangent_cone}

The next two lemmas provide some properties of the sets with positive reach, which will be used in the paper.

\begin{lemma}
Given a closed set $\mathcal{A}\subset\mathbb{R}^n$, we define the set $\mathcal{G} = \left(\mathcal{A} \oplus \mathcal{B}_{\alpha}^{\circ}(\mathbf{0})\right)^c.$ If $\mathbf{reach}(\mathcal{A}) \geq \alpha$, then $\mathbf{reach}(\mathcal{G}) \geq \alpha$. 
\label{lemma:reach_property}
\end{lemma}
\begin{proof}
See Appendix \ref{proof:reach_property}.
\end{proof}


\begin{lemma}
    Consider a closed set $\mathcal{A}\subset\mathbb{R}^n$ and scalars $\alpha\geq \beta\geq 0$. If $\mathbf{reach}(\mathcal{A}) \geq \alpha,$ then $\mathbf{reach}(\mathcal{D}_{\beta}(\mathcal{A})) \geq \alpha - \beta.$\label{lemma:reach_extends}
\end{lemma}

\begin{proof}
See Appendix \ref{proof:reach_extends}.
\end{proof}
  
\subsection{Hybrid system framework}
\label{sec:hybrid_threory}
A hybrid dynamical system
\cite{goebel2012hybrid} is represented using differential and difference inclusions for the state $\mathbf{\xi}\in\mathbb{R}^n$ as follows:
\begin{align}
    \begin{cases}\begin{matrix}\mathbf{\dot{\xi}} \in \mathbf{F}(\mathbf{\xi}) , & \mathbf{\xi} \in \mathcal{F}, \\
    \mathbf{\xi}^{+}\in \mathbf{J}(\mathbf{\xi}), & \mathbf{\xi}\in\mathcal{J},\end{matrix}\end{cases}\label{hybrid_system_general_model}
\end{align}
where the \textit{flow map} $\mathbf{F}:\mathbb{R}^n\rightrightarrows\mathbb{R}^n$ is the differential inclusion which governs the continuous evolution when $\mathbf{\xi}$ belongs to the \textit{flow set} $\mathcal{F}\subseteq\mathbb{R}^n$, where the symbol `$\rightrightarrows$' represents set-valued mapping. The \textit{jump map} $\mathbf{J}:\mathbb{R}^n\rightrightarrows\mathbb{R}^n$ is the difference inclusion that governs the discrete evolution when $\mathbf{\xi}$ belongs to the \textit{jump set} $\mathcal{J}\subseteq\mathbb{R}^n$. The hybrid system \eqref{hybrid_system_general_model} is defined by its data and denoted as $\mathcal{H} = (\mathcal{F}, \mathbf{F}, \mathcal{J}, \mathbf{J}).$

A subset $\mathbb{T}\subset\mathbb{R}_{\geq0}\times\mathbb{N}$ is a \textit{hybrid time domain} if it is a union of a finite or infinite sequence of intervals $[t_j, t_{j + 1}]\times \{j\},$ where the last interval (if existent) is possibly of the form $[t_j, T)$ with $T$ finite or $T = +\infty$. The ordering of points on each hybrid time domain is such that $(t, j)\preceq(t^{\prime}, j^{\prime})$ if $t < t^{\prime},$ or $t = t^{\prime}$ and $j \leq j^{\prime}$. A \textit{hybrid solution} $\phi$ is maximal if it cannot be extended, and complete if its domain dom $\phi$ (which is a hybrid time domain) is unbounded.

\section{Problem Formulation}
\label{Section:problem_formulation}
We consider a disk-shaped robot operating in a two dimensional, compact, arbitrarily-shaped (possibly non-convex) subset of the Euclidean space $\mathcal{W}\subset\mathbb{R}^2$. The workspace $\mathcal{W}$ is cluttered with a finite number of compact, pairwise disjoint obstacles $\mathcal{O}_i\subset\mathcal{W}, i\in\{1, \ldots, b\}, b\in\mathbb{N}$. We define obstacle $\mathcal{O}_0 := (\mathcal{W}^{\circ})^c$ as the complement of the interior of the workspace. The robot is governed by single integrator dynamics
\begin{equation}
    \mathbf{\dot{x}} = \mathbf{u},
\end{equation}
where $\mathbf{x}$ is the location of the center of the robot and $\mathbf{u}\in\mathbb{R}^2$ is the control input.
The task is to reach a predefined obstacle-free target location from any obstacle-free region while avoiding collisions. Without loss of generality, consider the origin $\mathbf{0}$ as the target location. 

We define the obstacle-occupied workspace as $\mathcal{O}_{\mathcal{W}} := \bigcup_{i\in\mathbb{I}}\mathcal{O}_i, $ where the set $\mathbb{I}:= \{0, 1, ..., b\}$, contains indices corresponding to the disjoint obstacles. The obstacle-free workspace is denoted by $\mathcal{W}_0$, where, given $y\geq 0$, the $y-$eroded version of the obstacle-free workspace \textit{i.e.}, $\mathcal{W}_y$ is defined as
\begin{equation}
    \mathcal{W}_{y} := \mathbb{R}^2\setminus \mathcal{D}_{y}(\mathcal{O}_{\mathcal{W}}^{\circ})\subset\mathcal{W}_0.\label{y_free_workspace}
\end{equation}

Let $r>0$ be the radius of the robot and $r_s >0$ be the minimum distance that the robot should maintain with respect to any obstacle for safe navigation. Hence, $\mathcal{W}_{r_a}$, with $r_a = r + r_s$, is a free workspace with respect to the center of the robot \textit{i.e.}, $\mathbf{x}\in\mathcal{W}_{r_a}\iff\mathcal{B}_{r_a}(\mathbf{x})\subset\mathcal{W}_0$. Since the obstacles can be non-convex and can be in close proximity with each other, to maintain the feasibility of the robot navigation, we make the following assumption:

\begin{figure}
    \centering
    \includegraphics[width = 0.9\linewidth]{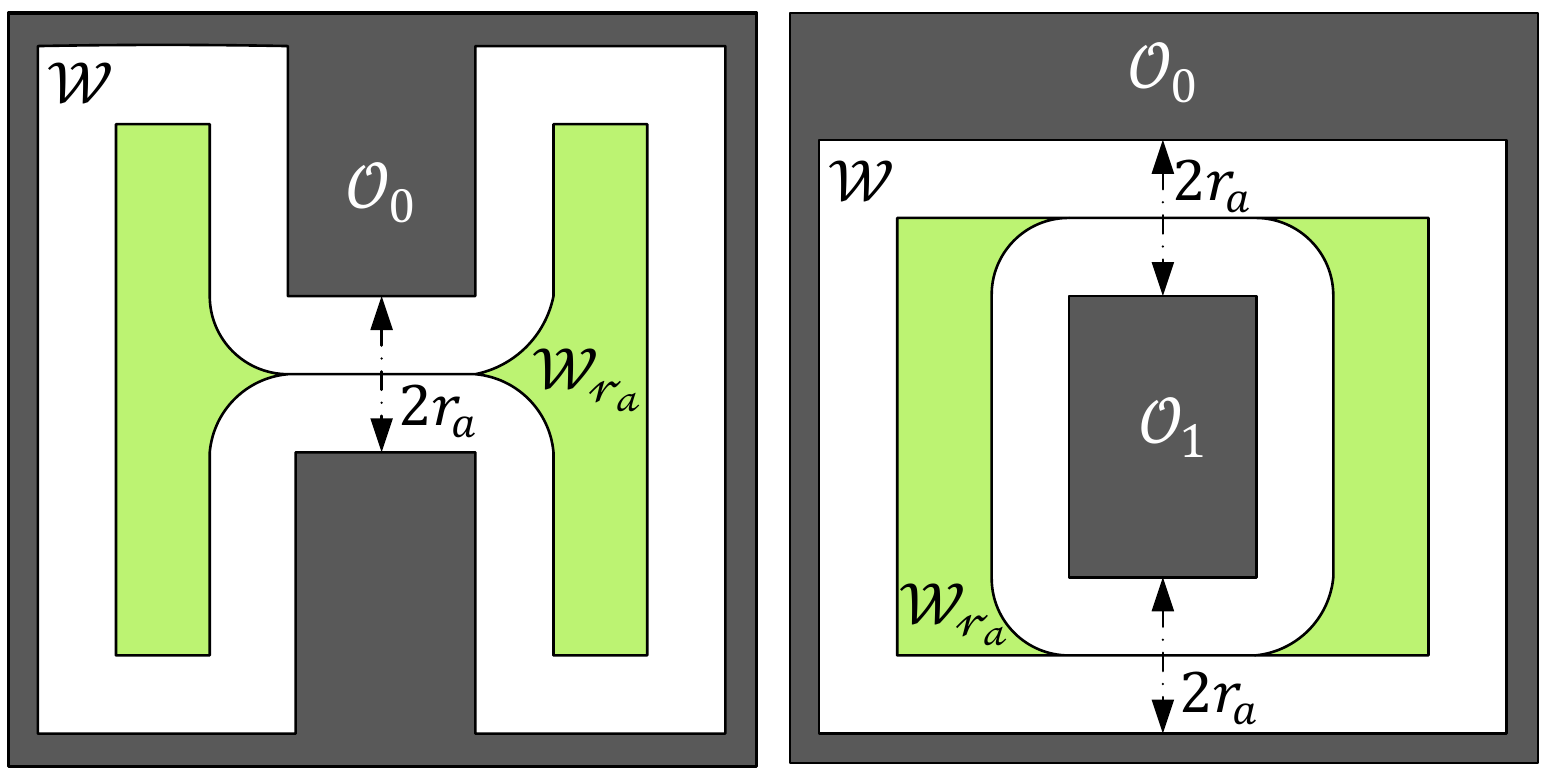}
    \caption{Two examples of workspaces that do not satisfy Assumption \ref{assumption:connected_interior}.}
    \label{diag:workspace_not_satisfying_Assumption1}
\end{figure}

\begin{assumption}
The interior of the obstacle-free workspace w.r.t. the center of the robot \textit{i.e.}, $\mathcal{W}_{r_a}^{\circ}$ is pathwise connected, and $\mathbf{0}\in\mathcal{W}_{r_a}^{\circ}$.\label{assumption:connected_interior}
\end{assumption}

According to Assumption \ref{assumption:connected_interior}, from any location in the set $\mathcal{W}_{r_a}$, there exists at least one feasible path to the target location. We require the origin to be in the interior of the set $\mathcal{W}_{r_a}$ to ensure its stability, as discussed later in Theorem \ref{theorem:global_stability}. Since we require the interior of the set $\mathcal{W}_{r_a}$ to be pathwise connected, the environments, such as the ones showed in Fig. \ref{diag:workspace_not_satisfying_Assumption1}, which do not satisfy Assumption \ref{assumption:connected_interior}, are invalid.

\begin{figure*}
    \centering
    \includegraphics[width = 0.85\textwidth]{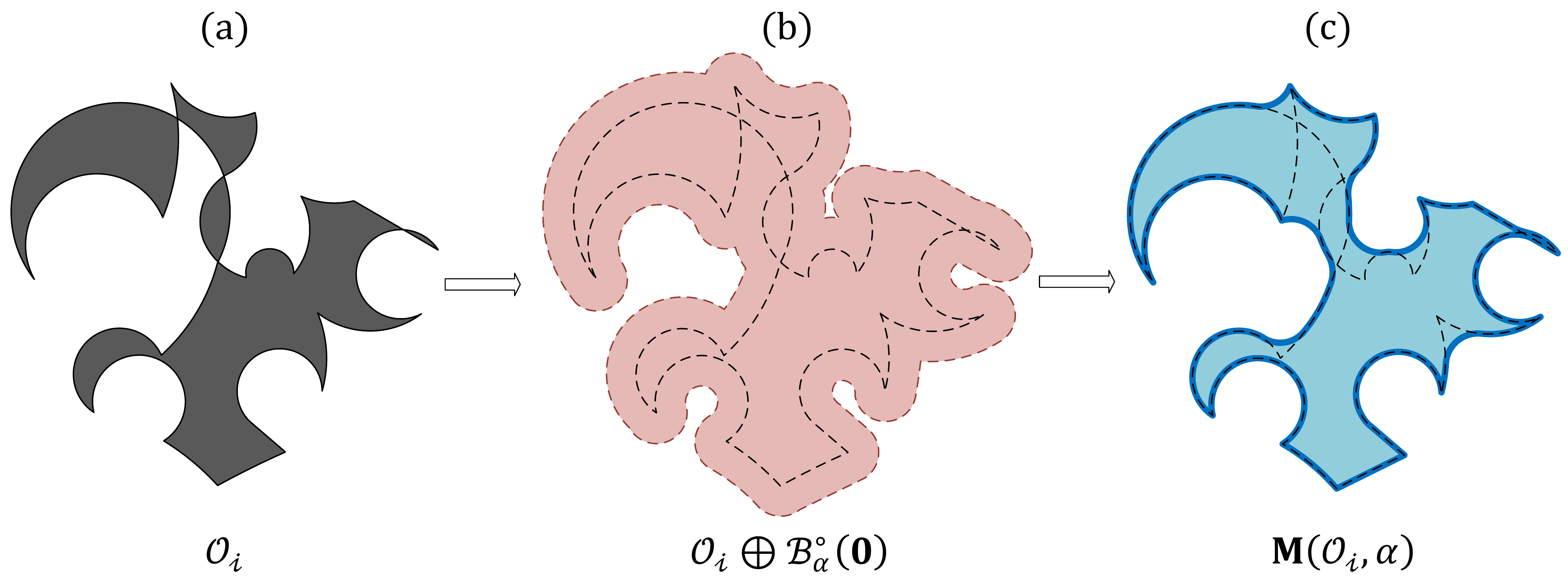}
    \caption{
    (a) The original obstacle $\mathcal{O}_i\subset\mathbb{R}^2$. (b) Dilation of obstacle $\mathcal{O}_i$ by a structuring element $\mathcal{B}_{\alpha}^{\circ}(\mathbf{0}), \alpha > 0$. (c) Erosion of the dilated obstacle $\mathcal{O}_i\oplus\mathcal{B}_{\alpha}^{\circ}(\mathbf{0})$ by the same structuring element $\mathcal{B}_{\alpha}^{\circ}(\mathbf{0})$.}
    \label{diagram:obstacle_reshaping}
\end{figure*}

The obstacle-avoidance strategy, which will be detailed later in Section \ref{sec:the_proposed_contrroller}, requires a unique closest point on the obstacle-occupied workspace from the robot's center. This condition is not always satisfied in the case of non-convex and closely positioned obstacles. Constructing convex hulls around the non-convex obstacles is conservative solution since it makes much of the obstacle-free workspace non-available for navigation. Therefore, in Section \ref{section:obstacle_reshaping}, we will introduce an obstacle reshaping technique that generates a modified obstacle-occupied workspace $\mathcal{O}_{\mathcal{W}}^M$, in a less conservative manner (without necessarily convexifying the obstacles), in a way that ensures uniqueness of the closest point on $\mathcal{O}_{\mathcal{W}}^M$ from the robot's center. 

Similar to $\mathcal{W}_{y}$ in \eqref{y_free_workspace}, the $y-$eroded modified obstacle-free workspace, which is denoted by $\mathcal{V}_{y}$, is defined as
\begin{equation}
    \mathcal{V}_y := \mathbb{R}^2\setminus\mathcal{D}_{y}(\mathcal{O}_{\mathcal{W}}^M)^{\circ},\label{y_modified_free_workspace}
\end{equation}
where $y \geq 0$. Hence, the set $\mathcal{V}_{r_a}$ denotes the modified obstacle-free workspace with respect to the center of the robot. We construct $\mathcal{O}_{\mathcal{W}}^M$ such that the modified obstacle-free workspace $\mathcal{V}_{r_a}$ is a subset of the original obstacle-free workspace $\mathcal{W}_{r_a}$, as stated with more details later in Remark \ref{remark:larger_area}.

Given a target location in the interior of the obstacle-free workspace, \textit{i.e.,} $\mathbf{0}\in(\mathcal{W}_{r_a})^{\circ}$, as stated in Assumption \ref{assumption:connected_interior}, we aim to design a hybrid feedback control law such that: 

\begin{enumerate}
    \item the set $\mathcal{V}_{r_a}$ is forward invariant.
    \item the target location $\mathbf{x} = 0$ is globally asymptotically stable.
\end{enumerate}
As it is going to be shown later, the obstacle reshaping procedure guarantees that if the target location belongs to $(\mathcal{W}_{r_a})^{\circ}$ then it also belongs to $\mathcal{V}_{r_a}$.

We use hybrid feedback control techniques \cite{goebel2012hybrid} to develop a navigation scheme for the robot operating in environments that satisfy Assumption \ref{assumption:connected_interior}. The design process can be summarized as follows:
\begin{enumerate}
\item the proposed hybrid navigation scheme involves two modes of operation for the robot: \textit{move-to-target} and \textit{obstacle-avoidance}. The design of the obstacle avoidance strategy requires a unique projection onto the unsafe region within its close proximity. However, ensuring this uniqueness can be challenging in cases where obstacles have arbitrary shapes and are in close proximity to one another. Hence, before implementing the hybrid navigation scheme, we first transform the obstacles using an obstacle-reshaping operator, as discussed later in Section \ref{section:obstacle_reshaping}, to obtain the modified obstacles. This operator transforms the obstacle-occupied workspace and guarantees the uniqueness of the projection of the robot's center onto the modified obstacle-occupied workspace in its $\alpha$-neighbourhood, where the parameter $\alpha$ is chosen as per Lemma \ref{lemma:alpha_existence}.

\item when the center of the robot is outside the $\alpha$-neighbourhood of the modified obstacles or the nearest disjoint modified obstacle does not intersect with its straight path to the target location, the robot moves straight towards the target in the \textit{move-to-target} mode.

\item when the center of the robot enters the $\alpha-$neighbourhood of the modified obstacle that is obstructing its straight path towards the target location, the robot switches to the \textit{obstacle-avoidance} mode.

\item in the \textit{obstacle-avoidance} mode, to avoid collision, the robot moves parallel to the boundary of the nearest modified obstacle until it reaches the location at which the following two conditions are satisfied: 1) the robot is closer to the target location than the location where it entered in the \textit{obstacle-avoidance} mode; 2) the straight path towards the target from that location does not intersect with the nearest disjoint modified obstacle. At this location, the robot switches back to the \textit{move-to-target} mode.

\item later in Lemma \ref{lemma:properties_of_modified_workspace}, we show that the target location belongs to the modified obstacle-free workspace, and from any location away from the interior of the modified obstacles, there exists a feasible path towards the target location. Hence, with a consecutive implementation of steps $2-4$ for the environment with modified obstacles, we guarantee asymptotic convergence of the center of the robot to the target location.
\end{enumerate}

In the next section, we provide the transformation that modifies the obstacle-occupied workspace $\mathcal{O}_{\mathcal{W}}$ which satisfies Assumption \ref{assumption:connected_interior}, such that the robot always has a unique closest point on the modified obstacle-occupied workspace inside its $\alpha-$neighbourhood.

\section{Obstacle reshaping}
\label{section:obstacle_reshaping}
Given $\mathcal{O}_{\mathcal{W}}$ that satisfies Assumption \ref{assumption:connected_interior}, the objective of the obstacle-reshaping task is to obtain a modified obstacle-occupied workspace $\mathcal{O}_{\mathcal{W}}^M$ such that every location less than $\alpha-$distance away from the set $\mathcal{O}_{\mathcal{W}}^M$ has a unique closest point on the set. The choice of the parameter $\alpha$ is crucial for the successful implementation of the proposed navigation scheme, as stated later in Lemma \ref{lemma:alpha_existence}. For now, we assume that $\alpha > 0$ such that the $\alpha-$eroded obstacle-free workspace $\mathcal{W}_{\alpha}$ is not an empty set. The obstacle-reshaping operator $\mathbf{M}$ is defined as
\begin{equation}
    \mathbf{M}(\mathcal{O}_{\mathcal{W}}, \alpha) = (\mathcal{O}_{\mathcal{W}}\oplus\mathcal{B}_{\alpha}^{\circ}(\mathbf{0}))\ominus\mathcal{B}_{\alpha}^{\circ}(\mathbf{0})=:\mathcal{O}_{\mathcal{W}}^{M}.\label{obstacle_modification_step}
\end{equation}

The operator $\mathbf{M}$ first dilates the set $\mathcal{O}_{\mathcal{W}}$ using the open Euclidean ball of radius $\alpha$ centered at the origin as the structuring element, and then erodes the dilated set using the same structuring element, resulting in the modified set $\mathcal{O}_{\mathcal{W}}^M$. This process is similar to the \textit{closing} operator commonly used in the field of mathematical morphology \cite{haralick1987image}. Note that the proposed modification scheme is applicable to $n-$dimensional environments.
Next, we discuss some of the features of the obstacle-reshaping operator $\mathbf{M}.$

\begin{remark}
Consider the modified set $\mathcal{O}_{\mathcal{W}}^M$ obtained after applying the operator $\mathbf{M}$ on the set $\mathcal{O}_{\mathcal{W}}$ with $\alpha > 0.$ Some of the features of the obstacle-reshaping operator $\mathbf{M}$, as stated in \cite[Table 1]{serra1986introduction}, are as follows:\\
\textbf{Idempotent:} the application of the transformation $\mathbf{M}$ to a modified set $\mathcal{O}_{\mathcal{W}}^M$ with the same structuring element (the open Euclidean ball $\mathcal{B}_{\alpha}^{\circ}(\mathbf{0})$), does not change the set $\mathcal{O}_{\mathcal{W}}^M$ \textit{i.e.,} 
\begin{equation}
\mathbf{M}(\mathbf{M}(\mathcal{O}_{\mathcal{W}}, \alpha), \alpha) = \mathbf{M}(\mathcal{O}_{\mathcal{W}},\alpha) = \mathcal{O}_{\mathcal{W}}^M.
\end{equation}\\
\textbf{Extensive:} the modified set always contains the original set \textit{i.e.}, $\mathcal{O}_{\mathcal{W}}\subset\mathcal{O}_{\mathcal{W}}^M.$\\
\textbf{Increasing:} for any subset $\mathcal{A}\subset\mathcal{O}_{\mathcal{W}}$, the modified set $\mathcal{A}^M$ always belongs to the modified set $\mathcal{O}_{\mathcal{W}}^M$ \textit{i.e.}, $\mathcal{A}^M\subset\mathcal{O}_{\mathcal{W}}^M.$
\label{remark:features}
\end{remark}

Notice that, by duality of dilation and erosion \cite[Theorem 25]{haralick1987image}, the dilation of a set, with an open Euclidean ball centered at the origin as the structuring element, is equivalent to the erosion of the complement of that set with the same structuring element. This allows us to provide alternative representations of the proposed obstacle-reshaping operator, as stated in the next remark.

\begin{remark}\label{remark:alternate_modified_obstacle}
The $\alpha-$eroded obstacle-free workspace $\mathcal{W}_{\alpha}$, defined according to \eqref{y_free_workspace}, is equivalent to the complement of the set obtained after dilating $\mathcal{O}_{\mathcal{W}}$ with the open Euclidean ball of radius $\alpha$ centered at the origin $\mathcal{B}_{\alpha}^{\circ}(\mathbf{0})$ \textit{i.e.,} $\mathcal{W}_{\alpha} = (\mathcal{O}_{\mathcal{W}} \oplus\mathcal{B}_{\alpha}^{\circ}(\mathbf{0}))^c$. Therefore, by duality of dilation and erosion \cite[Theorem 25]{haralick1987image}, the modified obstacle-occupied workspace is equivalent to the complement of the set obtained after dilating $\mathcal{W}_{\alpha}$ by the open Euclidean ball of radius $\alpha$ centered at the origin $\mathcal{B}_{\alpha}^{\circ}(\mathbf{0})$. In other words,
\begin{equation}
    \mathbf{M}(\mathcal{O}_{\mathcal{W}}, \alpha) = (\mathcal{W}_{\alpha}\oplus\mathcal{B}_{\alpha}^{\circ}(\mathbf{0}))^c = \mathcal{W}_{\alpha}^c\ominus\mathcal{B}_{\alpha}^{\circ}(\mathbf{0}).\label{alternate_modified_obstacles}
\end{equation}
\end{remark}

The operator $\mathbf{M}$ does not guarantee a unique projection onto the modified obstacle from every point in its $\alpha-$neighbourhood. To illustrate this fact, we consider an environment with two obstacles $\mathcal{O}_{\mathcal{W}} = \mathcal{O}_i\cup\mathcal{O}_j$ such that $d(\mathcal{O}_i, \mathcal{O}_j) < 2 \alpha,$ as shown in Fig. \ref{diagram:assumption_requirement}. In Fig. \ref{diagram:assumption_requirement}b, one can see that the operation $\mathbf{M}(\mathcal{O}_{\mathcal{W}}, \alpha)$ has fused these obstacles into a single set, represented in blue. However, depending on the arrangement of the obstacles, it may happen that even though $d(\mathcal{O}_i, \mathcal{O}_j)< 2 \alpha$, the modified obstacle-occupied workspace $\mathcal{O}_{\mathcal{W}}^M$ contains two disjoint modified obstacles which are less than $2\alpha$ distance apart from each other, as shown in Fig. \ref{diagram:assumption_requirement}a. In this case, it is possible to find a location $\mathbf{x}$ less than $\alpha$ distance away from the set $\mathcal{O}_{\mathcal{W}}^M$ which has multiple closest points on the set $\mathcal{O}_{\mathcal{W}}^M$, as shown in Fig. \ref{diagram:assumption_requirement}a. 

Observe that in Fig. \ref{diagram:assumption_requirement}a, the open Euclidean balls of radius $\alpha$ centered at the locations $\mathbf{c}_1$ and $\mathbf{c}_2$ intersect each other. Hence, when the set $\mathcal{O}_{\mathcal{W}}\oplus\mathcal{B}_{\alpha}^{\circ}(\mathbf{0})$ is eroded to obtain the modified set \eqref{obstacle_modification_step}, two disjoint modified obstacles are obtained, even though $d(\mathcal{O}_i, \mathcal{O}_j)<2\alpha$. As a result, at the location $\mathbf{x}$ inside the $\alpha-$dilated modified obstacle, one can get multiple projections, as shown in Fig. \ref{diagram:assumption_requirement}a.

\begin{figure}
    \centering
    \includegraphics[width = 1\linewidth]{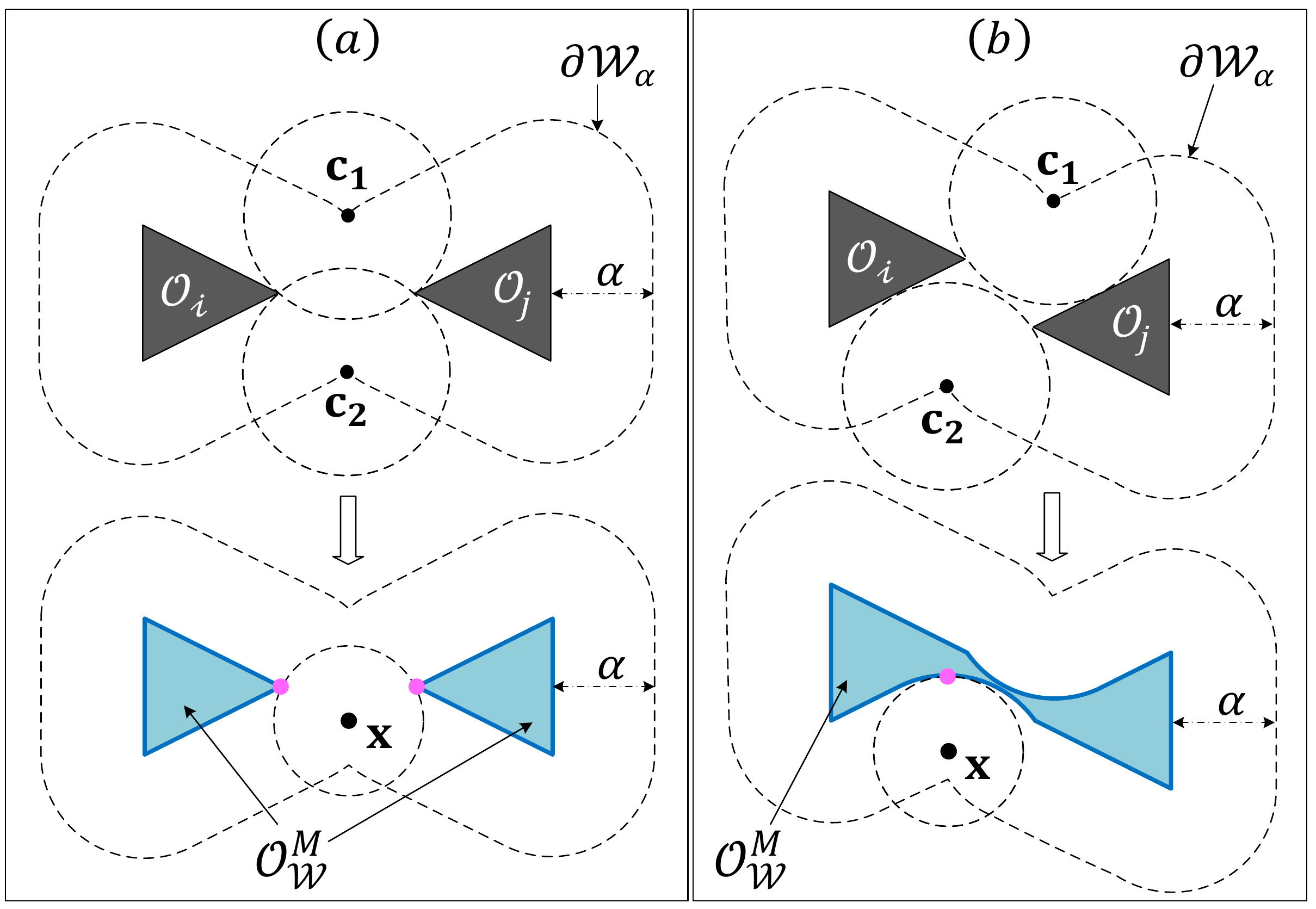}
    \caption{Workspace with two obstacles $\mathcal{O}_{\mathcal{W}} = \mathcal{O}_i\cup\mathcal{O}_j$ such that $d(\mathcal{O}_i, \mathcal{O}_j)<2\alpha$. Left figure shows that the set $\mathcal{O}_{\mathcal{W}}^M$ is not a connected set. Right figure shows that the set $\mathcal{O}_{\mathcal{W}}^M$ is a connected set.}
    \label{diagram:assumption_requirement}
\end{figure}

To guarantee a unique projection onto the modified obstacle from all locations inside its $\alpha-$neighbourhood, we require the following assumption on the obstacle-occupied workspace:

\begin{assumption}
    For all $\mathbf{x}\in\partial\mathcal{W}_{\alpha}$ and for every $\mathbf{n}\in\mathbf{N}_{\mathcal{W}_{\alpha}}(\mathbf{x})$ with $\|\mathbf{n}\| = 1$, the intersection $\mathcal{B}_{\alpha}^{\circ}(\mathbf{x} + \alpha\mathbf{n})\cap\mathcal{W}_{\alpha}$ is an empty set, where $\alpha > 0$ such that $\mathcal{W}_{\alpha}\ne\emptyset$.\label{Assumption:reach}
\end{assumption}

According to Assumption \ref{Assumption:reach}, for all $\mathbf{x}$ on the boundary of the $\alpha-$eroded obstacle-free workspace $\mathcal{W}_{\alpha}$ and for every unit normal vector $\mathbf{n}$ to $\mathcal{W}_{\alpha}$ at $\mathbf{x}$, the open Euclidean ball $\mathcal{B}_{\alpha}(\mathbf{x} + \alpha\mathbf{n})$ does not intersect with the set $\mathcal{W}_{\alpha}.$ 
Figure \ref{diag:reach_diagrammatic_representation}a shows a two-dimensional workspace that does not satisfy Assumption \ref{Assumption:reach}, whereas the inter-obstacle arrangement shown in Fig. \ref{diag:reach_diagrammatic_representation}b satisfies Assumption \ref{Assumption:reach}.

Unlike \cite[Assumprtion 1]{arslan2019sensor}, \cite[Assumprtion 2]{verginis2021adaptive}, and \cite[Section V-C3]{berkane2021obstacle}, Assumption \ref{Assumption:reach} does not impose restrictions on the minimum separation between any pair of obstacles and allows obstacles to be non-convex. If fact, if one assumes (as in the above mentioned references) that the obstacles are convex and the minimum separation between any pair of obstacles is greater than $2r_a$, then Assumption \ref{Assumption:reach} along with Assumption \ref{assumption:connected_interior} are satisfied, as stated later in Proposition \ref{lemma:assumption_equivalence}.

\begin{figure}
    \centering
    \includegraphics[width = 1\linewidth]{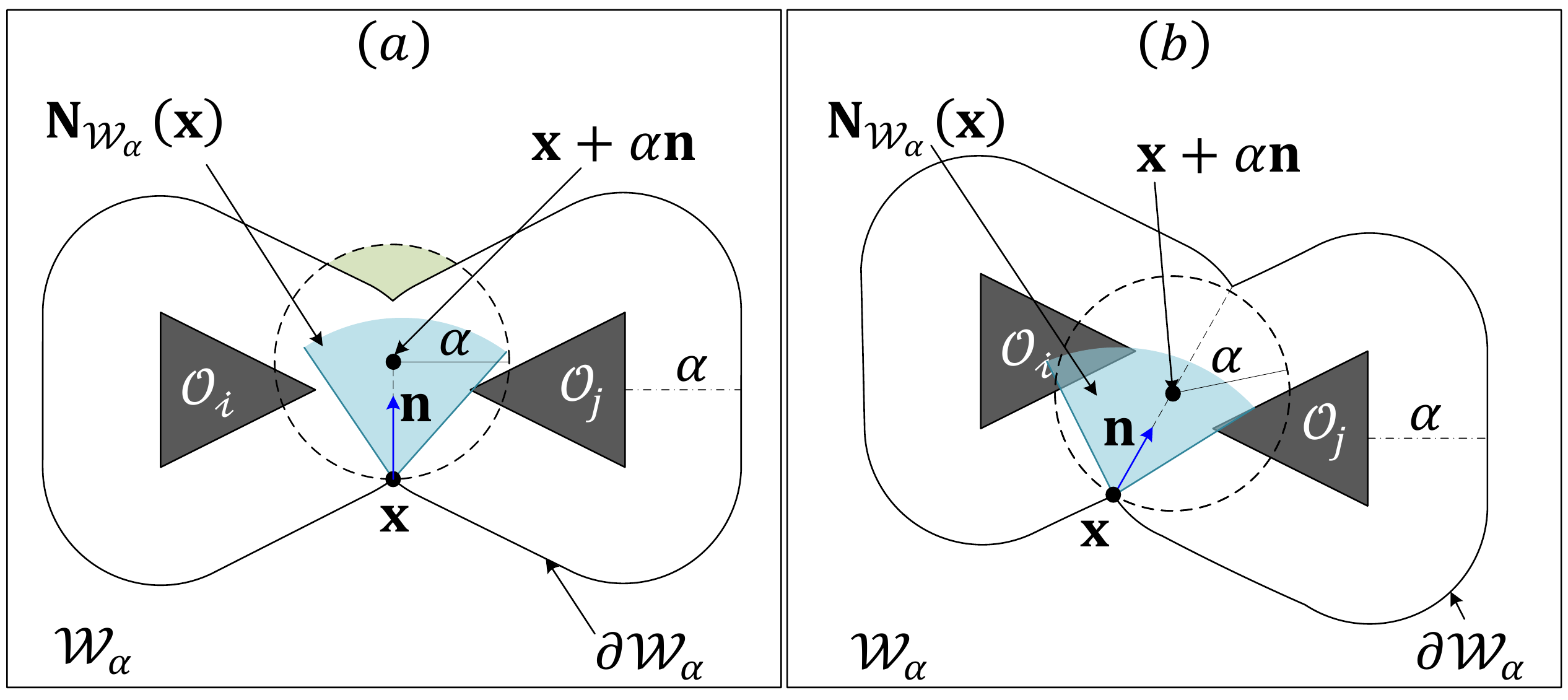}
    \caption{Workspace with two obstacle $\mathcal{O}_{\mathcal{W}} = \mathcal{O}_i\cup\mathcal{O}_j$ such that $d(\mathcal{O}_i, \mathcal{O}_j) < 2\alpha$. (a) $\mathcal{O}_{\mathcal{W}}$ does not satisfy Assumption \ref{Assumption:reach}. (b) $\mathcal{O}_{\mathcal{W}}$ satisfies Assumption \ref{Assumption:reach}.}
    \label{diag:reach_diagrammatic_representation}
\end{figure}

Next, we show that if $\mathcal{O}_{\mathcal{W}}$ satisfies Assumption \ref{Assumption:reach} for some $\alpha > 0$, then from any location in the $\alpha-$neighbourhood of the modified obstacles $\mathcal{O}_{\mathcal{W}}^M$, obtained using \eqref{obstacle_modification_step}, the projection onto $\mathcal{O}_{\mathcal{W}}^M$ is always unique.

\begin{lemma} Let Assumption \ref{Assumption:reach} hold, then, for all locations $\mathbf{x}$ less than $\alpha$ distance away from the modified obstacle-occupied workspace, there is a unique closest point from $\mathbf{x}$ to the set $\mathcal{O}_{\mathcal{W}}^M$ \textit{i.e.,} $\forall\mathbf{x}\in\mathcal{D}_{\alpha}((\mathcal{O}_{\mathcal{W}}^M)^{\circ}), \mathbf{card}(\mathcal{PJ}(\mathbf{x}, \mathcal{O}_{\mathcal{W}}^M)) = 1.$
\label{lemma:unique_projection}
\end{lemma}
\begin{proof}
See Appendix \ref{proof:unique_projection}.
\end{proof}

Notice that in Fig. \ref{diagram:assumption_requirement}b, even though initially obstacles were disjoint, the obstacle-reshaping operator combined them into one pathwise connected modified obstacle set. In fact, for the obstacle-occupied workspace that satisfies Assumption \ref{Assumption:reach}, if two obstacles are less than $2\alpha$ distance apart, then the modified obstacle set obtained for the union of these two obstacles is a connected set. Next, we elaborate on this feature of the proposed obstacle-reshaping operator.

For each obstacle $\mathcal{O}_{i}, i\in\mathbb{I},$ we define the following set:
\begin{equation}
    \mathcal{O}_{i, \alpha} = \bigcup_{k\in\mathbb{I}_{i, \alpha}}\bigcup_{j\in\mathbb{I}_{k, \alpha}}\mathcal{O}_j,\label{obstacle_alpha_chain}
\end{equation}
where, for $k\in\mathbb{I},$ the set $\mathbb{I}_{k, \alpha}$, which is defined as
\begin{equation}
    \mathbb{I}_{k, \alpha} = \{j\in\mathbb{I}|d(\mathcal{O}_k, \mathcal{O}_j) < 2\alpha\},\label{obstacles_less_than_2alpha}
\end{equation}
contains the indices corresponding to obstacles that are at distances less than $2\alpha$ from obstacle $\mathcal{O}_k.$ According to \eqref{obstacle_alpha_chain} and \eqref{obstacles_less_than_2alpha}, the distance between any proper subset of the obstacle set $\mathcal{O}_{i, \alpha}$ and its relative complement with respect to the same set, is always less than $2\alpha.$ In other words, if $\mathcal{O}_{\mathbb{C}}\subset\mathcal{O}_{i, \alpha}$ and $ \mathcal{O}_{\mathbb{D}} = \mathcal{O}_{i, \alpha}\setminus\mathcal{O}_{\mathbb{C}},$ then the distance $d(\mathcal{O}_{\mathbb{C}},\mathcal{O}_{\mathbb{D}}) <2\alpha.$ If $\mathcal{O}_j\subset\mathcal{O}_{i, \alpha}, j\in\mathbb{I}\setminus\{i\}$, then $\mathcal{O}_{i, \alpha} = \mathcal{O}_{j, \alpha}$. Next, we show that the modified obstacle-occupied set $\mathcal{O}_{i, \alpha}^M$, for any $i\in\mathbb{I}$, is a connected set. 

\begin{lemma}\label{lemma:pathwise_connected}
    Let Assumption \ref{Assumption:reach} hold, then the modified obstacle-occupied set $\mathcal{O}_{i, \alpha}^M$, for any $i\in\mathbb{I},$ is a connected set.
\end{lemma}
\begin{proof}
See Appendix \ref{proof:pathwise_connected}
\end{proof}

{According to Lemma \ref{lemma:pathwise_connected}, if the modified obstacle-occupied workspace contains two disjoint modified obstacles, then the distance between these two modified obstacles will always be greater than or equal to $2\alpha.$

\begin{remark}
When the obstacle is non-convex, the modified obstacle obtained using the operator $\mathbf{M}$, defined in \eqref{obstacle_modification_step}, always occupies less workspace as opposed to the convex hull \cite[Section 2.1.4]{boyd2004convex} of the same obstacle. Although, one obtains a unique projection onto the convex hull of a given obstacle within its $\alpha-$neighbourhood, the use of the convex hull renders most of the obstacle-free workspace unavailable for the robot's navigation, as shown in Fig. \ref{diagram:modified_convexhull}. Moreover, if a given obstacle $\mathcal{O}_i$ is convex, then the modified obstacle $\mathcal{O}_i^M$ obtained using \eqref{obstacle_modification_step} for any $\alpha > 0$ is equal to the original obstacle $\mathcal{O}_i$ \textit{i.e.}, $\mathcal{O}_i^M = \mathcal{O}_i.$
\end{remark}

Since, as per Lemma \ref{lemma:unique_projection}, from any location less than $\alpha$ distance away from the modified obstacles $\mathcal{O}_{\mathcal{W}}^M$, the projection on the set $\mathcal{O}_{\mathcal{W}}^M$ is unique, one can roll up an Euclidean ball of radius at most $\alpha$ on the boundary $\partial\mathcal{O}_{\mathcal{W}}^M$, as stated in \cite[Definition 11]{thale200850}. This motivates an alternative procedure to obtain the boundary of the set $\mathcal{O}_{\mathcal{W}}^M$ when the set $\mathcal{O}_{\mathcal{W}}$ is two-dimensional, by having a virtual ring of radius $\alpha$ rolling on the boundary of the set $\mathcal{O}_{\mathcal{W}}$, as stated in the next remark.

\begin{figure}[ht]
    \centering
    \includegraphics[width = 1\linewidth]{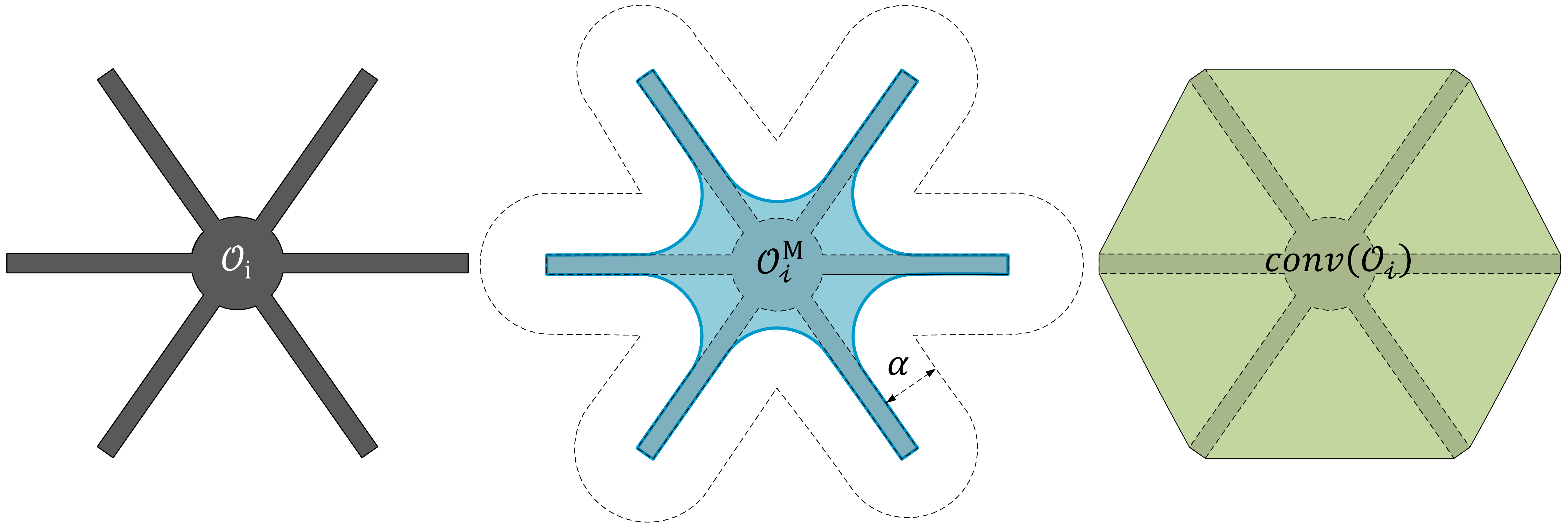}
    \caption{Left figure shows the original obstacle $\mathcal{O}_i\subset\mathbb{R}^2$. Middle figure shows the modified obstacle $\mathcal{O}_{{i}}^M = \mathbf{M}(\mathcal{O}_i, \alpha)$ obtained using \eqref{obstacle_modification_step}. Right figure shows the convex hull $\text{conv}(\mathcal{O}_i)$ for the obstacle $\mathcal{O}_i$.}
    \label{diagram:modified_convexhull}
\end{figure}

\begin{remark}
Given an obstacle-occupied workspace $\mathcal{O}_{\mathcal{W}}\subset\mathbb{R}^2$, satisfying Assumptions \ref{assumption:connected_interior} and \ref{Assumption:reach}, one can construct the modified obstacle $\mathcal{O}_{i, \alpha}^M = \mathbf{M}(\mathcal{O}_{i, \alpha}, \alpha)$, where $i\in\mathbb{I},$ by rotating a virtual ring, of radius $\alpha$ and center $\mathbf{c}$, around the set $\mathcal{O}_{i, \alpha}$, just touching the set $\mathcal{O}_{i, \alpha}$, while ensuring that the ring does not intersect with the interior of that set, as shown in Fig. \ref{program:modified_obstacle_generation}. Then, based on the number of projections of the location $\mathbf{c}$ on the obstacle set $\mathcal{O}_{i, \alpha}$ \textit{i.e.}, $\mathbf{card}(\mathcal{PJ}(\mathbf{c}, \mathcal{O}_{i, \alpha}))$, we construct the boundary of the modified obstacle set $\partial\mathcal{O}_{i, \alpha}^M$ as follows:
\begin{itemize}
    \item if the ring $\partial\mathcal{B}_{\alpha}(\mathbf{c})$ is touching the set $\mathcal{O}_{i, \alpha}$ at a single location, then include that location in the set $\partial\mathcal{O}_{i, \alpha}^{M}$. That is, if $\mathbf{card}(\mathcal{PJ}(\mathbf{c}, \mathcal{O}_{i, \alpha}))=1$, then  $\Pi(\mathbf{c}, \mathcal{O}_{i, \alpha})\in\partial\mathcal{O}_{i, \alpha}^{M}$.
    \item if the ring $\partial\mathcal{B}_{\alpha}(\mathbf{c})$ is simultaneously touching the set $\mathcal{O}_{i, \alpha}$ at more than one location, then include in the set $\partial\mathcal{O}_{i, \alpha}^M$ the part of the ring $\partial\mathcal{B}_{\alpha}(\mathbf{c})$ which intersects the conic hull of the set $\mathcal{PJ}(\mathbf{c}, \mathcal{O}_{i, \alpha})$ with its vertex at $\mathbf{c}$. That is, if $\mathbf{card}(\mathcal{PJ}(\mathbf{c}, \mathcal{O}_{i, \alpha})) > 1$, then  $\mathcal{Y}(\mathbf{c})\subset\partial\mathcal{O}_{i, \alpha}^M$, where the set $\mathcal{Y}(\mathbf{c})$ is given by
    \begin{equation}
    \mathcal{Y}(\mathbf{c}) = \mathcal{CH}(\mathbf{c}, \mathcal{PJ}(\mathbf{c}, \mathcal{O}_{i, \alpha}))\cap\partial\mathcal{B}_{\alpha}(\mathbf{c}).
    \end{equation}
\end{itemize}
\label{remark:virtual_ring}
\end{remark}

We consider the modified obstacle-occupied workspace $\mathcal{O}_{\mathcal{W}}^M\subset\mathcal{W}$ to be the region that the robot should avoid. The set $\mathcal{V}_{r_a}$, defined in \eqref{y_modified_free_workspace}, represents the modified obstacle-free workspace for the center of the robot.

We require the set $\mathcal{V}_{r_a}$ to be a pathwise connected set. However, the pathwise connectedness of the set $\mathcal{V}_{r_a}$ mainly depends on the value of the parameter $\alpha.$ For example, see Fig. \ref{diagram:alpha_should_not_be_very_large} in which the set $\mathcal{V}_{r_a}$ is not connected due to improper selection of the parameter $\alpha$. To that end, we require the following lemma:

\begin{figure}
    \centering
    \includegraphics[width = 0.9\linewidth]{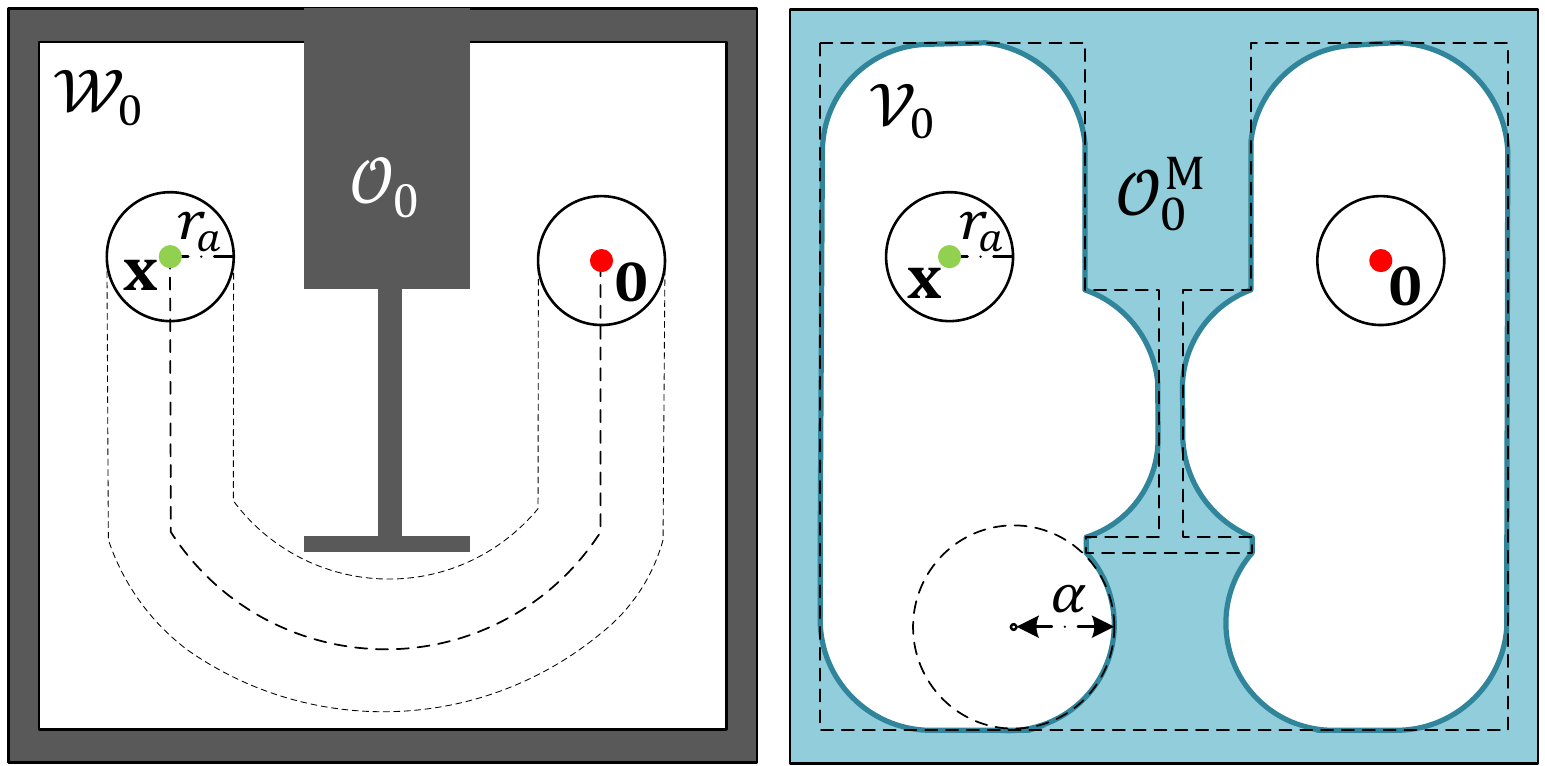}
    \caption{The left figure shows an original workspace where a feasible path exists from location $\mathbf{x}$ to the origin. The right figure shows the modified workspace obtained using \eqref{obstacle_modification_step} with some $\alpha > r_a$, where there is no feasible path from the location $\mathbf{x}$ to the origin.}
    \label{diagram:alpha_should_not_be_very_large}
\end{figure}
\begin{lemma}
    Under Assumption \ref{assumption:connected_interior}, there exists $\bar{\alpha} > r_a$ such that for all $\alpha\in(r_a, \bar{\alpha}]$ the following conditions are satisfied:
    \begin{enumerate}
        \item the $\alpha-$eroded obstacle-free workspace $\mathcal{W}_{\alpha}$ is a pathwise connected set,
        \item the distance between the origin and the set $\mathcal{W}_{\alpha}$ is less than ${\alpha}-r_a$.
    \end{enumerate}
    \label{lemma:alpha_existence}
\end{lemma}
\begin{proof}
See Appendix \ref{proof:alpha_existence}.
\end{proof}
Next, we show that if we choose the parameter $\alpha$, which is used in \eqref{obstacle_modification_step} and \eqref{y_modified_free_workspace} to obtain the set $\mathcal{V}_{r_a}$, as per Lemma \ref{lemma:alpha_existence}, then the set $\mathcal{V}_{r_a}$ is pathwise connected and the origin belongs to its interior.
\begin{lemma}
    If the parameter $\alpha$, which is used in \eqref{obstacle_modification_step}, is chosen as per Lemma \ref{lemma:alpha_existence}, then the modified obstacle-free workspace w.r.t. the center of the robot $\mathcal{V}_{r_a}$ is pathwise connected and $\mathbf{0}\in\mathcal{V}_{r_a}^{\circ}.$\label{lemma:properties_of_modified_workspace}
\end{lemma}
\begin{proof}
See Appendix \ref{proof:properties_of_modified_workspace}
\end{proof}

This concludes the discussion on the formulation and features of the obstacle-reshaping operator $\mathbf{M}$ \eqref{obstacle_modification_step}, which is applicable to $n-$dimensional Euclidean subsets of $\mathbb{R}^n$. Next, we provide the hybrid control design for the robot operating in a two-dimensional workspace \textit{i.e.,} $\mathcal{W}\subset\mathbb{R}^2.$
}

\begin{figure}
    \centering
    \includegraphics[width = 0.85\linewidth]{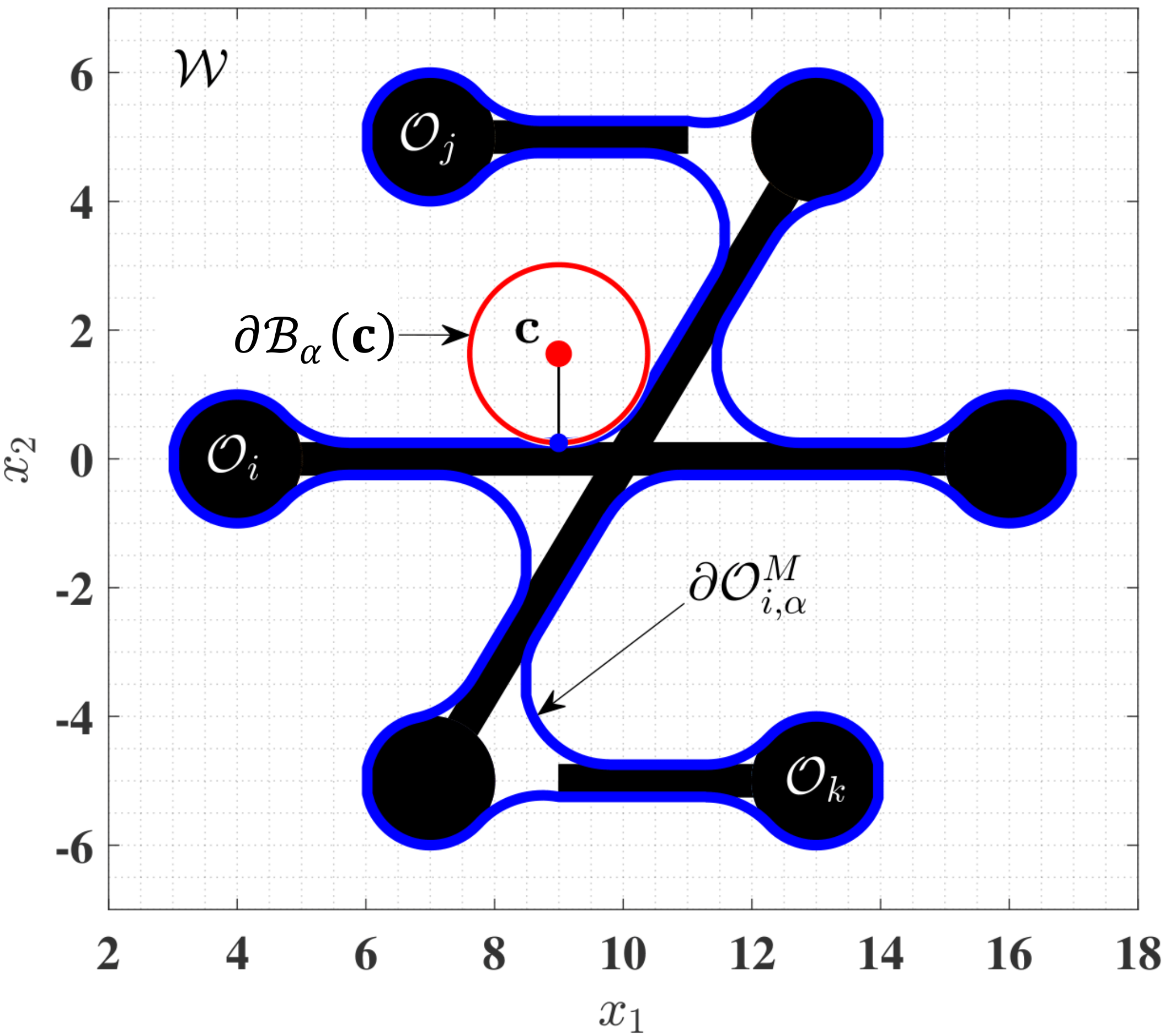}
    \caption{An obstacle-occupied workspace $\mathcal{O}_{\mathcal{W}}$ with three obstacles $\mathcal{O}_i,\mathcal{O}_j$ and $\mathcal{O}_k$, and the boundary of the modified obstacle $\mathcal{O}_{i, \alpha}^M$ obtained using a virtual ring with radius $\alpha$, where $\mathcal{O}_{i, \alpha} = \mathcal{O}_i\cup\mathcal{O}_j\cup\mathcal{O}_k$, see \eqref{obstacle_alpha_chain}. A video can be seen here \url{https://youtu.be/mylrTOtYOSY}.}
    \label{program:modified_obstacle_generation}
\end{figure}

\section{Hybrid Control for Obstacle Avoidance}\label{sec:the_proposed_contrroller}
In the proposed scheme, similar to \cite{berkane2021obstacle}, depending upon the value of the mode indicator $m\in\{-1, 0, 1\}=:\mathbb{M},$ the robot operates in two different modes, namely the \textit{move-to-target} mode $(m = 0)$ when it is away from the modified obstacles and the \textit{obstacle-avoidance} mode $(m\in\{-1, 1\})$ when it is in the vicinity of an modified obstacle. In the \textit{move-to-target} mode, the robot moves straight towards the target, whereas during the \textit{obstacle-avoidance} mode the robot moves around the nearest modified obstacle, either in the clockwise direction $(m = 1)$ or in the counter-clockwise direction $(m = -1)$. We utilize a vector joining the center of the robot and its projection on the modified obstacle-occupied workspace to select between the modes and assign the direction of motion while operating in the \textit{obstacle-avoidance} mode.  

\subsection{Hybrid control Design}
The proposed hybrid control $\mathbf{u}(\mathbf{x}, \mathbf{h}, m)$ is given by
\begin{subequations}
\begin{align}
    \mathbf{u}(\xi) &= -\kappa_s(1 - m^2)\mathbf{x} + \kappa_r m^2\mathbf{v}(\mathbf{x}, m),\label{hybrid_control_input_1}\\
    &\underbrace{\begin{matrix}\mathbf{\dot{h}}\\\dot{m}\end{matrix}\begin{matrix*}[l] =\mathbf{0}\\= 0\end{matrix*}}_{\xi\in\mathcal{F}}, \quad \underbrace{\begin{bmatrix}\mathbf{h}^+\\m^+\end{bmatrix}\in\mathbf{L}(\xi)}_{\xi\in\mathcal{J}},\label{hybrid_control_input_2}
    \end{align}\label{hybrid_control_input}
\end{subequations}
where $\kappa_s > 0$, $\kappa_r > 0$ and $\xi = (\mathbf{x}, \mathbf{h}, m)\in\mathcal{V}_{r_a}\times\mathcal{V}_{r_a}\times\mathbb{M} =:\mathcal{K},$ is the composite state vector. In \eqref{hybrid_control_input_1}, $\mathbf{x}\in\mathcal{V}_{r_a}$ is the location of the center of the robot. The state $\mathbf{h}\in\mathcal{V}_{r_a}$, referred to as a \textit{hit point}, is the location of the center of the robot when it enters in the \textit{obstacle-avoidance} mode. The discrete variable $m\in\mathbb{M}$ is the mode indicator. The update law $\mathbf{L}(\xi)$, used in \eqref{hybrid_control_input_2}, which allows the robot to switch between the modes, is discussed in Section \ref{section:update_law}. The symbols $\mathcal{F}$ and $\mathcal{J}$ denote the flow and jump sets related to different modes of operations, respectively, whose constructions are provided in Section \ref{section:flowset_jumpset_construction}. Next, we provide the design of the vector $\mathbf{v}(\mathbf{x}, m)\in\mathbb{R}^2$, used in \eqref{hybrid_control_input_1}.

The vector $\mathbf{v}(\mathbf{x}, m)$ is defined as
\begin{equation}
    \mathbf{v}(\mathbf{x}, m) = \begin{bmatrix}0 & m\\-m & 0\end{bmatrix}\frac{\mathbf{x} - \Pi(\mathbf{x},\mathcal{O}_{\mathcal{W}}^M)}{\norm{\mathbf{x} - \Pi(\mathbf{x},\mathcal{O}_{\mathcal{W}}^M)}},\label{definition:vxm}
\end{equation}
where $\Pi(\mathbf{x}, \mathcal{O}_{\mathcal{W}}^M)$ is the point on the modified obstacle-occupied workspace $\mathcal{O}_{\mathcal{W}}^M$ which is closest to the center of the robot $\mathbf{x}$, as defined in Section \ref{section:metric_projection}. As per Lemma \ref{lemma:unique_projection}, if the center of the robot $\mathbf{x}$ is inside the $\alpha-$neighbourhood of the set $\mathcal{O}_{\mathcal{W}}^M$, then $\Pi(\mathbf{x}, \mathcal{O}_{\mathcal{W}}^M)$ is unique.
When the robot operates in the \textit{obstacle-avoidance} mode, the vector $\mathbf{v}(\mathbf{x}, m)$ allows it to move around the nearest obstacle either in the clockwise direction $(m = 1)$ or in the counter-clockwise direction $(m = -1)$.  Next, we discuss the design of the flow set $\mathcal{F}$ and the jump set $\mathcal{J}$, used in \eqref{hybrid_control_input}.

\subsection{Geometric construction of the flow and jump sets}\label{section:flowset_jumpset_construction}

\begin{figure}
    \centering
    \includegraphics[width = 0.508\linewidth]{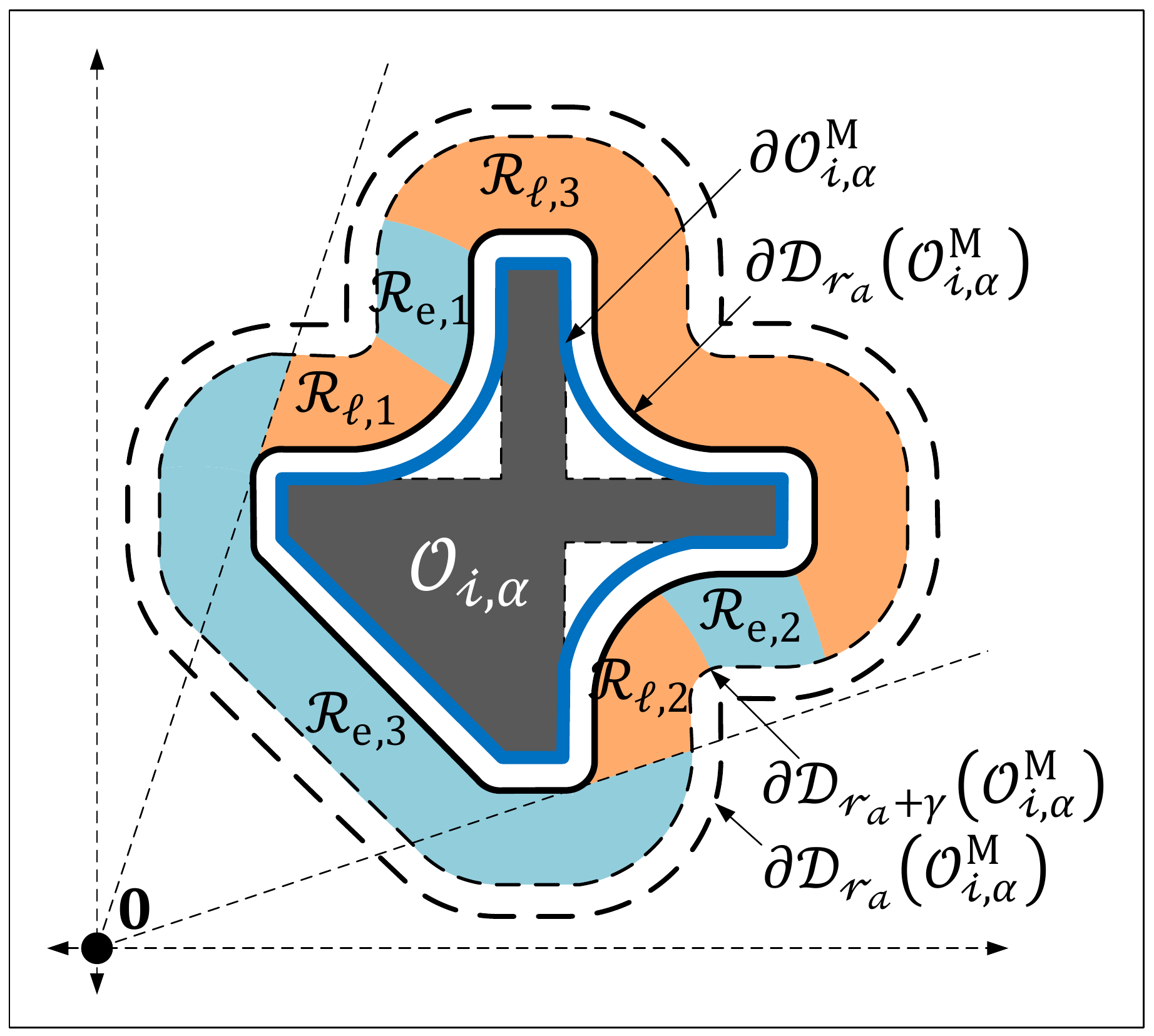}
    \includegraphics[width = 0.475\linewidth]{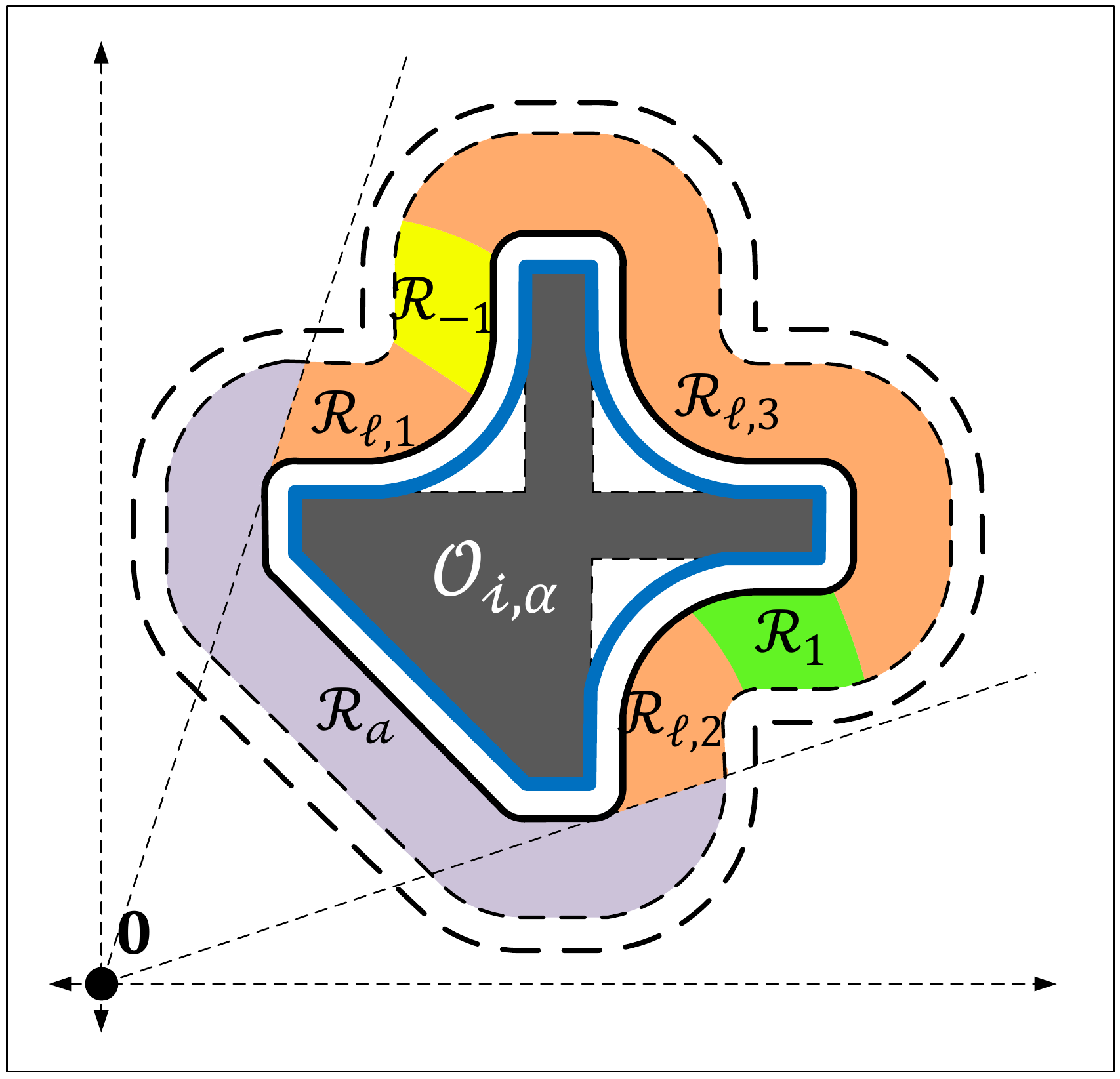}
        \caption{The left figure shows the partitioning of the region $\mathcal{N}_{\gamma}(\mathcal{D}_{r_a}(\mathcal{O}_{i}^M))$ into the \textit{landing} region $\mathcal{R}_l = \mathcal{R}_{l, 1}\cup\mathcal{R}_{l, 2}\cup\mathcal{R}_{l, 3}$, and the \textit{exit} region $\mathcal{R}_e = \mathcal{R}_{e, 1}\cup\mathcal{R}_{e, 2}\cup\mathcal{R}_{e, 3},$ using \eqref{definition:landing_region} and \eqref{definition:exit_region}. The right figure shows the partitioning of the \textit{exit} region into three sub-regions namely the \textit{always exit} region $\mathcal{R}_a$, the \textit{clockwise exit} region $\mathcal{R}_1$, and the \textit{counter-clockwise exit} region $\mathcal{R}_{-1}, $ using \eqref{always_exit_region}, \eqref{clockwise_exit_region} and \eqref{counter_clockwise_exit_region}, respectively.} 
    \label{diagram:exit_region_partitions}
\end{figure}

{When the robot operates in the \textit{move-to-target} mode, in the modified free workspace, its velocity is directed towards the target location. If any connected modified obstacle $\mathcal{O}_{i, \alpha}^M, i\in\mathbb{I}$, is on the line segment $\mathcal{L}_s(\mathbf{x}, \mathbf{0})$ connecting the robot's location with the target location, then the robot, operating in the \textit{move-to-target} mode, will enter the $\alpha-$neighbourhood of that modified obstacle (\textit{i.e.}, $d(\mathbf{x}, \mathcal{O}_{i, \alpha}^M)\leq \alpha$) from the \textit{landing} region $\mathcal{R}_l$, as depicted in Fig. \ref{diagram:exit_region_partitions}. The parameter $\alpha$ is chosen as per Lemma \ref{lemma:alpha_existence}. The \textit{landing} region is defined as
\begin{equation}
    \mathcal{R}_l := \bigcup_{i\in\mathbb{I}}\mathcal{R}_l^i,\label{definition:landing_region}
\end{equation}
where, for each $i\in\mathbb{I}$, the set $\mathcal{R}_l^i$ is given by
\begin{equation}
\begin{aligned}
\mathcal{R}_l^i := \{&\mathbf{x}\in\mathcal{N}_{\gamma}(\mathcal{D}_{r_a}(\mathcal{O}_{i,\alpha}^M))| \mathbf{x}^{\intercal}(\mathbf{x} - \Pi(\mathbf{x}, \mathcal{O}_{i, \alpha}^M)) \geq 0, \\
&\mathcal{L}_{s}(\mathbf{x}, \mathbf{0})\cap\left(\mathcal{D}_{r_a}(\mathcal{O}_{i,\alpha}^M)\right)^{\circ} \ne\emptyset\},\label{definition:individual_landing_region}
\end{aligned}
\end{equation}
where $\gamma \in (0, \alpha - r_a)$. Given a set $\mathcal{A}\subset\mathbb{R}^n$ and a scalar $r > 0$, the $r-$neighbourhood of the set $\mathcal{A}$ is denoted by $\mathcal{N}_r(\mathcal{A}) = \mathcal{D}_r(\mathcal{A})\setminus(\mathcal{A})^\circ.$

Notice that for a connected modified obstacle $\mathcal{O}_{i, \alpha}^M, i\in\mathbb{I}$, the \textit{landing} region $\mathcal{R}_{l}^i$, defined in \eqref{definition:individual_landing_region}, is the intersection of the following two regions:
\begin{enumerate}
    \item the region where the line segment $\mathcal{L}_{s}(\mathbf{x}, \mathbf{0})$, which joins the center of the robot and the target location, intersects with the interior of the $r_a-$dilated connected modified obstacle $\mathcal{D}_{r_a}^{\circ}(\mathcal{O}_{i, \alpha}^M)$. Hence, if the robot moves straight towards the target in this region, it will eventually collide with the modified obstacle $\mathcal{O}_{i, \alpha}^M.$
    \item the region where the inner product between the vectors $\mathbf{x}$ and $\mathbf{x} - \Pi(\mathbf{x}, \mathcal{O}_{i, \alpha}^M)$ is non-negative, as shown in Fig. \ref{diag:tangent_cone}. Notice that, when the robot moves straight towards the target in this region, its distance from the modified obstacle $\mathcal{O}_{i, \alpha}^M$ \textit{i.e.}, $d(\mathbf{x}, \mathcal{O}_{i, \alpha}^M)$ does not increase. To understand this fact, observe that if for any $\mathbf{x}\in\mathcal{N}_{\gamma}(\mathcal{D}_{r_a}(\mathcal{O}_{i, \alpha}^M))$, $\mathbf{x}^{\intercal}(\mathbf{x} - \Pi(\mathbf{x}, \mathcal{O}_{i, \alpha}^M))\geq 0$, then $-\mathbf{x}\in\mathbf{T}_{\mathcal{D}_{d(\mathbf{x}, \mathcal{O}_{i, \alpha}^M)}(\mathcal{O}_{i, \alpha}^M)}(\mathbf{x}) = \mathcal{P}_{\leq}(\mathbf{0}, \mathbf{x} - \Pi(\mathbf{x}, \mathcal{O}_{i, \alpha}^M))$ \textit{i.e.}, the robot, moving straight towards the target in this region,
    does not leave the region $\mathcal{D}_{d(\mathbf{x}, \mathcal{O}_{i, \alpha}^M)}(\mathcal{O}_{i, \alpha}^M)$. The notation $\mathbf{T}_{\mathcal{D}_{d(\mathbf{x}, \mathcal{O}_{i, \alpha}^M)}(\mathcal{O}_{i, \alpha}^M)}(\mathbf{x})$ denotes the tangent cone
     to the set $\mathcal{D}_{d(\mathbf{x}, \mathcal{O}_{i, \alpha}^M)}(\mathcal{O}_{i, \alpha}^M)$ at $\mathbf{x}.$
\end{enumerate}

\begin{figure}
    \centering
    \includegraphics[width = 0.9\linewidth]{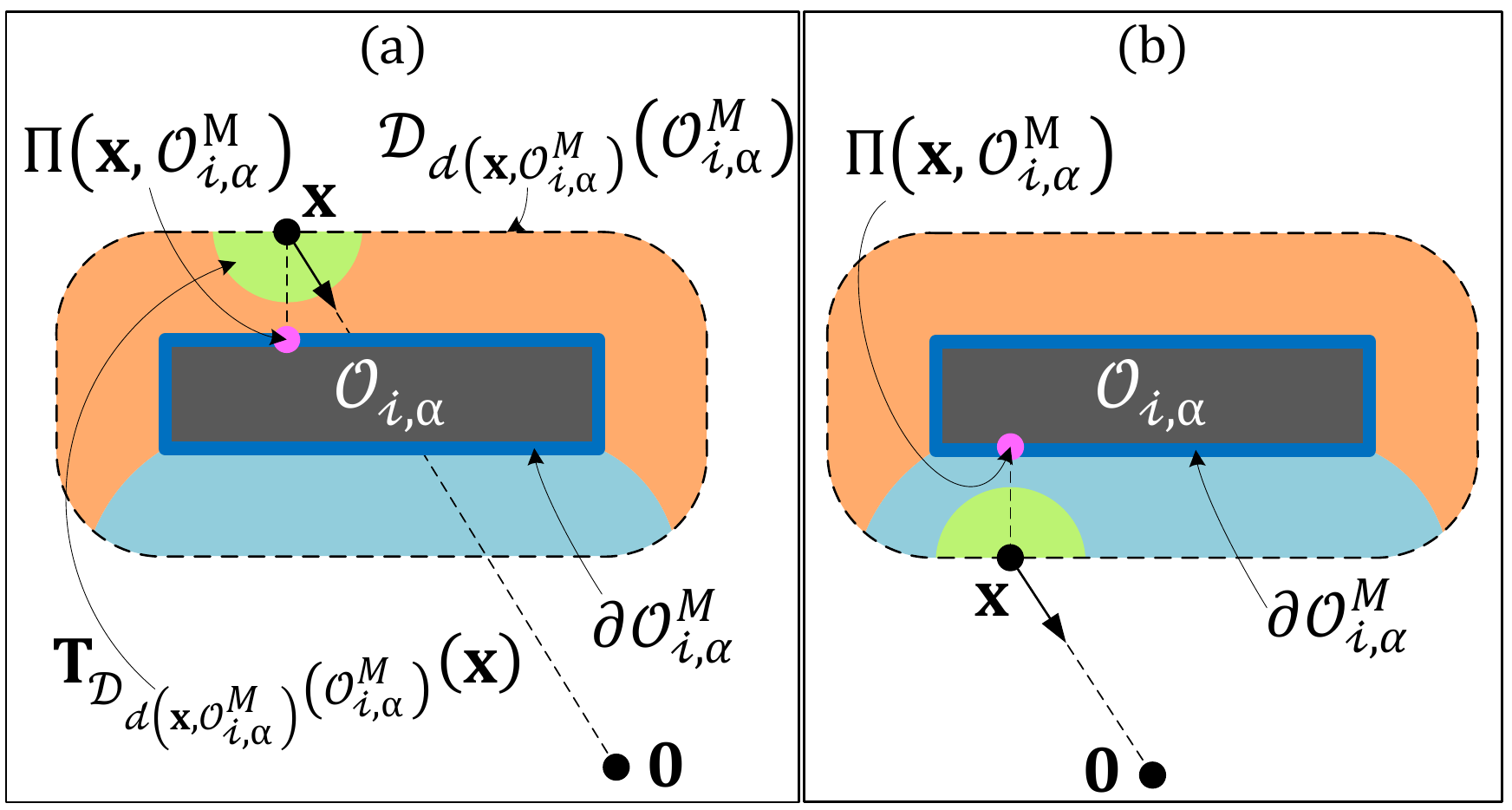}
    \caption{Partitions of the neighbourhood of the modified obstacle $\mathcal{O}_{i, \alpha}^M$ based on the value of the inner product between the two vectors $-\mathbf{x}$ and $\mathbf{x} - \Pi(\mathbf{x}, \mathcal{O}_{i, \alpha}^M)$ and the tangent cone to the set $\mathcal{D}_{d(\mathbf{x}, \mathcal{O}_{i, \alpha}^M)}(\mathcal{O}_{i, \alpha}^M)$ at $\mathbf{x}$. (a) The vector $-\mathbf{x}$ is pointing inside the tangent cone. (b) The vector $-\mathbf{x}$ is pointing outside the tangent cone.}
    \label{diag:tangent_cone}
\end{figure}

According to \eqref{definition:individual_landing_region}, due to the intersection of the above-mentioned two regions, the \textit{landing} region $\mathcal{R}_l^i$ excludes a set of locations from the region $\mathcal{N}_{\gamma}(\mathcal{D}_{r_a}(\mathcal{O}_{i, \alpha}^M))$, for which the inner product between the vectors $\mathbf{x}$ and $\mathbf{x} - \Pi(\mathbf{x}, \mathcal{O}_{\mathcal{W}}^M)$ is negative and the line segment $\mathcal{L}_{s}(\mathbf{x}, \mathbf{0})$ intersects with the interior of the set $\mathcal{D}_{r_a}(\mathcal{O}_{i, \alpha}^M)$, for example see the regions $\mathcal{R}_{e, 1}$ and $\mathcal{R}_{e, 2}$ in Fig. \ref{diagram:exit_region_partitions}. Even though the robot does not have a line-of-sight towards the target location, as long as it moves straight towards the target in these regions, its distance from the modified obstacle $\mathcal{O}_{i, \alpha}^M$ \textit{i.e.}, $d(\mathbf{x}, \mathcal{O}_{i, \alpha}^M)$ increases. To understand this fact, observe that if for any $\mathbf{x}\in\mathcal{N}_{\gamma}(\mathcal{D}_{r_a}(\mathcal{O}_{i, \alpha}^M))$, $\mathbf{x}^{\intercal}(\mathbf{x} - \Pi(\mathbf{x}, \mathcal{O}_{i, \alpha}^M))$ is negative, then $-\mathbf{x}\notin\mathbf{T}_{\mathcal{D}_{d(\mathbf{x}, \mathcal{O}_{i, \alpha}^M)}(\mathcal{O}_{i, \alpha}^M)}(\mathbf{x}) = \mathcal{P}_{\leq}(\mathbf{0}, \mathbf{x} - \Pi(\mathbf{x}, \mathcal{O}_{i, \alpha}^M))$ \textit{i.e.}, the robot, moving straight towards the target in this region, does not enter the region $\mathcal{D}_{d(\mathbf{x}, \mathcal{O}_{i, \alpha}^M)}(\mathcal{O}_{i, \alpha}^M).$ Due to this property, we include these locations in the region called an \textit{exit} region, wherein the robot can operate in the \textit{move-to-target} mode, as discussed next.




When the robot operates in the \textit{obstacle-avoidance} mode in the $\gamma-$ neighbourhood of the modified obstacles, it will switch to the \textit{move-to-target} mode from the \textit{exit} region, which is defined as follows:
\begin{equation}
    \mathcal{R}_e := \overline{\mathcal{N}_{\gamma}(\mathcal{D}_{r_a}(\mathcal{O}_{\mathcal{W}}^M))\setminus\mathcal{R}_l}.\label{definition:exit_region}
\end{equation} 
According to \eqref{definition:landing_region}, \eqref{definition:individual_landing_region} and \eqref{definition:exit_region}, the \textit{exit} region is a combination of following two regions:
\begin{enumerate}
    \item the region where the inner product $\mathbf{x}^\intercal(\mathbf{x} - \Pi(\mathbf{x}, \mathcal{O}_{\mathcal{W}}^M))$ is non-positive. As discussed earlier, when the robot moves straight towards the target in this region, its distance from the unsafe region does not decrease.
    \item the region where the inner product $\mathbf{x}^\intercal(\mathbf{x} - \Pi(\mathbf{x}, \mathcal{O}_{\mathcal{W}}^M))$ is positive and the line segment $\mathcal{L}_{s}(\mathbf{x}, \mathbf{0})$, which joins the robot's location and the target location, does not intersect with the nearest modified obstacle. Hence, the robot can move straight towards the target and safely leave the $\alpha-$neighbourhood of this connected modified obstacle.
\end{enumerate}
}

As shown in Fig. \ref{diagram:exit_region_partitions} (left), for a modified non-convex obstacle $\mathcal{O}_{i, \alpha}^M, i\in\mathbb{I}$, the \textit{exit} region is not a connected set. Consider a situation, wherein the robot is moving in the clockwise direction with respect to the set $\mathcal{O}_{i, \alpha}^M$ in the \textit{landing} region $\mathcal{R}_{l, 1}$. If the robot were to start moving straight towards the target after entering the \textit{exit} region $\mathcal{R}_{e, 1}$, it will re-enter the region $\mathcal{R}_{l, 1}$, resulting in multiple simultaneous switching instances. Similar situation can occur for the robot moving in the \textit{counter-clockwise} direction with respect to the set $\mathcal{O}_{i,\alpha}^M$ in the \textit{landing} region $\mathcal{R}_{l, 2}$, if it moves straight towards the target after entering the \textit{exit} region $\mathcal{R}_{e, 2}.$ On the other hand, if the robot enters in the region $\mathcal{R}_{e, 3}$, then, irrespective of the direction of motion around the obstacle, it can safely move straight towards the target and leave the $\gamma-$neighbourhood of the modified obstacle $\mathcal{O}_{i,\alpha}^M.$

Hence, as shown in Fig. \ref{diagram:exit_region_partitions} (right), based on the angle between the vectors $\mathbf{x}$ and $\mathbf{x} - \Pi(\mathbf{x}, \mathcal{O}_{\mathcal{W}}^M)$, and the presence of a line-of-sight to the target location with respect to the nearest disjoint modified obstacle, we divide the \textit{exit} region $\mathcal{R}_e$ into three sub-regions $\mathcal{R}_a$, $\mathcal{R}_1$ and $\mathcal{R}_{-1}$, referred to as the \textit{always exit} region, the \textit{clockwise exit} region and the \textit{counter-clockwise exit} region, respectively, as follows:

The \textit{always exit} region $\mathcal{R}_a$ is defined as
\begin{equation}
    \mathcal{R}_a := \bigcup_{i\in\mathbb{I}}\mathcal{R}_a^i,\label{always_exit_region}
\end{equation}
where, for each $i\in\mathbb{I}$, the set $\mathcal{R}_a^i$, which is given by
\begin{equation}
    \mathcal{R}_a^i:= \{\mathbf{x}\in\mathcal{R}_e|\mathcal{L}_s(\mathbf{x}, \mathbf{0})\cap\left(\mathcal{D}_{r_a}(\mathcal{O}_{i, \alpha}^M)\right)^{\circ} = \emptyset\},
\end{equation}
contains the locations from the exit region $\mathcal{R}_e$ such that the line segments joining them to the origin do not intersect with the interior of the $r_a-$dilated modified obstacle $\mathcal{D}_{r_a}(\mathcal{O}_{i, \alpha}^M).$

The \textit{clockwise exit} region $\mathcal{R}_1$ is defined as
\begin{equation}
    \mathcal{R}_1 := \left\{\mathbf{x}\in\mathcal{R}_e\setminus\mathcal{R}_a|\measuredangle\left(\mathbf{x}, \mathbf{x} - \Pi(\mathbf{x}, \mathcal{O}_{\mathcal{W}}^M)\right) \in\left[\pi, 2\pi\right]\right\},\label{clockwise_exit_region}
\end{equation}
and the \textit{counter-clockwise exit} region $\mathcal{R}_{-1}$ is defined as
\begin{equation}
    \mathcal{R}_{-1} := \left\{\mathbf{x}\in\mathcal{R}_e\setminus\mathcal{R}_a|\measuredangle\left(\mathbf{x}, \mathbf{x} - \Pi(\mathbf{x}, \mathcal{O}_{\mathcal{W}}^M)\right) \in\left[0, \pi\right]\right\}\label{counter_clockwise_exit_region}
\end{equation}
where, given two vectors $\mathbf{p}, \mathbf{q}\in\mathbb{R}^2, $ the notation $\measuredangle(\mathbf{p}, \mathbf{q})$  indicates the angle measured from $\mathbf{p}$ to $\mathbf{q}$. The angle measured in the counter-clockwise direction is considered positive, and vice versa.

While moving in the clockwise direction around the modified obstacle, the robot is allowed to move straight towards the target only if its center is in the region $\mathcal{R}_1\cup\mathcal{R}_a$. Whereas, the robot moving in the counter-clockwise direction around the modified obstacle should move straight towards the target only if its center is in the region $\mathcal{R}_{-1}\cup\mathcal{R}_a$. Next, we provide the geometric constructions of the flow set $\mathcal{F}$ and the jump set $\mathcal{J}$, used in \eqref{hybrid_control_input}.


\subsubsection{Flow and jump sets (move-to-target mode)}

As discussed earlier, if the robot, which is moving straight towards the target, is on a collision path towards a connected modified obstacle $\mathcal{O}_{i,\alpha}^M$, for some $i\in\mathbb{I}$, then it will enter the $\gamma-$neighbourhood of this modified obstacle through the \textit{landing} region. Hence, the jump set of the \textit{move-to-target} mode for the state $\mathbf{x}$ is defined as
\begin{equation}
    \mathcal{J}_0^{\mathcal{W}}:=\mathcal{N}_{\gamma_s}(\mathcal{D}_{r_a}(\mathcal{O}_{\mathcal{W}}^M))\cap\mathcal{R}_l\label{stabilization_jump_set},
\end{equation}
where $\gamma_s\in[0, \gamma)$. For robustness purposes (with respect to noise), we introduce a hysteresis region by allowing the robot, operating in the \textit{move-to-target} mode inside the $\gamma-$neighbourhood of the set $\mathcal{O}_{\mathcal{W}}^M$, to move closer to the set $\mathcal{O}_{\mathcal{W}}^M$ before switching to the \textit{obstacle-avoidance} mode.

The flow set of the \textit{move-to-target} mode for the state $\mathbf{x}$ is then defined as
\begin{equation}
    \mathcal{F}_0^{\mathcal{W}}:=\left(\mathcal{W}\setminus\left(\mathcal{D}_{r_a+\gamma_s}(\mathcal{O}_{\mathcal{W}}^M)\right)^{\circ}\right)\cup\mathcal{R}_e\label{stabilization_flow_set}.
\end{equation}
Notice that the union of the jump set \eqref{stabilization_jump_set} and the flow set \eqref{stabilization_flow_set} exactly covers the modified robot-centred obstacle-free workspace $\mathcal{V}_{r_a}$ \eqref{y_modified_free_workspace}. Refer to Fig. \ref{diagram:flwoandjumpset} for the representation of the flow and jump sets related to the modified obstacle-occupied workspace $\mathcal{O}_{\mathcal{W}}^M$. Next, we provide the construction of the flow and jump sets for the \textit{obstacle-avoidance} mode.

\subsubsection{Flow and jump sets (\textit{obstacle-avoidance} mode)}

The robot operates in the \textit{obstacle-avoidance} mode only in the $\gamma-$neighbourhood of the modified obstacles $\mathcal{O}_{\mathcal{W}}^M.$ 
The mode indicator variable $ m = 1$ and $m = -1$ prompts the robot to move either in the clockwise direction or in the counter-clockwise direction with respect to the nearest boundary of the set $\mathcal{O}_{\mathcal{W}}^M$, respectively. As discussed earlier, for some $m\in\{-1, 1\},$ the robot should exit the \textit{obstacle-avoidance} mode and switch to the \textit{move-to-target} mode only if its center belongs to the exit region $\mathcal{R}_m\cup\mathcal{R}_a$.

To that end, we make use of the \textit{hit point} $\mathbf{h}$ (\textit{i.e.}, the location of the center of the robot when it switched from the \textit{move-to-target} mode to the current \textit{obstacle-avoidance} mode) to define the jump set of the \textit{obstacle-avoidance} mode $\mathcal{J}_{m}^{\mathcal{W}}$ for the state $\mathbf{x}$ as follows:

\begin{equation}
    \mathcal{J}_m^{\mathcal{W}} := \left(\mathcal{W}\setminus(\mathcal{D}_{r_a + \gamma}(\mathcal{O}_{\mathcal{W}}^M))^{\circ}\right)\cup\mathcal{ER}_m^{\mathbf{h}}\cup\mathcal{B}_{\delta}(\mathbf{0}),\label{avoidance_jump_set}
\end{equation}
where $m\in\{-1, 1\}$ and the parameter $\delta\in(0, d(\mathbf{0}, \mathcal{O}_{\mathcal{W}}^M) - r_a).$ Note that, according to Lemma \ref{lemma:properties_of_modified_workspace}, the target location $\mathbf{0}\in\mathcal{V}_{r_a}^{\circ}$. As a result, the distance $d(\mathbf{0}, \mathcal{O}_{\mathcal{W}}^M) > r_a,$ which guarantees the existence of the parameter $\delta.$ 
In \eqref{avoidance_jump_set}, the inclusion of the set $\mathcal{B}_{\delta}(\mathbf{0})$ in the set $\mathcal{J}_{m}^{\mathcal{W}}$ allows us to ensure the stability of the origin, as stated later in Theorem \ref{theorem:global_stability}.

For some $m\in\{-1, 1\}$ and the \textit{hit point} $\mathbf{h}\in\mathcal{V}_{r_a}$, the set $\mathcal{ER}_m^{\mathbf{h}}$ is given by
\begin{equation}
    \mathcal{ER}_m^{\mathbf{h}} := \{\mathbf{x}\in\overline{\mathcal{R}_m\cup\mathcal{R}_a}\big{|}\norm{\mathbf{h}} - \norm{\mathbf{x}} \geq \epsilon\},\label{partition_rm}
\end{equation}
{where $\epsilon \in(0, \bar{\epsilon}]$, with $\bar{\epsilon} > 0$. This set contains the locations from the exit region $\mathcal{R}_m\cup\mathcal{R}_a$ which are at least $\epsilon$ units closer to the target location than the current \textit{hit point}.

\begin{figure}[h]
    \centering
    \includegraphics[width = 0.9\linewidth]{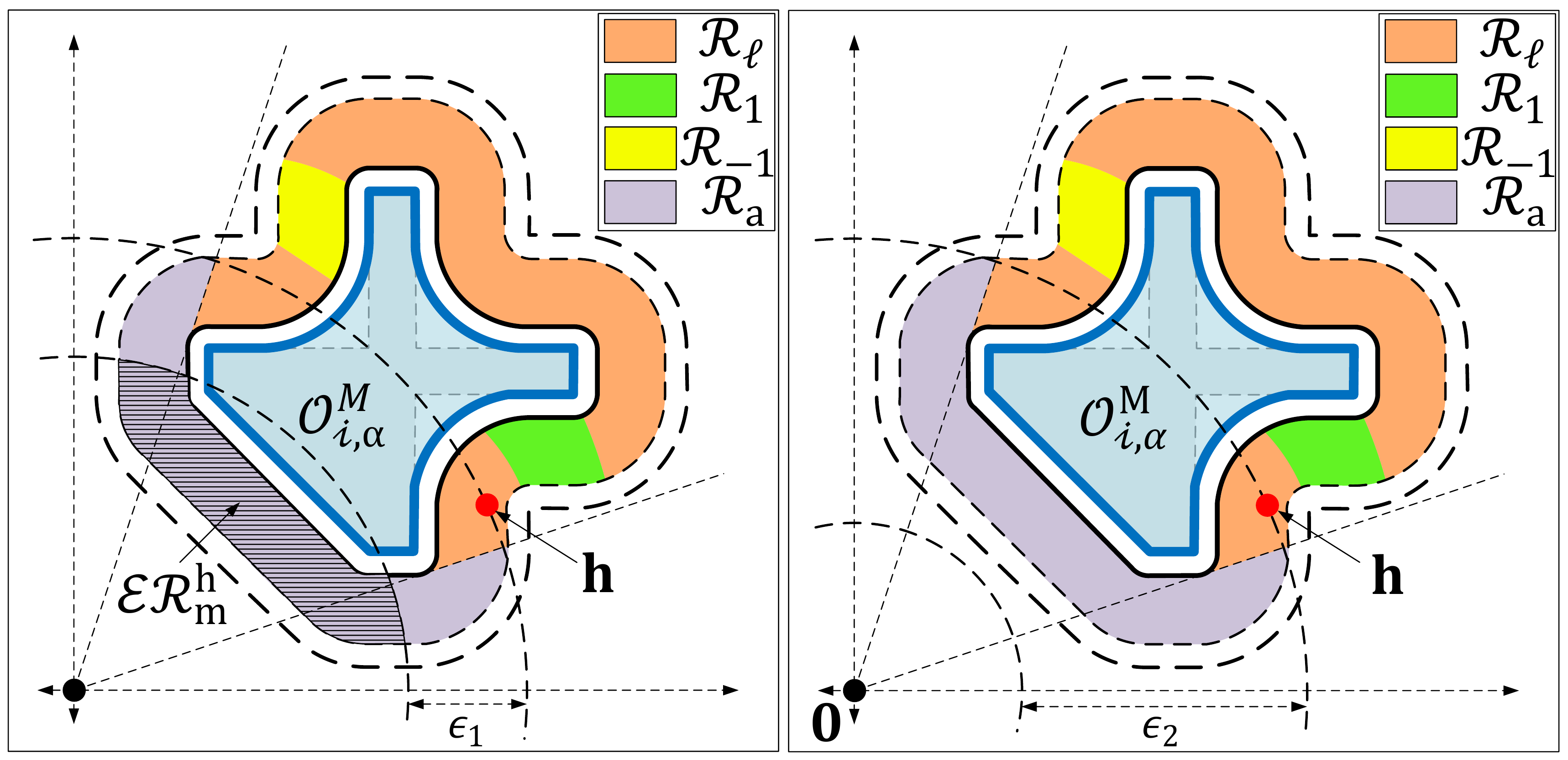}    \caption{Effects of $\epsilon$, used in \eqref{partition_rm}, on the construction of the sets $\mathcal{ER}_{m}^{\mathbf{h}}$, for $m\in\{-1, 1\}$ and a \textit{hit point} $\mathbf{h}\in\mathcal{J}_0^{\mathcal{W}}.$ The left figure shows that when $\epsilon = \epsilon_1$, the sets $\mathcal{ER}_m^{\mathbf{h}}, m\in\{-1, 1\},$ are non-empty. The right figure shows that when $\epsilon = \epsilon_2$, the sets $\mathcal{ER}_m^{\mathbf{h}}, m\in\{-1, 1\},$ are empty.}
    \label{diag:epsilon_diag}
\end{figure}

\begin{figure*}
    \centering
    \includegraphics[width = 0.9\textwidth]{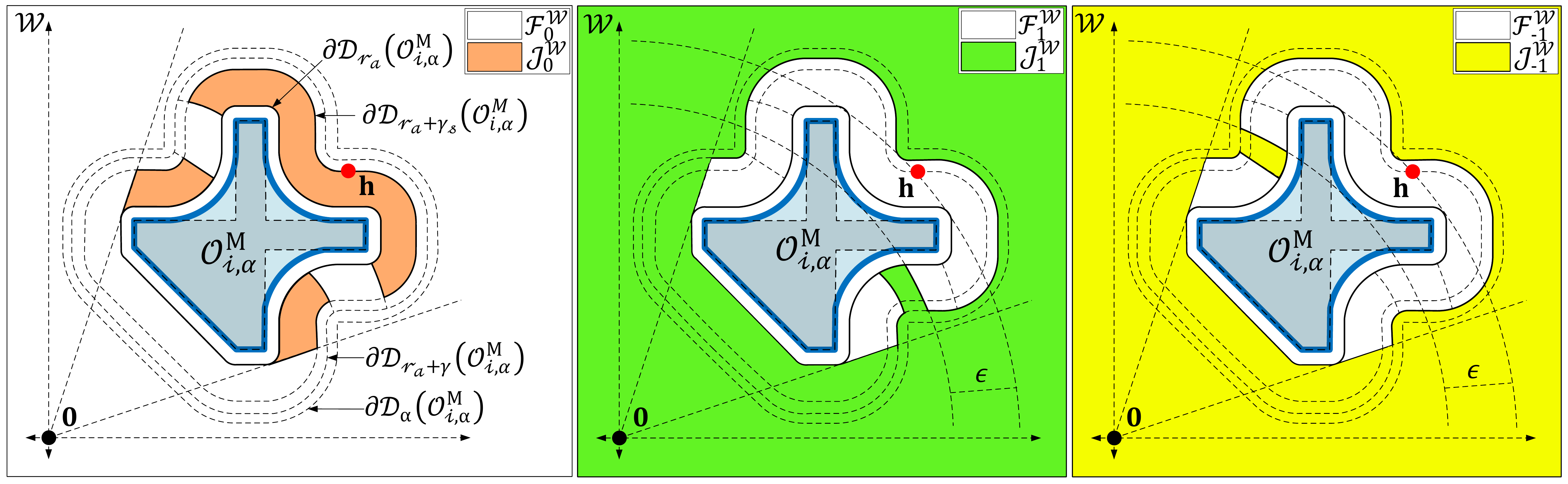}
    \caption{Flow and jump sets $\mathcal{F}_m^{\mathcal{W}}$, $\mathcal{J}_m^{\mathcal{W}}$, related to the set $\mathcal{O}_{i, \alpha}$, shown in Fig. \ref{diagram:exit_region_partitions}.
    The left figure $(m = 0)$ illustrates the case where the robot operates in the \textit{move-to-target} mode and moves straight toward the target location. The middle figure $(m = 1)$ illustrates the case where the robot, operating in the \textit{obstacle-avoidance} mode, moves in the clockwise direction with respect to $\partial\mathcal{O}_{i,  \alpha}^M.$ The right figure $(m = -1)$ illustrates the case where the robot, operating in the \textit{obstacle-avoidance} mode, moves in the counter-clockwise direction with respect to $\partial\mathcal{O}_{i, \alpha}^M.$}
    \label{diagram:flwoandjumpset}
\end{figure*}

Since, according to Lemma \ref{lemma:properties_of_modified_workspace}, the target location belongs to the interior of the modified  obstacle-free workspace w.r.t. the center of the robot \textit{i.e.}, $\mathbf{0}\in(\mathcal{V}_{r_a})^{\circ}$, the existence of a positive scalar $\bar{\epsilon}$ can be guaranteed.
However, if one selects a very high value for $\epsilon$, then for some connected modified obstacles $\mathcal{O}_{i, \alpha}^M, i\in\mathbb{I}$, the set $\mathcal{ER}_m^{\mathbf{h}}\cap\mathcal{N}_{\gamma}(\mathcal{D}_{r_a}(\mathcal{O}_{i, \alpha}^M))$ will become empty, as shown in Fig. \ref{diag:epsilon_diag}, and the robot might get stuck (indefinitely) in the \textit{obstacle-avoidance} mode in the vicinity of those modified obstacles.
Therefore, in the next lemma, we provide an upper bound on the value of $\bar{\epsilon}$.

\begin{lemma}
Under Assumptions \ref{assumption:connected_interior} and \ref{Assumption:reach}, we consider a connected modified obstacle $\mathcal{O}_{i,\alpha}^M, i\in\mathbb{I}$. If $\bar{\epsilon}\in(0, \epsilon_h]$, where
\begin{equation}
\epsilon_h = \sqrt{\left(d(\mathbf{0},\mathcal{O}_{\mathcal{W}}^M)^2 - r_a^2\right)} - \left(d(\mathbf{0},\mathcal{O}_{\mathcal{W}}^M) - r_a\right),\label{epsilon_h}
\end{equation}
then, for every $\mathbf{h}\in\mathcal{J}_0^{\mathcal{W}}$ with $d(\mathbf{h}, \mathcal{O}_{i, \alpha}^M) = \beta\in[r_a, r_a +\gamma]$, and any location $\mathbf{p}\in\mathcal{PJ}(\mathbf{0}, \partial\mathcal{D}_{\beta}(\mathcal{O}_{i, \alpha}^M))$, the set $\mathcal{H}_{\mathbf{p}}:=\mathcal{B}_{\delta}(\mathbf{p})\cap\mathcal{N}_{\gamma}(\mathcal{D}_{r_a}(\mathcal{O}_{i, \alpha}^M))\subset\mathcal{ER}_{m}^{\mathbf{h}},$ for some $m\in\{-1, 1\}$ and $\delta  = \min\{\beta - r_a, d(\mathbf{0},\mathcal{O}_{\mathcal{W}}^M) - r_a\}$, where the set $\mathcal{ER}_m^{\mathbf{h}}$ is defined in \eqref{partition_rm}.
\label{lemma:bar_epsilon}
\end{lemma}
\begin{proof}
See Appendix \ref{proof:bar_epsilon}.
\end{proof}

According to Lemma \ref{lemma:bar_epsilon}, if the \textit{hit point} belongs to the jump set of the \textit{move-to-target} mode associated with a connected modified obstacle $\mathcal{O}_{i, \alpha}^M$ and $\epsilon\in(0, \bar{\epsilon}]$, where $\bar{\epsilon}$ is chosen as per Lemma \ref{lemma:bar_epsilon}, then the set $\mathcal{ER}_m^{\mathbf{h}}\cap\mathcal{N}_{\gamma}(\mathcal{D}_{r_a}(\mathcal{O}_{i, \alpha}^M))$ is non-empty. Hence, we initialize the robot in the \textit{move-to-target} mode so that the \textit{hit point} will always belong to the jump set of the \textit{move-to-target} mode, as stated later in Theorem \ref{theorem:global_stability}.

}

According to \eqref{avoidance_jump_set} and \eqref{partition_rm}, while operating in the \textit{obstacle-avoidance} mode with some $m\in\{-1, 1\}$, the robot can switch to the \textit{move-to-target} mode only when its center belongs to the \textit{exit} region $\mathcal{R}_m\cup\mathcal{R}_a$ and is at least $\epsilon$ units closer to the target than the current \textit{hit point} $\mathbf{h}.$ This creates a hysteresis region and ensures  Zeno-free switching between the modes. This switching strategy is inspired by \cite{kamon1998tangentbug}, which allows us to establish convergence properties of the target location, as discussed later in Theorem \ref{theorem:global_stability}.

We then define the flow set of the \textit{obstacle-avoidance} mode $\mathcal{F}_m^{\mathcal{W}}$ for the state $\mathbf{x}$ as follows:
\begin{equation}
    \mathcal{F}_m^{\mathcal{W}}:= \mathcal{R}_l \cup\overline{\mathcal{R}_{-m}}\cup \overline{\mathcal{R}_{m}\cup\mathcal{R}_a\setminus\mathcal{ER}_m^{\mathbf{h}}} ,\label{avoidance_flow_set}
\end{equation}
where $m\in\{-1, 1\}$. Notice that the union of the jump set \eqref{avoidance_jump_set} and the flow set \eqref{avoidance_flow_set} exactly covers the modified robot-centred free workspace $\mathcal{V}_{r_a}$ \eqref{y_modified_free_workspace}. Refer to Fig. \ref{diagram:flwoandjumpset} for the representation of the flow and jump sets related to the modified obstacle-occupied workspace $\mathcal{O}_{\mathcal{W}}^M$.

Finally, the flow set $\mathcal{F}$ and the jump set $\mathcal{J}$, used in \eqref{hybrid_control_input}, are defined as
\begin{equation}
    \mathcal{F}:= \bigcup_{m\in\mathbb{M}}\mathcal{F}_m, \quad\mathcal{J}:= \bigcup_{m\in\mathbb{M}}\mathcal{J}_m,\label{bothmodes_flowjumpset}
\end{equation}
where for $m\in\mathbb{M},$ the sets $\mathcal{F}_m$ and $\mathcal{J}_m$ are given by
\begin{equation}
    \mathcal{F}_m := \mathcal{F}_{m}^{\mathcal{W}}\times\mathcal{V}_{r_a}\times\{m\}, \quad \mathcal{J}_{m} := \mathcal{J}_{m}^{\mathcal{W}}\times\mathcal{V}_{r_a}\times\{m\},\label{bothmodes_flowjumpset1}
\end{equation}
with $\mathcal{F}_m^{\mathcal{W}}, \mathcal{J}_m^{\mathcal{W}}$ defined in \eqref{stabilization_flow_set}, \eqref{stabilization_jump_set} for $m = 0$ and in \eqref{avoidance_flow_set}, \eqref{avoidance_jump_set} for $m\in\{-1, 1\}.$ 

\begin{figure}
    \centering
    \includegraphics[width = 0.9\linewidth]{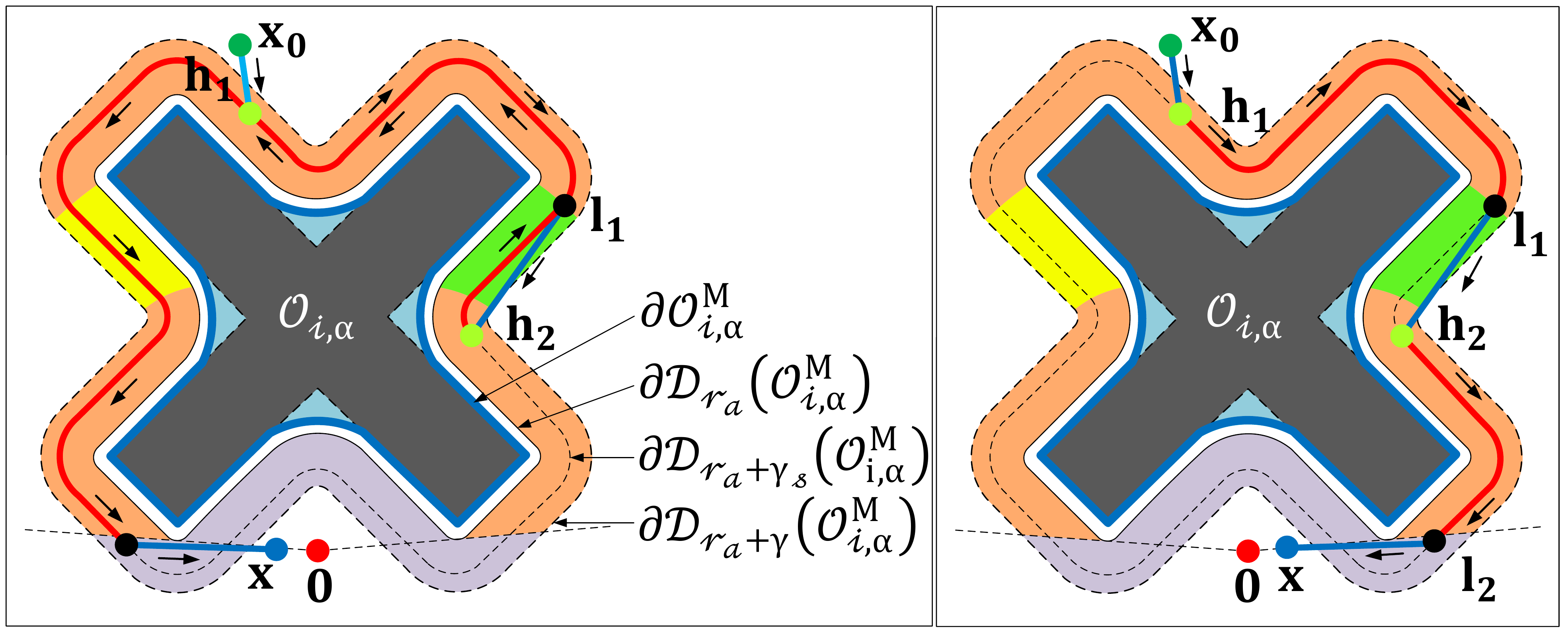}
        \caption{The figures show two possible paths that the robot, initialized at $\mathbf{x}_0$, takes before converging to the origin. The left figure illustrates a case in which the robot, operating in the \textit{obstacle-avoidance} mode, moves initially in the clockwise direction from $\mathbf{h}_1$ to $\mathbf{l}_1$, and then in the counter-clockwise direction from $\mathbf{h}_2$ to $\mathbf{l}_2$ with respect to the nearest point on the modified obstacle $\mathcal{O}_{i, \alpha}^M$. The right figure illustrates a case in which the robot, operating in the \textit{obstacle-avoidance} mode, moves in the clockwise direction with respect to the nearest point on the modified obstacle $\mathcal{O}_{i, \alpha}^M$.}
    \label{diagram:one_direction}
\end{figure}

\begin{remark}
\label{remark:one_direction}
Let us look at the case where the robot is moving in the $\gamma-$neighbourhood of a connected modified obstacle $\mathcal{O}_{i, \alpha}^M$, for some $i\in\mathbb{I}$. If the robot needs to switch between the modes of operation multiple times before leaving the $\gamma-$neighbourhood of the modified obstacle $\mathcal{O}_{i, \alpha}^M$, then it should move in the same direction in the \textit{obstacle-avoidance} mode \textit{i.e.}, either in the clockwise direction or in the counter-clockwise direction, to avoid retracing the previously travelled path, as shown in Fig. \ref{diagram:one_direction}b. In fact, if the robot does not maintain the same direction of motion in the \textit{obstacle-avoidance} mode, while operating in the $\gamma-$neighbourhood of the connected modified obstacle $\mathcal{O}_{i, \alpha}^M$, then it will retraces the previously travelled path, as shown in Fig. \ref{diagram:one_direction}a. 
\end{remark}

Next, we provide the update law $\mathbf{L}(\mathbf{x}, \mathbf{h}, m)$, used in \eqref{hybrid_control_input_2}.

\subsection{Update law $\mathbf{L}(\mathbf{x}, \mathbf{h}, m)$}\label{section:update_law}
The update law $\mathbf{L}(\xi)$, used in \eqref{hybrid_control_input_2}, updates the value of the \textit{hit point} $\mathbf{h}$ and the mode indicator $m$ when the state $(\mathbf{x}, \mathbf{h}, m)$ belongs to the jump set $\mathcal{J}$, which is defined in Section \ref{section:flowset_jumpset_construction}. When the robot, operating in the \textit{move-to-target} mode, enters in the jump set $\mathcal{J}_0$, defined in \eqref{stabilization_jump_set} and \eqref{bothmodes_flowjumpset1}, the update law $\mathbf{L}(\mathbf{x}, \mathbf{h}, 0)$ is given as
\begin{equation}
    \mathbf{L}(\mathbf{x}, \mathbf{h}, 0) = \left\{\begin{bmatrix}\mathbf{x}\\ z\end{bmatrix}\bigg|z\in\{-1, 1\}\right\} .\label{updatelaw_part1}
\end{equation}
Notice that, when the robot switches from the \textit{move-to-target} mode to the \textit{obstacle-avoidance} mode, the coordinates of the \textit{hit point} gets updated to the current value of $\mathbf{x}$. 

On the other hand, when the robot operating in the \textit{obstacle-avoidance} mode, enters in the jump set $\mathcal{J}_m,m\in\{-1, 1\}$, defined in \eqref{avoidance_jump_set} and \eqref{bothmodes_flowjumpset1}, the update law $\mathbf{L}(\mathbf{x}, \mathbf{h}, m), m\in\{-1, 1\}$, is given by
\begin{equation}
    \mathbf{L}(\mathbf{x}, \mathbf{h}, m) = \begin{bmatrix}\mathbf{h}\\ 0\end{bmatrix}.\label{updatelaw_part2}
\end{equation}
When the robot switches from the \textit{obstacle-avoidance} mode to the \textit{move-to-target mode}, the value of the \textit{hit point} remains unchanged. 

\textbf{Control design summary}: The proposed hybrid feedback control law can be summarized as follows:
\begin{itemize}
    \item \textbf{Parameters selection}: the target location is set at the origin with $\mathbf{0}\in\mathcal{W}_{r_a}^{\circ}$. The parameter $\alpha$ is set such that $\alpha > r_a$ and it satisfies the conditions in Lemma \ref{lemma:alpha_existence}.
    The gain parameters $\kappa_s$ and $\kappa_r$ are set to positive values. The parameter $\bar{\epsilon}$, used in \eqref{partition_rm}, is chosen as per Lemma \ref{lemma:bar_epsilon}. The scalar parameter $\gamma$, used in the construction of the flow set $\mathcal{F}$ and the jump set $\mathcal{J}$, is selected such that $\gamma\in(0, \alpha - r_a)$, and the parameter $\gamma_s$ is set to satisfy $0 < \gamma_s < \gamma.$

    \item \textbf{Obstacle modification}: the obstacle reshaping operator \eqref{obstacle_modification_step} is used to obtain the modified obstacle-occupied workspace $\mathcal{O}_{\mathcal{W}}^M$. {The state vector is initialized in the set $\mathcal{K}$ \textit{i.e.}, $\xi(0,0)\in\mathcal{K}.$} 
    
    \item \textbf{\textit{Move-to-target} mode $m = 0$}: this mode is activated when $\xi \in \mathcal{F}_0$. As per \eqref{hybrid_control_input_1}, the control input is given by $\mathbf{u}(\xi) = -\kappa_s\mathbf{x}$, causing $\mathbf{x}$ to evolve along the line segment $\mathcal{L}_s(\mathbf{0}, \mathbf{x})$ towards the origin. If, at some instance of time, $\xi$ enters in the jump set $\mathcal{J}_0$ of the \textit{move-to-target} mode, the state variables $(\mathbf{h}, m)$ are updated using \eqref{updatelaw_part1}, and the control input switches to the \textit{obstacle-avoidance} mode.
    
    \item\textbf{\textit{Obstacle-avoidance} mode $m \in\{-1, 1\}$}: this mode is activated when $\xi \in \mathcal{F}_m$ for some $m\in\{-1, 1\}$. As per \eqref{hybrid_control_input_1}, the control input is given by $\mathbf{u}(\xi) = \kappa_r\mathbf{v}(\mathbf{x},m)$, causing $\mathbf{x}$ to evolve in the $\gamma-$neighborhood of the nearest modified obstacle until the state $\xi$ enters in the jump set $\mathcal{J}_m$, $m\in\{-1, 1\}$. When $\xi\in\mathcal{J}_m$, $m\in\{-1, 1\}$, the control input switches to the \textit{move-to-target} mode by setting $m = 0$, as per \eqref{updatelaw_part2}.
\end{itemize}

This concludes the design of the proposed hybrid feedback controller \eqref{hybrid_control_input}.

\section{Stability Analysis}\label{sec:stabilitY_analysis}
The hybrid closed-loop system resulting from the hybrid control law \eqref{hybrid_control_input} is given by
\begin{equation}
    \underbrace{\begin{matrix}
    \mathbf{\dot{x}}\\\mathbf{\dot{h}}\\\dot{m}
    \end{matrix}\begin{matrix*}[l]=\mathbf{u}(\xi)\\=\mathbf{0}\\=0\end{matrix*}}_{\dot{\xi} = \mathbf{F}(\xi)}, {\xi\in\mathcal{F}},\quad\underbrace{\begin{matrix}\mathbf{x}^+\\\begin{bmatrix}\mathbf{h}^+\\m^+\end{bmatrix}\end{matrix} \begin{matrix*}[l]=\mathbf{x}\vspace{0.16cm}\\\vspace{0.3cm}\in\mathbf{L}(\xi)\end{matrix*}}_{\xi^+ \in\mathbf{J}(\xi)}, {\xi\in\mathcal{J}},\label{hybrid_closed_loop_system}
\end{equation}
where $\mathbf{u}(\xi)$ is defined in \eqref{hybrid_control_input_1}, and the update law $\mathbf{L}(\xi)$ is provided in \eqref{updatelaw_part1}, \eqref{updatelaw_part2}. The definitions of the flow set $\mathcal{F}$ and the jump set $\mathcal{J}$ are provided in \eqref{bothmodes_flowjumpset}, \eqref{bothmodes_flowjumpset1}. Next, we analyze the hybrid closed-loop system \eqref{hybrid_closed_loop_system} in terms of the forward invariance of the obstacle-free state space $\mathcal{K}:=\mathcal{V}_{r_a}\times\mathcal{V}_{r_a}\times\mathbb{M}$ along with the stability properties of the target set $\mathcal{A}$, which is defined as
\begin{equation}
    \mathcal{A}:= \{\mathbf{0}\}\times\mathcal{V}_{r_a}\times\mathbb{M}\label{target_set}.
\end{equation}
First, we analyze the forward invariance of the modified obstacle-free workspace, which then will be followed by the convergence analysis.

For safe autonomous navigation, the state $\mathbf{x}$ must always evolve within the set $\mathcal{V}_{r_a}$ \eqref{y_modified_free_workspace}.
This is equivalent to having the set $\mathcal{K}:=\mathcal{V}_{r_a}\times\mathcal{V}_{r_a}\times\mathbb{M}$  forward invariant for the hybrid closed-loop system \eqref{hybrid_closed_loop_system}. This is stated in the next Lemma.

\begin{lemma}
Under Assumptions \ref{assumption:connected_interior} and \ref{Assumption:reach}, for the hybrid closed-loop system \eqref{hybrid_closed_loop_system}, the obstacle-free set $\mathcal{K}:= \mathcal{V}_{r_a}\times\mathcal{V}_{r_a}\times\mathbb{M}$ is forward invariant.\label{lemma:forward_invariance}
\end{lemma}
\begin{proof}
See Appendix \ref{proof:forward_invariance}.
\end{proof}

Next, we show that if the robot is initialized in the \textit{move-to-target} mode, at any location in $\mathcal{V}_{r_a}$ and the parameter $\bar{\epsilon}$, used in \eqref{partition_rm}, is chosen as per Lemma \ref{lemma:bar_epsilon}, then it will safely and asymptotically converge to the target location at the origin.

\begin{theorem}\label{theorem:global_stability}
Consider the hybrid closed-loop system \eqref{hybrid_closed_loop_system} and let Assumption \ref{assumption:connected_interior} holds true. Also, let Assumption \ref{Assumption:reach} hold true for the parameter $\alpha$ chosen as per Lemma \ref{lemma:alpha_existence}. If the parameter $\bar{\epsilon}$, used in \eqref{partition_rm}, is chosen as per Lemma \ref{lemma:bar_epsilon}, then
\begin{itemize}
\item[i)] the obstacle-free set $\mathcal{K}$ is forward invariant,
\item[ii)] the set $\mathcal{A}$ is stable and attractive from all initial conditions $\xi(0,0) \in \mathcal{V}_{r_a}\times\mathcal{V}_{r_a}\times\{0\},$
\item[iii)] the number of jumps is finite.
\end{itemize}
\end{theorem}
\begin{proof}
See Appendix \ref{proof:global_stability}.
\end{proof}

According to Theorem \ref{theorem:global_stability}, we initialize the robot in the \textit{move-to-target} mode to ensure that when it switches to the \textit{obstacle-avoidance} mode, the \textit{hit point} $\mathbf{h}$ belongs to the set $\mathcal{J}_{0}^{\mathcal{W}}.$ This allows us to establish an upper bound on the value of the parameter $\bar{\epsilon}$ as given in Lemma \ref{lemma:bar_epsilon}, which is crucial to ensure that the robot, when operating in the \textit{obstacle-avoidance} mode, will certainly enter in the \textit{move-to-target} mode.

{
\begin{remark}
Theorem \ref{theorem:global_stability} guarantees global asymptotic stability of the target location in the modified set $\mathcal{V}_{r_a}$ and not in the original set $\mathcal{W}_{r_a}$. Since the obstacle reshaping operator $\mathbf{M}$ is \textit{extensive} \cite[Table 1]{serra1986introduction} \textit{i.e.}, $\mathcal{O}_{\mathcal{W}}\subseteq\mathcal{O}_{\mathcal{W}}^M$, the set $\mathcal{V}_{r_a}$ is a subset of the set $\mathcal{W}_{r_a}$ \textit{i.e.}, $\mathcal{V}_{r_a}\subseteq\mathcal{W}_{r_a}.$ Interestingly, if one chooses the value of the parameter $\alpha$ close to $r_a$, then the region occupied by the set $\mathcal{V}_{r_a}$ approaches the original set $\mathcal{W}_{r_a}$. {In other words, if $\mathcal{V}_{r_a}^1$ and $\mathcal{V}_{r_a}^2$ are two modified obstacle-free workspaces obtained for two different values of the parameter $\alpha_1$ and $\alpha_2$, respectively, using \eqref{obstacle_modification_step} and \eqref{y_modified_free_workspace}, where $\alpha_1, \alpha_2\in(r_a, \bar{\alpha}]$, $ \alpha_1 > \alpha_2$ and $\bar{\alpha}$ defined as per Lemma \ref{lemma:alpha_existence}, then the set $\mathcal{V}_{r_a}^1\subseteq\mathcal{V}_{r_a}^2\subseteq\mathcal{W}_{r_a}$.} Hence, by selecting a smaller value of the parameter $\alpha$ one can implement the proposed hybrid feedback controller \eqref{hybrid_control_input} in a larger area.

\label{remark:larger_area}
\end{remark}

Unlike \cite[Assumprtion 1]{arslan2019sensor}, \cite[Assumprtion 2]{verginis2021adaptive}, and \cite[Section V-C3]{berkane2021obstacle}, Assumptions \ref{assumption:connected_interior} and \ref{Assumption:reach} do not impose restrictions on the minimum separation between any pair of obstacles and allow obstacles to be non-convex. In particular, Assumptions \ref{assumption:connected_interior} and \ref{Assumption:reach} are satisfied
in the case of environments with convex obstacles where the minimum separation between any pair of obstacles is greater than $2r_a$, as discussed next.
\begin{proposition}
Let the workspace $\mathcal{W}$ be a compact, convex subset of $\mathbb{R}^2$. Let the obstacles $\mathcal{O}_i, i\in\mathbb{I}\setminus\{0\}$, be compact and convex, such that $d(\mathcal{O}_i, \mathcal{O}_j) > 2r_a, \forall i,j\in\mathbb{I}, i\ne j$, and $\mathbf{0}\in\mathcal{W}_{r_a}^{\circ}$. Then Assumptions \ref{assumption:connected_interior} and \ref{Assumption:reach} hold true for all $\alpha\in(r_a, \bar{\alpha}]$, where the parameter $\bar{\alpha}$ is defined as
\begin{equation}
    \bar{\alpha} = \underset{i, j\in\mathbb{I}, i\ne j}{\min}d(\mathcal{O}_i, \mathcal{O}_j)/2.\label{bar_alpha_special_formula}
\end{equation}
\label{lemma:assumption_equivalence}
\end{proposition}
\begin{proof}
See Appendix \ref{proof:proposition}.
\end{proof}

The workspace that satisfies the conditions (commonly used in the literature) in  Proposition \ref{lemma:assumption_equivalence}, also satisfies Assumptions \ref{assumption:connected_interior} and \ref{Assumption:reach}. 
Notice that, since the internal obstacles are convex, if one chooses $\alpha\in(0, \bar{\alpha}]$, which is used in \eqref{obstacle_modification_step}, where $\bar{\alpha}$ is defined in \eqref{bar_alpha_special_formula}, then the shapes of the internal convex obstacles remains unchanged in the modified obstacle-occupied workspace. Hence, for any $\alpha\in(r_a, \bar{\alpha}]$, the set of locations that do not belong to the modified set $\mathcal{V}_{r_a}$ always belongs inside the $\alpha-$neighbourhood of the workspace obstacle $\mathcal{O}_0$ \textit{i.e.}, the set $\mathcal{W}_{r_a}\setminus\mathcal{V}_{r_a}\subset\mathcal{D}_{\alpha}(\mathcal{O}_0).$ 

However, as the workspace is convex, if the robot is initialized in the \textit{move-to-target} mode, in the set $\mathcal{W}_{r_a}\setminus\mathcal{V}_{r_a}$, it will initially move straight towards the target and enter the set $\mathcal{V}_{r_a}$. Then, according to Lemma \ref{lemma:forward_invariance}, the robot will continue to move inside the set $\mathcal{V}_{r_a}$ and according to Theorem \ref{theorem:global_stability}, will asymptotically converge to the target location.


Next, we provide procedural steps to implement the proposed hybrid feedback controller \eqref{hybrid_control_input} for safe autonomous navigation in \textit{a priori} known and \textit{a priori} unknown environments.}

\section{Sensor-based Implementation Procedure}
\label{section:sensor-based_implementation}
We choose the origin as the target location. 
We initialize the center of the robot in the interior of the set $\mathcal{W}_{r_a}$ and assume that the value of parameter $\alpha$, defined in Lemma \ref{lemma:alpha_existence}, is \textit{a priori} known. The robot is initialized in the \textit{move-to-target} mode \textit{i.e.}, $m(0, 0) = 0$, as stated in Theorem \ref{theorem:global_stability}, and the \textit{hit point} is initialized at the initial location of the robot. We choose $\epsilon\in(0, \bar{\epsilon}],$ where $\bar{\epsilon}$ is selected as per Lemma \ref{lemma:bar_epsilon}. The obstacles can have arbitrary shapes and can be in close proximity with each other as long as Assumptions \ref{assumption:connected_interior} and \ref{Assumption:reach} are satisfied.

Notice that the robot can have multiple closest points in the proximity of non-convex obstacles, in which case, the obstacle avoidance term $\mathbf{v}(\mathbf{x}, m)$, defined in \eqref{definition:vxm}, is not viable, since it requires a unique closest point. 
Moreover, in an unknown environment, the modified obstacle-occupied workspace cannot be obtained in advance. Therefore, motivated by the method described in Remark \ref{remark:virtual_ring}, a virtual ring $\partial\mathcal{B}_{v_r}(\mathbf{c})$ is constructed whenever the robot enters the \textit{obstacle-avoidance} mode, as described in Section \ref{section:switching_to_obstacle_avoidance}. One should ensure that the robot's body is always enclosed by the ring, that the ring does not intersect with the interior of the obstacle-occupied workspace $\mathcal{O}_{\mathcal{W}}$, and that the ring moves along with the robot in the \textit{obstacle-avoidance} mode. Using this ring, the robot can then anticipate the possibility of multiple projections of its center onto the nearest obstacle and locally modify it to ensure that the projection of its center onto the modified obstacle is always unique, as discussed later in Section \ref{section:moving_in_the_obstacle_avoidance_mode}.

Note that if the robot is initialized on the boundary of the set $\mathcal{O}_{\mathcal{W}}$ such that its center has multiple closest points on $\mathcal{O}_{\mathcal{W}}$, then one cannot construct a virtual ring with a radius greater than $r_a$ that not only encloses the robot's body but also does not intersect with the interior of the set $\mathcal{O}_{\mathcal{W}}$. Consequently, one should initialize the robot in the interior of the obstacle-free workspace.

For safe navigation in \textit{a priori} unknown environments, we assume that the robot is equipped with a range-bearing sensor with angular scanning range of $360^{\circ}$ and sensing radius $R_s>2\alpha.$ Similar to \cite{sawant2021hybrid} and \cite{berkane2021navigation}, the range-bearing sensor is modeled using a polar curve $r_g(\mathbf{x}, \theta):\mathcal{W}_{r_a}\times[-\pi, \pi]\to[0, R_s],$ which is defined as
\begin{equation}
    r_g(\mathbf{x}, \theta) = \min\left\{R_s, \underset{\begin{matrix}\mathbf{y}\in\partial\mathcal{O}_{\mathcal{W}}\\\text{atan2v}(\mathbf{y} - \mathbf{x}) = \theta\end{matrix}}{\min}{\norm{\mathbf{x} - \mathbf{y}}}\right\}.\label{range_evalution}
\end{equation}
where $\text{atan2v}(\mathbf{q}) = \text{ atan2}(q_2, q_1), \mathbf{q} = [q_1, q_2]^\intercal$. The notation $r_g(\mathbf{x}, \theta)$ represents the distance between the center of the robot $\mathbf{x}$ and the boundary of the unsafe region $\partial\mathcal{O}_{\mathcal{W}}$, measured by the sensor, in the direction defined by the angle $\theta.$ Given the location $\mathbf{x},$ along with the bearing angle $\theta$, the mapping $\lambda(\mathbf{x}, \theta):\mathcal{W}_{r_a}\times[-\pi, \pi]\to\mathcal{W}_{0},$ which is given by
\begin{equation}
    \lambda(\mathbf{x}, \theta) = \mathbf{x} + r_g(\mathbf{x}, \theta)[\cos(\theta), \sin(\theta)]^{\intercal},\label{coordinate_evalution}
\end{equation}
evaluates the Cartesian coordinates of the detected point.

{Using \eqref{range_evalution} and \eqref{coordinate_evalution}, the distance between the center of the robot $\mathbf{x}\in\mathcal{W}_{r_a}$ and the unsafe region $\mathcal{O}_{\mathcal{W}}$ \textit{i.e.,} $d(\mathbf{x}, \mathcal{O}_{\mathcal{W}})$ is calculated as follows: 
\begin{equation}
    d(\mathbf{x}, \mathcal{O}_{\mathcal{W}}) = r_g(\mathbf{x}, \theta),\label{sensor:closest_distance}
\end{equation}
where $\theta\in\Theta$. The set $\Theta$, which is defined as
\begin{equation}
    \Theta = \left\{\theta_p\in[-\pi, \pi]\bigg|\theta_p =\underset{\theta\in[-\pi, \pi]}{\text{arg min }}r_g(\mathbf{x}, \theta)\right\},
\end{equation}
contains bearing angles such that the range measurement \eqref{range_evalution} in the directions defined by these bearing angles gives the smallest value when compared to the values obtained in any other directions. Then, the set of projections of the location $\mathbf{x}$ onto the unsafe regions \textit{i.e.,} $\mathcal{PJ}(\mathbf{x}, \mathcal{O}_{\mathcal{W}})$ is given by
\begin{equation}
    \mathcal{PJ}(\mathbf{x}, \mathcal{O}_{\mathcal{W}}) = \left\{\lambda(\mathbf{x}, \theta)|\forall\theta\in\Theta\right\}.\label{sensor:closest_point}
\end{equation}

For a given location of the robot $\mathbf{x}$, the set $\partial\mathcal{O}$, which is defined as
\begin{equation}
    \partial\mathcal{O} = \{\lambda(\mathbf{x}, \theta), \theta\in[-\pi, \pi]|r_g(\mathbf{x}, \theta) < R_s\},
\end{equation}
contains the locations in the sensing region that belong to the boundary of the unsafe region. The robot moving straight towards the target will collide with the obstacles if the following condition holds true:
\begin{equation}
    \partial\mathcal{O} \cap\square(\mathbf{x}, \mathbf{0}) \ne \emptyset,\label{sensor:observed_boundary}
\end{equation}
where the notation $\square(\mathbf{x}, \mathbf{0})$ represents a rectangle, as shown in Fig. \ref{diagram:square_landing}, with its vertices located at $\mathbf{x}_1, \mathbf{x}_{-1}, \mathbf{0}_1$ and $\mathbf{0}_{-1}$, which are evaluated as
\begin{equation}
    \begin{aligned}
    \begin{bmatrix}
    \mathbf{x}_z\\
    \mathbf{0}_z
    \end{bmatrix} = \begin{bmatrix}
    \mathbf{x}\\
    \mathbf{0}
    \end{bmatrix} + zr_a\begin{bmatrix}\mathbf{I}\\\mathbf{I}\end{bmatrix}\begin{bmatrix}\cos\theta_d\\\sin\theta_d\end{bmatrix}, z\in\{-1, 1\},
    \end{aligned}\label{vertices_of_rectangle}
\end{equation}
where $\mathbf{I}$ is a $2\times2 $ identity matrix, and $\theta_d = \pi/2 + \text{atan2v}(\mathbf{x})$ The robot moving in the \textit{move-to-target} mode can infer the possibility of a collision with the unsafe region by verifying the condition in \eqref{sensor:observed_boundary}. Next, we provide a procedure, summarized in Algorithm 2, which allows robot to identify whether the state $(\mathbf{x}, \mathbf{h}, m)$ belongs to the jump set or not.

\begin{figure}
    \centering
    \includegraphics[width = 0.6\linewidth]{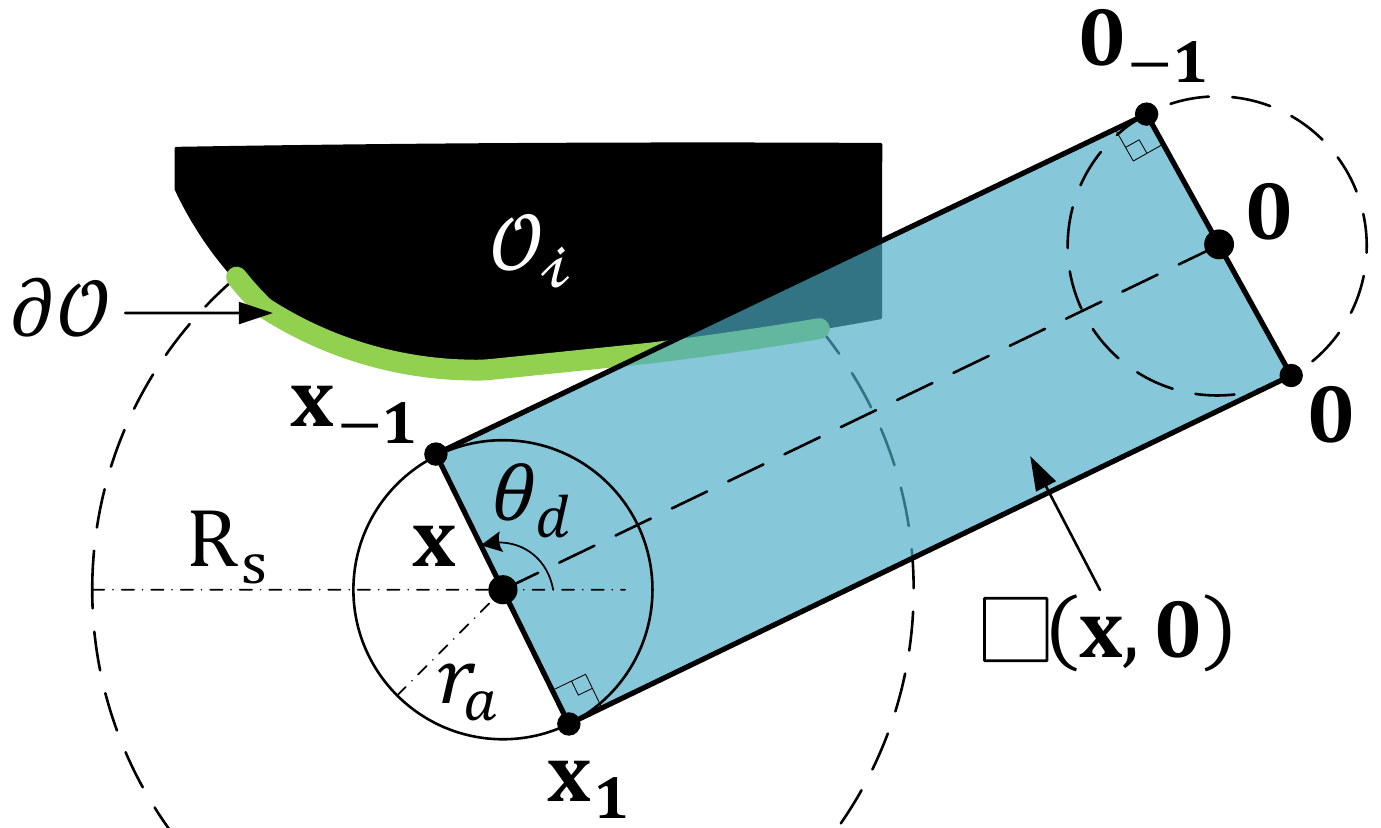}
    \caption{Construction of the rectangle $\square(\mathbf{x}, \mathbf{0})$ based on the location of the robot and the target location at the origin.}
    \label{diagram:square_landing}
\end{figure}

\subsection{Switching to the \textit{obstacle-avoidance} mode}
\label{section:switching_to_obstacle_avoidance}
Since the robot is initialized in the \textit{move-to-target} mode, it will initially move towards the target, under the influence of the stabilizing control vector $-\kappa_s\mathbf{x}, \kappa_s>0.$ Suppose, there exists an obstacle such that the line segment $\mathcal{L}_{s}(\mathbf{x}, \mathbf{0})$ intersects with the \textit{landing region} \textit{i.e.}, $\mathcal{L}_{s}(\mathbf{x}, \mathbf{0})\cap\mathcal{R}_{l} \ne\emptyset. $
 This can be identified by evaluating the inner product between the vectors $\mathbf{x}$ and $\mathbf{x} - \Pi(\mathbf{x}, \mathcal{O}_{\mathcal{W}})$, according to \eqref{definition:individual_landing_region} and \eqref{sensor:closest_point}, and by verifying the condition given in \eqref{sensor:observed_boundary}. Eventually, the robot moving straight towards the target will enter the $\beta-$neighbourhood of the obstacle-occupied workspace, where $\beta\in(r_a, \alpha)$, \textit{i.e.,} $d(\mathbf{x}, \mathcal{O}_{\mathcal{W}}) = \beta$ such that one of the following two cases holds:

\textbf{Case A: }there is a unique projection of the robot's center onto the obstacle-occupied workspace \textit{i.e.}, $\bf{card}(\mathcal{PJ}(\mathbf{x}, \mathcal{O}_{\mathcal{W}})) = 1.$

\textbf{Case B: }there are more than one projections of the robot's center onto the obstacle-occupied workspace \textit{i.e.}, $\bf{card}(\mathcal{PJ}(\mathbf{x}, \mathcal{O}_{\mathcal{W}})) > 1$, and $-\mathbf{x}\in\mathcal{CH}(\mathbf{x}, \mathcal{PJ}(\mathbf{x}, \mathcal{O}_{\mathcal{W}}))$.

First, we consider case A. Since $\mathcal{L}_{s}(\mathbf{x}, \mathbf{0})\cap\mathcal{R}_{l} \ne\emptyset$, the robot has to switch to the \textit{obstacle-avoidance} mode. However, before that, it needs to construct a virtual ring $\partial\mathcal{B}_{v_r}(\mathbf{c})$ to locally modify the set $\mathcal{O}_{\mathcal{W}}$ to ensure the uniqueness of the projection of its center onto the unsafe region. We locate the center of the virtual ring $\partial\mathcal{B}_{v_r}(\mathbf{c})$ \textit{i.e.}, $\mathbf{c}$ using the following formula:
\begin{equation}
    \mathbf{c} = \Pi(\mathbf{x}, \mathcal{O}_{\mathcal{W}}) + (r_a + \gamma)\frac{\mathbf{x} - \Pi(\mathbf{x}, \mathcal{O}_{\mathcal{W}})}{\norm{\mathbf{x} - \Pi(\mathbf{x}, \mathcal{O}_{\mathcal{W}})}},\label{sensor_c_location}
\end{equation}
where $\gamma\in(\beta - r_a, \alpha - r_a)$ such that $\mathbf{card}(\mathcal{PJ}(\mathbf{c}, \mathcal{O}_{\mathcal{W}})) = 1$ and $\Pi(\mathbf{c}, \mathcal{O}_{\mathcal{W}}) = \Pi(\mathbf{x}, \mathcal{O}_{\mathcal{W}})$. Then, the radius $v_r$ of the virtual ring $\partial\mathcal{B}_{v_r}(\mathbf{c})$ is set to $r_a + \gamma$. The robot then sets $\gamma_s=\beta - r_a$,  and enters in the \textit{obstacle-avoidance} mode \textit{i.e.}, switches $m$ to $+1$ or $-1$. At this instance, we assign the current location as the \textit{hit point} $\mathbf{h}$, as per \eqref{updatelaw_part1}. Case A is illustrated in Fig. \ref{stabilization_to_avoidance}a.

Now, we consider case B. We set $\mathbf{c}$ to be the current location of the robot's center \textit{i.e.}, $\mathbf{c} = \mathbf{x}$, and $v_r = \beta.$
Since the robot has multiple projections on the set $\mathcal{O}_{\mathcal{W}}$, it indicates the presence of a non-convex obstacle in its immediate neighbourhood, as shown in Fig. \ref{stabilization_to_avoidance}b. Hence, to ensure the uniqueness of the projection of the center of the robot onto the unsafe region, we augment the boundary of the obstacle-occupied workspace with a curve $\mathcal{Y}$, which is defined as
\begin{equation}
    \mathcal{Y} = \partial\mathcal{B}_{v_r}(\mathbf{c})\cap\mathcal{CH}(\mathbf{c}, \mathcal{PJ}(\mathbf{c}, \mathcal{O}_{\mathcal{W}})).\label{partial_ring}
\end{equation}
The curve $\mathcal{Y}$ is the section of the virtual ring $\partial\mathcal{B}_{v_r}(\mathbf{c})$ that belongs to the conic hull $\mathcal{CH}(\mathbf{c}, \mathcal{PJ}(\mathbf{c}, \mathcal{O}_{\mathcal{W}})).$ Notice that the curve $\mathcal{Y}$ belongs to the boundary of the modified obstacle $\mathbf{M}(\mathcal{O}_{\mathcal{W}}, v_r)$, as per Remark \ref{remark:virtual_ring}. We treat this curve as a part of the boundary of the unsafe region \textit{i.e.}, $\partial\mathcal{O}_{\mathcal{W}} \leftarrow \partial\mathcal{O}_{\mathcal{W}}\cup\mathcal{Y}.$

The robot has not yet switched to the \textit{obstacle-avoidance} mode, and is moving straight towards the target inside the previously constructed virtual ring $\partial\mathcal{B}_{v_r}(\mathbf{c})$, along the line segment $\mathcal{L}_s(\mathbf{c}, \mathbf{0})$. Since $\mathcal{B}_{r_a}(\mathbf{x})\subset\mathcal{B}_{v_r}(\mathbf{c}),$ after moving straight towards the target, the robot will have unique projection on the curve $\mathcal{Y}$ and hence on the unsafe region $\mathcal{O}_{\mathcal{W}}$. Then the robot will switch to the \textit{obstacle-avoidance} mode, according to case A.
}

\begin{figure}
    \centering
    \includegraphics[width = 0.9\linewidth]{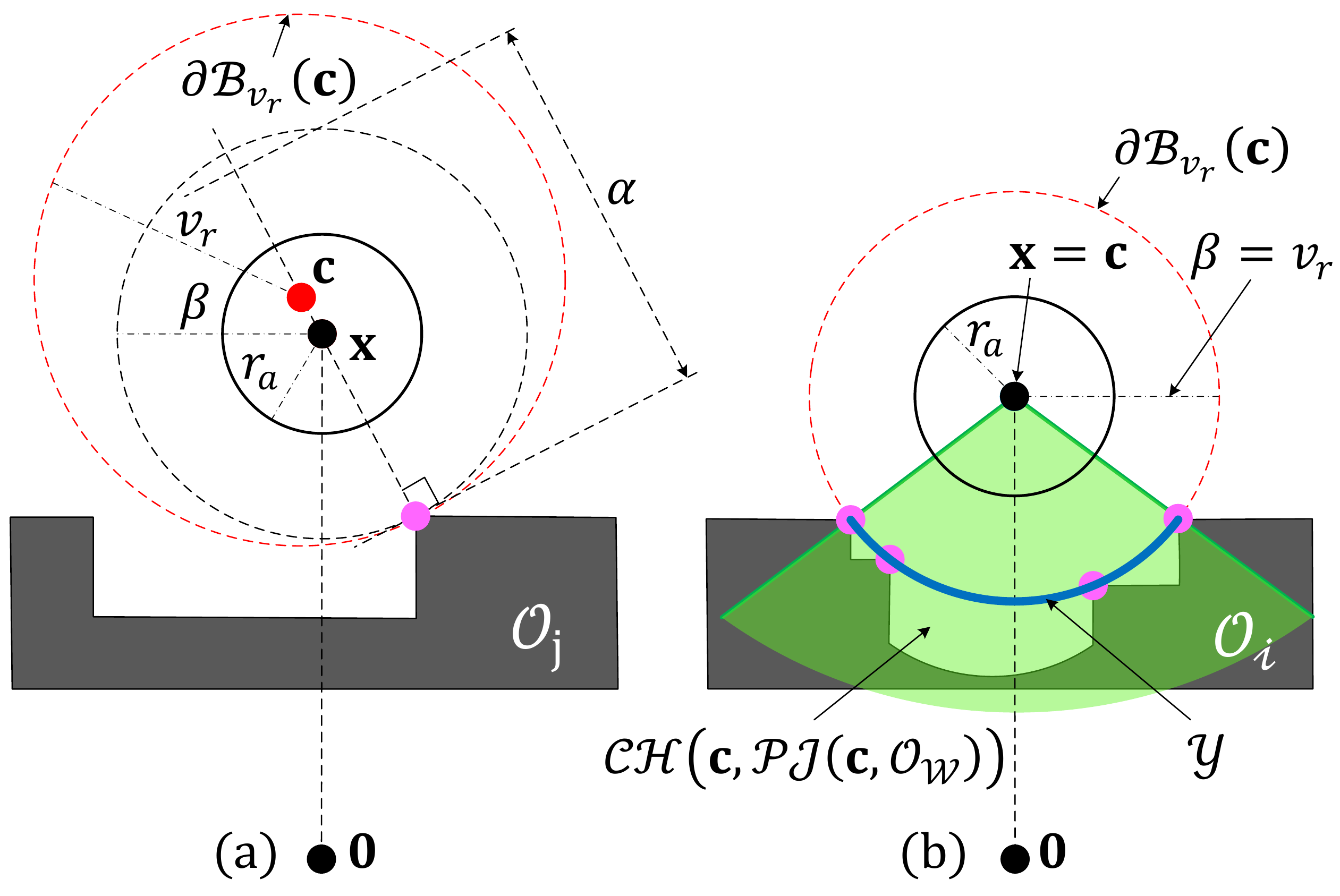}
    \caption{Illustration of two possible situations that can occur when the robot, operating in the \textit{move-to-target} mode, enters in the $\beta-$neighborhood of the obstacle-occupied workspace $\mathcal{O}_{\mathcal{W}}$, where $\beta \in (r_a, \alpha)$. (a) Situation when $\mathbf{card}(\mathcal{PJ}(\mathbf{x}, \mathcal{O}_{\mathcal{W}})) = 1$. (b) Situation when $\mathbf{card}(\mathcal{PJ}(\mathbf{x}, \mathcal{O}_{\mathcal{W}})) > 1.$
}
    \label{stabilization_to_avoidance}
\end{figure}

\subsection{Moving in the \textit{obstacle-avoidance} mode}
\label{section:moving_in_the_obstacle_avoidance_mode}
{We use a virtual ring $\partial\mathcal{B}_{v_r}(\mathbf{c})$ to ensure a unique projection in the obstacle-avoidance mode. This ring anticipates multiple projections and enables local modification of the obstacle-occupied workspace to maintain uniqueness of the projection of the robot's center.


Note that when the robot switches from the \textit{move-to-target} mode to the \textit{obstacle-avoidance} mode, it is enclosed by the virtual ring \textit{i.e.}, $\mathcal{B}_{r_a}(\mathbf{x})\subset\mathcal{B}_{v_r}(\mathbf{c}).$ Hence, if the virtual ring $\partial\mathcal{B}_{v_r}(\mathbf{c})$ touches the obstacle-occupied workspace $\mathcal{O}_{\mathcal{W}}$ at only one location \textit{i.e.}, $\mathbf{card}(\mathcal{PJ}(\mathbf{c}, \mathcal{O}_{\mathcal{W}})) = 1$, then $\Pi(\mathbf{x}, \mathcal{O}_{\mathcal{W}}) = \Pi(\mathbf{c}, \mathcal{O}_{\mathcal{W}}).$ Then, the robot can successfully implement the rotational control vector $\mathbf{v}(\mathbf{x}, m)$. To ensure that the robot's body is always enclosed by the virtual ring, we update the location of the center $\mathbf{c}$ as follows:
\begin{equation}
    \mathbf{c} = \Pi(\mathbf{x}, \mathcal{O}_{\mathcal{W}}) + v_r\frac{\mathbf{x} - \Pi(\mathbf{x}, \mathcal{O}_{\mathcal{W}})}{\norm{\mathbf{x} - \Pi(\mathbf{x}, \mathcal{O}_{\mathcal{W}})}},\label{sensor:c_location}
\end{equation}
where $v_r$ is defined when the robot switches from the \textit{move-to-target} mode to the current \textit{obstacle-avoidance} mode, as discussed in Section \ref{section:switching_to_obstacle_avoidance}.

When the virtual ring touches obstacles at multiple locations, it indicates the presence of a non-convex obstacle in the immediate neighbourhood of the robot, as shown in Fig. \ref{motion_avoidance}. In this case, the robot should use the projection of its center onto the part of the ring that intersects with the conic hull $\mathcal{CH}(\mathbf{c},\mathcal{PJ}(\mathbf{c}, \mathcal{O}_{\mathcal{W}}))$ \textit{i.e.,} onto the set $\mathcal{Y}$, defined in \eqref{partial_ring}, as the closest point. Note that since $\mathcal{B}_{r_a}(\mathbf{x})\subset\mathcal{B}_{v_r}(\mathbf{c}),$ the projection $\Pi(\mathbf{x}, \mathcal{Y})$, which is used to implement the rotational control vector $\mathbf{v}(\mathbf{x}, m)$, is unique.


\begin{figure}[ht]
    \centering
    \includegraphics[width = 0.43\linewidth]{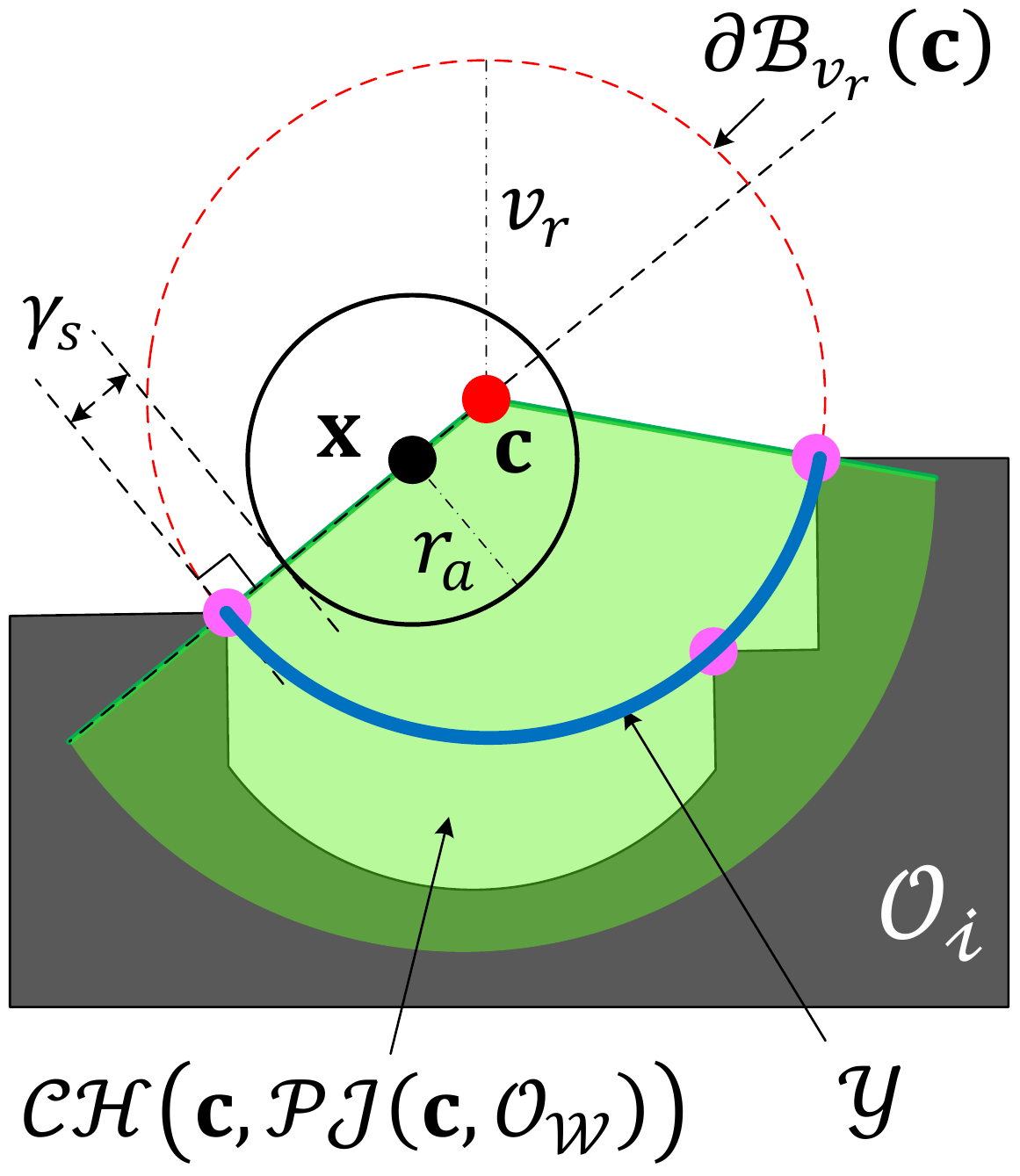}
    \caption{A scenario in which, for the robot operating in the \textit{obstacle-avoidance} mode, the virtual ring $\partial\mathcal{B}_{v_r}(\mathbf{c})$ encounters more than one intersection point with the obstacle-occupied workspace $\mathcal{O}_{\mathcal{W}}$ \textit{i.e.}, $\mathbf{card}(\mathcal{PJ}(\mathbf{c}, \mathcal{O}_{\mathcal{W}})) > 1.$}
    \label{motion_avoidance}
\end{figure}

\subsection{Switching to the \textit{move-to-target} mode}
}

When the robot, operating in the \textit{obstacle-avoidance} mode, with some $m\in\{-1, 1\}$, reaches the location $\mathbf{x}$, which is $\epsilon$ units closer to the target location than the current \textit{hit point} $\mathbf{h}$, and belongs to the \textit{exit} region $\mathcal{R}_m\cup\mathcal{R}_a$, defined in \eqref{always_exit_region}, \eqref{clockwise_exit_region} and \eqref{counter_clockwise_exit_region}, it switches to the \textit{move-to-target} mode by setting $m = 0$.

\begin{algorithm}
\caption{General implementation of the proposed hybrid control law \eqref{hybrid_control_input}.}

\begin{algorithmic}[1] 
\STATE \textbf{Set} target location at the origin $\mathbf{0}.$
\STATE \textbf{Initialize} $\mathbf{x}(0, 0)\in\mathcal{W}_{r_a}$, $\mathbf{h}(0, 0) = \mathbf{x}(0, 0)$ and $m(0, 0) = 0$. Choose $\bar{\epsilon}$ according to Lemma \ref{lemma:bar_epsilon}, and initialize $\epsilon\in(0, \bar{\epsilon}].$ Select ${\alpha} > 0$ according to Lemma \ref{lemma:alpha_existence}, and choose $\beta\in(r_a, \alpha)$.
\STATE\textbf{Measure} $\mathbf{x}$.
\IF{$m = 0$,}
{\color{blue}\STATE\textbf{Implement} Algorithm \ref{alg:jumpset_participation}.}
\IF{$\xi\in\mathcal{J}_0$,}
\STATE\textbf{Update }$(\mathbf{h}, m)\leftarrow\mathbf{L}(\mathbf{x}, \mathbf{h}, m)$ using \eqref{updatelaw_part1}.
\ENDIF
\ENDIF
\IF{$m\in\{-1, 1\}$,}
{\color{blue}\STATE\textbf{Implement} Algorithm \ref{alg:jumpset_participation}.}
\IF{$\xi\in\mathcal{J}_m$,}
\STATE\textbf{Update }$(\mathbf{h}, m)\leftarrow\mathbf{L}(\mathbf{x}, \mathbf{h}, m)$ using \eqref{updatelaw_part2}.
\ENDIF
\ENDIF

\IF{$m\in\{-1, 1\}$,}
\STATE\textbf{Measure} $\Pi(\mathbf{x}, \mathcal{O}_{\mathcal{W}})$ using \eqref{sensor:closest_point}.
\STATE\textbf{Locate }$\mathbf{c}$ using \eqref{sensor:c_location}.
\IF{$\mathbf{card}(\mathcal{PJ}(\mathbf{c}, \mathcal{O}_{\mathcal{W}}))>1$,}
\STATE \textbf{Construct} $\mathcal{Y}$ using \eqref{partial_ring}.
\STATE \textbf{Assign} $\partial\mathcal{O}_{\mathcal{W}}\leftarrow\partial\mathcal{O}_{\mathcal{W}}\cup\mathcal{Y}.$
\ENDIF
\ENDIF
\STATE\textbf{Execute }$\mathbf{u}(\mathbf{x}, \mathbf{h}, m)$ \eqref{hybrid_control_input}, used in \eqref{hybrid_closed_loop_system}.
\STATE\textbf{Go to} step 3.
\end{algorithmic}

\label{alg:general_implemenration}
\end{algorithm}

\begin{algorithm}
\caption{Sensor-based identification of the jump set.}

\begin{algorithmic}[1]
\STATE\textbf{Measure} $d(\mathbf{x}, \mathcal{O}_{\mathcal{W}})$ \eqref{sensor:closest_distance}, and $\mathcal{PJ}(\mathbf{x}, \mathcal{O}_{\mathcal{W}})$ \eqref{sensor:closest_point}.
\IF{$m = 0$,}
\IF{$d(\mathbf{x}, \mathcal{O}_{\mathcal{W}}) \leq \beta$}
\IF{$\mathbf{card}(\mathcal{PJ}(\mathbf{x}, \mathcal{O}_{\mathcal{W}})) = 1$,}
\IF{$\mathbf{x}^{\intercal}(\mathbf{x} - \Pi(\mathbf{x}, \mathcal{O}_{\mathcal{W}})) \geq 0$,}
\STATE\textbf{Identify} $\partial\mathcal{O}$ using \eqref{sensor:observed_boundary}.
\STATE\textbf{Construct} $\square(\mathbf{x}, \mathbf{0})$ using \eqref{vertices_of_rectangle}.
\IF{$\partial\mathcal{O} \cap\square(\mathbf{x}, \mathbf{0}) \ne \emptyset$}
\STATE $\xi\in\mathcal{J}_{0}.$

\STATE\textbf{Set} $\gamma_s = \beta - r_a.$
\STATE\textbf{Choose} $\gamma\in(\gamma_s, \alpha - r_a)$.
\STATE\textbf{Set} $v_r = r_a + \gamma.$
\STATE\textbf{Locate} $\mathbf{c}$ using \eqref{sensor_c_location}.
\ENDIF
\ENDIF
\ELSE
\IF{$-\mathbf{x} \in\mathcal{CH}(\mathbf{x}, \mathcal{PJ}(\mathbf{x}, \mathcal{O}_{\mathcal{W}}))$,}
\STATE \textbf{Set} $\mathbf{c} = \mathbf{x}$ and $v_r = \beta$.
\STATE \textbf{Construct} $\mathcal{Y}$ using \eqref{partial_ring}.
\STATE \textbf{Assign} $\partial\mathcal{O}_{\mathcal{W}}\leftarrow\partial\mathcal{O}_{\mathcal{W}}\cup\mathcal{Y}.$
\ENDIF
\ENDIF
\ENDIF
\ENDIF
\IF{$m\in\{-1,1\}$}
\IF{$d(\mathbf{x}, \mathcal{O}_{\mathcal{W}})<r_a + {\alpha},$}
\IF{$\mathbf{x}\in\mathcal{R}_m\cup\mathcal{R}_a$, defined in \eqref{always_exit_region}, \eqref{clockwise_exit_region} and \eqref{counter_clockwise_exit_region},}
\IF{$\norm{\mathbf{x}} \leq \norm{\mathbf{h}} - \epsilon,$}
\STATE$\xi \in\mathcal{J}_m$.
\ENDIF
\ENDIF
\ELSE
\STATE$\xi \in\mathcal{J}_m$.
\ENDIF
\ENDIF
\end{algorithmic}

\label{alg:jumpset_participation}
\end{algorithm}

\section{Simulation Results}
\label{sec:simulation_results}
In this section, we present simulation results for a robot navigating in \textit{a priori} unknown environments. In simulations discussed below, the robot is assumed to be equipped with a range-bearing sensor (\textit{e.g.} LiDAR) with an angular scanning range of $360^{\circ}$ and sensing radius $R_s = 3m$. The angular resolution of the sensor is chosen to be $1^{\circ}$. The simulations are performed in MATLAB 2020a.

In the first simulation scenario, we consider an unbounded workspace \textit{i.e.,} $\mathcal{O}_0 = \emptyset$, with $3$ non-convex obstacles, as shown in Fig. \ref{diagram:result1}. The robot with radius $r = 0.3m$ is initialized at $[-16, 4]^\intercal$. The target is located at the origin. The minimum safety distance $r_s  =0.1m$. The parameter $\alpha = 0.8m$ is known \textit{a priori}, as per Lemma \ref{lemma:alpha_existence}. We set the gain values $\kappa_s$ and $\kappa_r$, used in \eqref{hybrid_control_input}, to be $0.5$ and $2$, respectively. The parameter $\epsilon$, which is essential for the design of the jump set of the \textit{obstacle-avoidance} mode as given in \eqref{avoidance_jump_set} and \eqref{partition_rm}, is set to be $0.1m.$

The robot's motion in the \textit{move-to-target} mode is represented by the blue-coloured curves, whereas the red-coloured curves depict its motion in the \textit{obstacle-avoidance} mode. The locations $\mathbf{h}_1$ to $\mathbf{h}_6$ are the \textit{hit points} where the robot switches from the \textit{move-to-target} mode to the \textit{obstacle-avoidance} mode. Notice that the location of each \textit{hit point} is closer to the target location than the previous one, which ensures global convergence of the robot to the target location, as stated in Theorem \ref{theorem:global_stability}. Since the robot moves parallel to the boundary of the unsafe region in the \textit{obstacle-avoidance} mode, it maintains a safe distance from the unsafe region, as shown in Fig. \ref{diagram:distance_profile}. To avoid multiple projections onto the unsafe region, while operating in the \textit{obstacle-avoidance} mode, the robot constructs a virtual ring, as explained in Section \ref{section:sensor-based_implementation}, and moves along its boundaries around the obstacles $\mathcal {O}_1$ and $\mathcal{O}_2$. The complete simulation video can be found at \url{https://youtu.be/tRRUQNjLtGU}.

\begin{figure}[h]
    \centering
    \includegraphics[width = 1\linewidth]{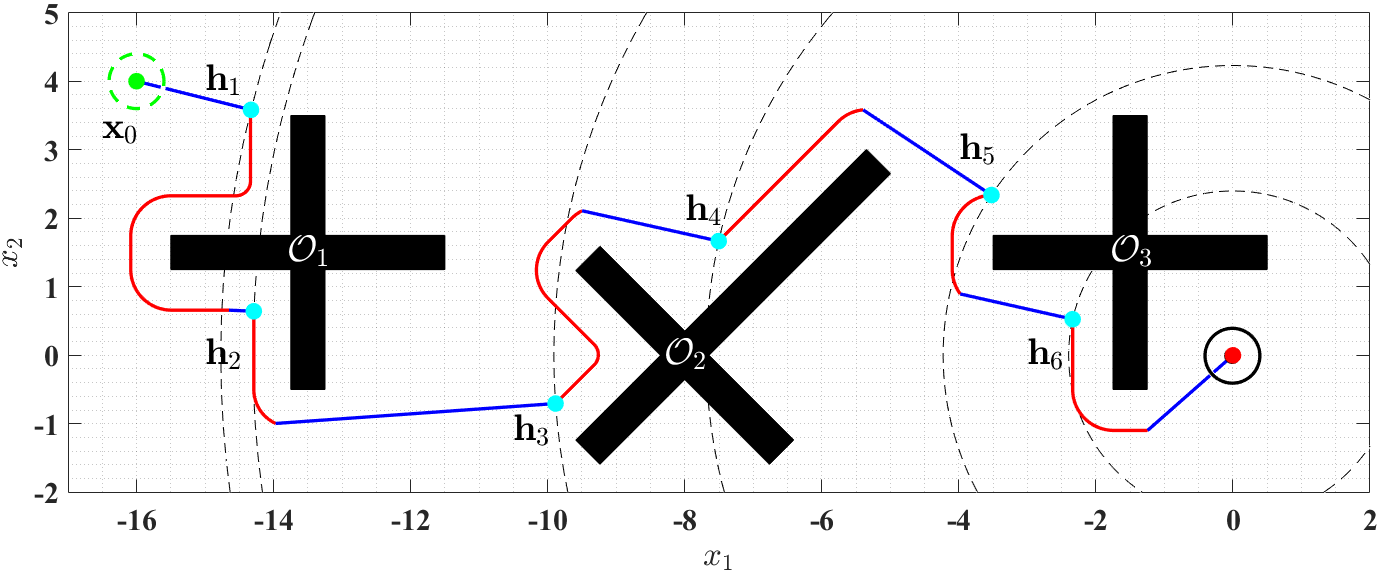}
    \caption{Trajectory of a robot, initialized at $\mathbf{x}_0$, safely converging to the target location at the origin.}
    \label{diagram:result1}
\end{figure}

\begin{figure}[h]
    \centering
    \includegraphics[width = 0.8\linewidth]{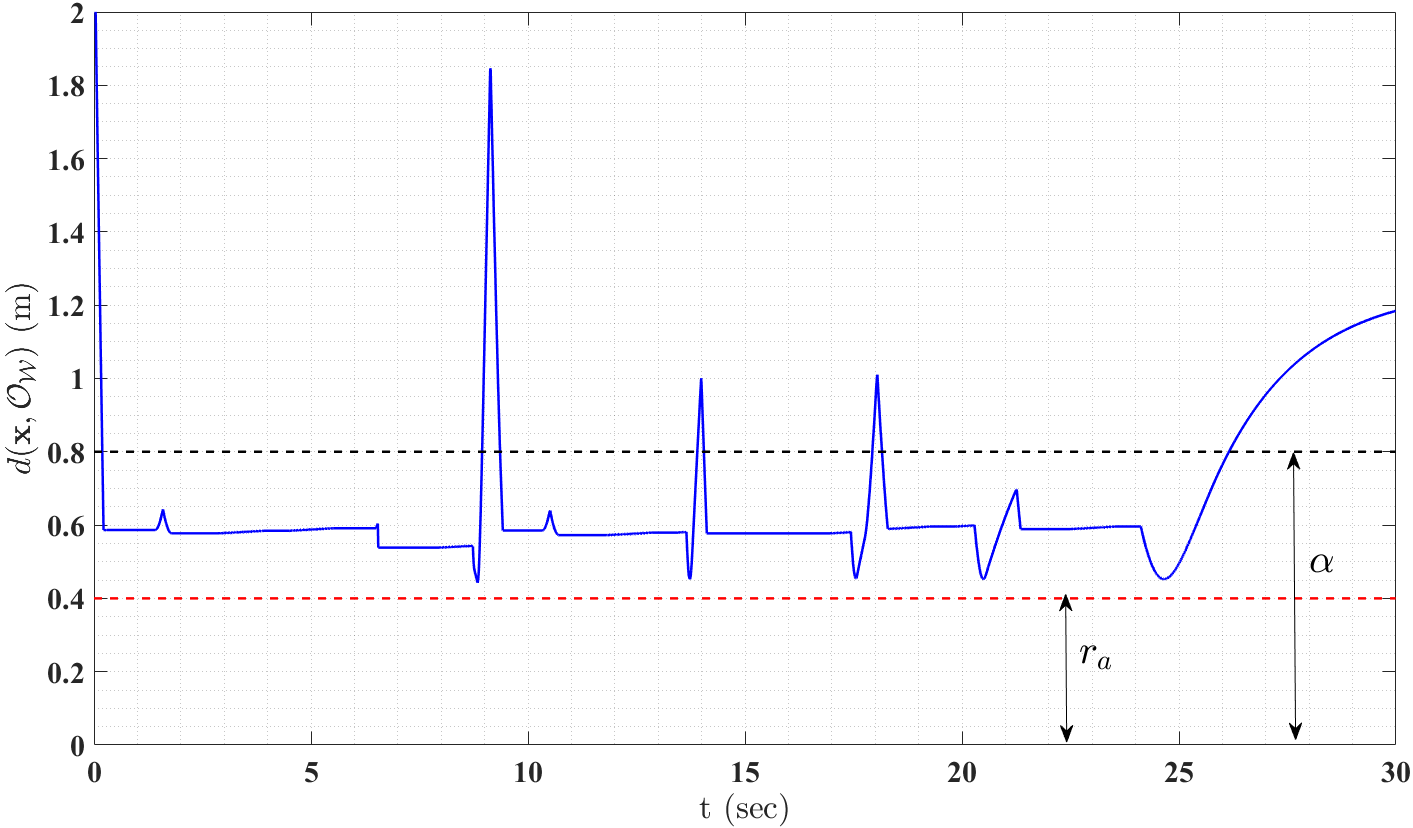}
    \caption{Distance of the robot's center from the boundary of the obstacle-occupied workspace for the simulation results shown in Fig. \ref{diagram:result1}.}
    \label{diagram:distance_profile}
\end{figure}

In the next simulation scenario, as shown in Fig. \ref{diagram:multiple_initials}, we consider an environment consisting of arbitrarily-shaped obstacles (possibly non-convex), and apply the proposed hybrid controller \eqref{hybrid_control_input} for a point robot navigation initialized at 10 different locations in the obstacle-free workspace. The target is located at the origin. The minimum safety distance $r_s = 0.1m$ and the parameter $\alpha = 0.5m$. We set the gains $\kappa_s$ and $\kappa_r$, used in \eqref{hybrid_control_input_1}, to be 0.25 and 2, respectively. The parameter $\epsilon$ is set to be $0.05m$.

Since the environment is \textit{a priori} unknown and contains non-convex obstacles, the robot maintains the same direction of motion when it moves in the \textit{obstacle-avoidance} mode to avoid retracing the previously travelled path, as discussed in remark \ref{remark:one_direction}. This does not necessarily result in the robot trajectories with the shortest lengths. The complete simulation video can be found at \url{https://youtu.be/OtHt-oQPg68}.

\begin{figure}[h]
    \centering
    \includegraphics[width = 0.7\linewidth]{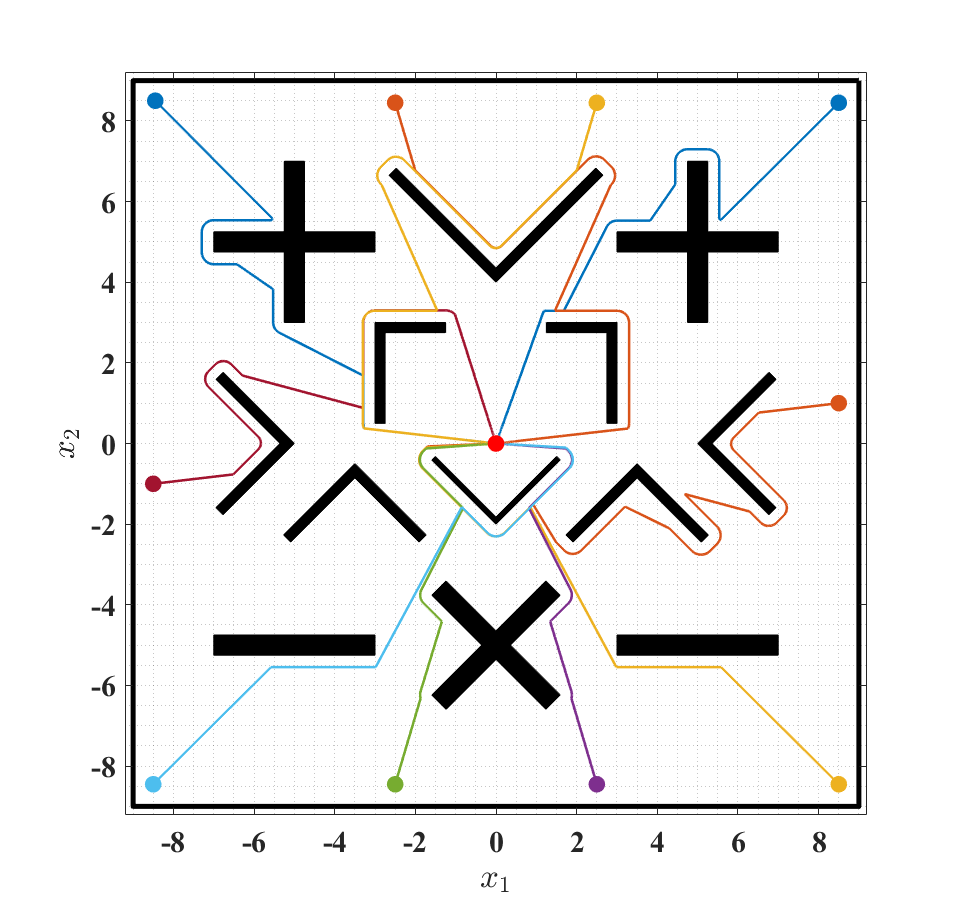}
    \caption{Robot trajectories starting from different locations.}
    \label{diagram:multiple_initials}
\end{figure}

Next, we consider an environment with unsafe region $\mathcal{O}_{\mathcal{W}}$, consisting of $2$ non-convex obstacles $\mathcal{O}_1$ and $\mathcal{O}_2$, that does not satisfy Assumption \ref{Assumption:reach}, as shown in Fig. \ref{diagram:assumption_2_not}. The robot with radius $r = 0.3m$ is initialized at $[-6, 3]^\intercal$. The target is located at the origin. The minimum safety distance $r_s = 0.1m$. The parameter $\alpha = 1m$ is known \textit{a priori}, as per Lemma \ref{lemma:alpha_existence}. We set the gain values $\kappa_s$ and $\kappa_r$, used in \eqref{hybrid_control_input}, to be 0.5 and 2, respectively. The parameter $\epsilon$, used in \eqref{partition_rm}, is set to be $0.1m$. 

Notice that even though the distance between the obstacles $\mathcal{O}_1$ and $\mathcal{O}_2$ is less than $2\alpha$, the modified set $\mathcal{O}_{\mathcal{W}}^M$, obtained using \eqref{obstacle_modification_step}, is not connected. However, using the virtual ring construction mentioned in Section \ref{section:moving_in_the_obstacle_avoidance_mode}, the robot maintains the uniqueness of its projection and moves safely across the gap in the \textit{obstacle-avoidance} mode, as shown in Fig. \ref{diagram:assumption_2_not}. The complete simulation video can be found at \url{https://youtu.be/T4xzo01_mkc}.

\begin{figure}[h]
    \centering
    \includegraphics[width = 0.9\linewidth]{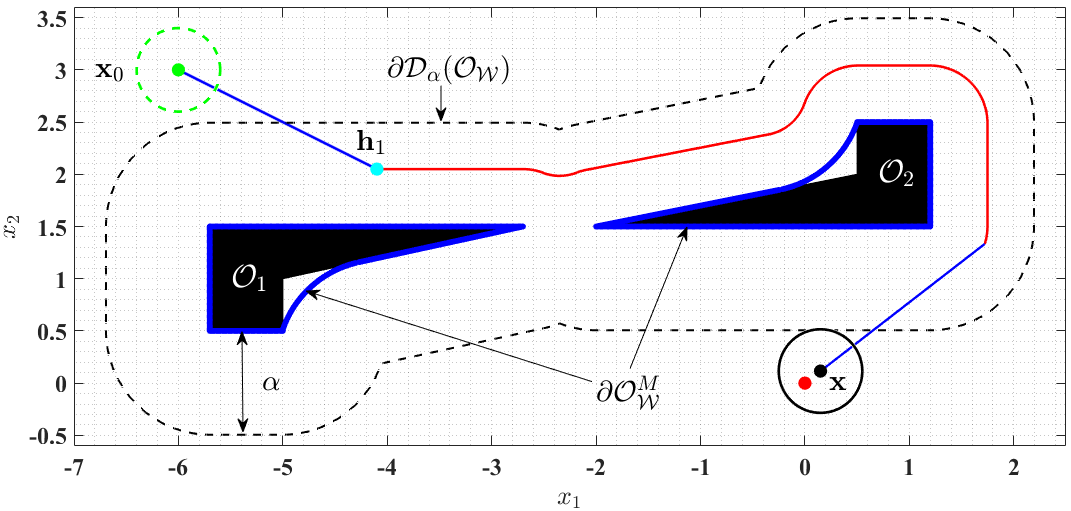}
    \caption{Robot safely navigating towards the target (red dot) in an environment that does not satisfy Assumption \ref{Assumption:reach}.}
    \label{diagram:assumption_2_not}
\end{figure}

\subsection{Gazebo simulation}
{The next simulation is performed using the Turtlebot3 Burger model in Gazebo. The simulation runs on a computer, equipped with 4 GB RAM, running Ubuntu 20.04 with the ROS Noetic distribution installed, which we refer to as \textit{Computer 1}. The proposed hybrid controller is run in Matlab R2020a on another computer running Windows 10, equipped with an Intel(R) i5-5200U CPU with a clock speed of 2.20 GHz and 12 GB RAM, referred to as \textit{Computer 2}.

The Turtlebot is equipped with a two-dimensional LiDAR scanner with a sensing range of 1m. The angular scanning range and angular resolution are set to 360 degrees and 1 degree, respectively. The sensor measurements are assumed to be affected by Gaussian noise with zero mean and a standard deviation of 0.01 m. The maximum bounds on the linear velocity and angular velocity, denoted as $\nu_{\max}$ and $\omega_{\max}$ respectively, are set to 0.15 m/s and 2.84 rad/s. The LiDAR scanning rate is set to 5 Hz.

At each iteration, the LiDAR measurements and pose information are sent from \textit{Computer 1} to \textit{Computer 2}. The control commands are then sent from \textit{Computer 2} to \textit{Computer 1}. The execution time of the proposed hybrid controller is approximately 5 ms. The sensor-based implementation of the proposed hybrid navigation algorithm only requires the measurements acquired via a range-bearing sensor. Since the size of the acquired sensor data is independent of the surrounding environment, the code execution time will remain approximately the same regardless of changes in the environment.

The Turtlebot can be represented with the following nonholonomic model: 
\begin{equation}
\begin{aligned}
    \dot{x}_1 &= \nu\cos\theta,\\
    \dot{x}_2 &= \nu\sin\theta,\\
    \dot{\theta} &= \omega,
\end{aligned}
\end{equation}
where $\mathbf{x} = [x_1, x_2]^\intercal$ is the location of the center of the robot and $\theta\in[-\pi, \pi)$ is the heading direction. The scalar control variables $\nu$ and $\omega$ represent the linear and angular velocities, respectively.

In practical applications, due to the discrete implementation of the control law designed for a point-mass robot, the nonholonomic Turtlebot (when operating in the \textit{obstacle-avoidance} mode) may get very close to the obstacle or exit the $\alpha$-neighborhood of the obstacle before getting closer to the target location than the current \textit{hit point}. To avoid this situation, we introduce some minor modifications to the proposed hybrid control law to ensure that the Turtlebot stays inside the $\alpha$-neighborhood of the nearest obstacle when operating in the \textit{obstacle-avoidance} mode. We replace the vector $\mathbf{v}(\mathbf{x}, m)$, used in \eqref{hybrid_control_input_1}, with a modified vector $\mathbf{v}_{mod}(\mathbf{x}, m),$ which is defined as
\begin{equation}
\mathbf{v}_{d}(\mathbf{x}, m) = \begin{bmatrix}\lambda(\varrho(\mathbf{x}))& m(1-\lambda(\varrho(\mathbf{x}))^2)\\-m(1-\lambda(\varrho(\mathbf{x}))^2)& \lambda(\varrho(\mathbf{x}))\end{bmatrix}\mathbf{f}(\mathbf{x}, m),
\end{equation}
where the vector-valued function $\mathbf{f}(\mathbf{x}, m)$ is given by
\begin{equation}
    \mathbf{f}(\mathbf{x}, m) = \frac{\mathbf{x} - \Pi(\mathbf{x}, \mathcal{O}_{\mathcal{W}}^M)}{\norm{\mathbf{x} - \Pi(\mathbf{x}, \mathcal{O}_{\mathcal{W}}^M)}}.
\end{equation}
The scalar-valued function $\lambda(\varrho(\mathbf{x}))$ is evaluated as
\begin{equation}
    \lambda(\varrho(\mathbf{x})) = \begin{cases}
        \frac{0.25\eta - \varrho(\mathbf{x})}{0.25\eta}, & 0 \leq \varrho(\mathbf{x}) \leq 0.25\eta,\\
        0, & 0.25\eta \leq \varrho(\mathbf{x}) \leq 0.75\eta,\\
        \frac{0.75\eta - \varrho(\mathbf{x})}{0.25\eta}, & 0.75\eta \leq \varrho(\mathbf{x}) \leq \eta,
    \end{cases}
\end{equation}
where $\eta = \alpha - r_a$ and $\varrho(\mathbf{x}) = d(\mathbf{x}, \mathcal{O}_{\mathcal{W}}^M) - r_a$. The continuous scalar-valued function $\lambda(\varrho(\mathbf{x}))\in[-1, 1]$, for all $\varrho(\mathbf{x})\in[0, \eta]$. 

Notice that when $\lambda(\varrho(\mathbf{x})) = 0$, the vector $\mathbf{v}_d(\mathbf{x}, m)$ equals to the vector $\mathbf{v}(\mathbf{x}, m)$, used in \eqref{hybrid_control_input_1}. When the Turtlebot, while operating in the \textit{obstacle-avoidance} mode, moves closer to the boundary of the modified obstacle-occupied workspace \textit{i.e.,} $\varrho(\mathbf{x})\to 0$, $\lambda(\varrho(\mathbf{x}))\to1$. As a result, the vector $\mathbf{v}_d(\mathbf{x}, m)\to\mathbf{f}(\mathbf{x}, m)$ and the Turtlebot is steered away from the unsafe region. On the other hand, if the center of the Turtlebot moves closer to the boundary of the $\alpha-$neighbourhood of the modified obstacles \textit{i.e.,} $\varrho(\mathbf{x})\to\eta,$ $\lambda(\varrho(\mathbf{x}))\to-1$. Due to this, the vector $\mathbf{v}_d(\mathbf{x}, m) \to-\mathbf{f}(\mathbf{x}, m)$ and the Turtlebot is steered back inside the $\alpha-$neighbourhood of the modified obstacles.

Finally, given the modified hybrid control law $\mathbf{u}_{mod}$, obtained by replacing $\mathbf{v}(\mathbf{x}, m)$ by $\mathbf{v}_{mod}(\mathbf{x}, m)$ in \eqref{hybrid_control_input_1}, the linear velocity $\nu$ and the angular velocity $\omega$ to be applied to the Turtlebot are obtained as follows:
\begin{equation}
    \nu = \kappa_{\nu}\min\bigg\{\|{\mathbf{u}_{mod}}\|\cos^{2n}\left(\frac{\theta - \theta_d}{2}\right), \nu_{\max}\bigg\},\label{linear_velocity}
\end{equation}
\begin{equation}
    \omega = -\kappa_{\omega}\omega_{\max}\sin(\theta - \theta_d),\label{angular_velocity}
\end{equation}
where $n > 1$, $\kappa_{\nu} > 0$ and $ \kappa_{\omega} > 0$. The angle $\theta$ represents the heading direction of the robot. The desired heading direction is denoted by $\theta_d$, which is evaluated as $\theta_d = \text{atan2v}(\mathbf{u}_{\text{mod}}).$ The expression $\frac{1 + \cos(\theta - \theta_d)}{2}$ in \eqref{linear_velocity} reduces the linear velocity of the Turtlebot based on the disparity between its current heading direction $\theta$ and the desired heading direction $\theta_d$. 

We set the gains $\kappa_s$, $\kappa_r$, used in \eqref{hybrid_control_input_1}, and $\kappa_{\nu}$, $\kappa_{\omega}$, used in \eqref{linear_velocity} and \eqref{angular_velocity}, to 1. The minimum safety distance $r_s = 0.03$ m and the parameter $\alpha = 0.3$ m. Additionally, the parameter $\epsilon$, used in \eqref{partition_rm}, is set to 0.15 m. The target location is set to the origin, represented by the light green dot, as shown in Fig. \ref{fig:gazebo_display}. In Fig. \ref{fig:gazebo_display}a, the left figure shows the workspace setup in Gazebo with the initial location of the Turtlebot, and the right figure shows the LiDAR sensor measurements. The desired heading direction is denoted by the blue arrow, while the red arrow represents the current heading direction of the Turtlebot. The Turtlebot and the target location are connected via a dotted red line. When the Turtlebot moves straight towards the target location, it eventually enters the $\alpha-$neighbourhood of the unsafe region and switches to the \textit{obstacle-avoidance} mode. In the \textit{obstacle-avoidance} mode, it constructs a virtual ring represented using the red dotted circle, as shown in Fig. \ref{fig:gazebo_display}b. When the nearby obstacle workspace is convex in nature, the projection of the center of the Turtlebot onto this workspace matches the intersection point between the virtual ring and the nearby obstacle, which is represented by the magenta-coloured dot. In Fig. \ref{fig:gazebo_display}b, the blue-dotted curve is a partial boundary of the circle with its center at the origin and a radius of $\norm{\mathbf{h}} - \epsilon$, where $\mathbf{h}$ is the current location of the \textit{hit point}. According to \eqref{avoidance_jump_set}, \eqref{partition_rm} and \eqref{bothmodes_flowjumpset1}, the Turtlebot can switch back to the \textit{move-to-target} mode only if it is inside this circle. When the nearby unsafe region is non-convex in nature, the virtual ring, which is larger in size compared to the robot's body, will have multiple intersections with the obstacles, as shown in Fig. \ref{fig:gazebo_display}c. This prompts the Turtlebot to project on the boundary of the virtual ring instead of the obstacle-occupied workspace to maintain the uniqueness of the projection. In other words, boundary of the virtual ring acts as the boundary of the modified obstacle, as discussed in Remark \ref{remark:virtual_ring}. Finally, in Fig. \ref{fig:gazebo_display}d we can see the Turtlebot approaching towards the target location at the origin. The complete simulation video can be found at \url{https://youtu.be/ZNeiS5qE00k}.

\begin{figure}
    \centering
    \includegraphics[width = 1\linewidth]{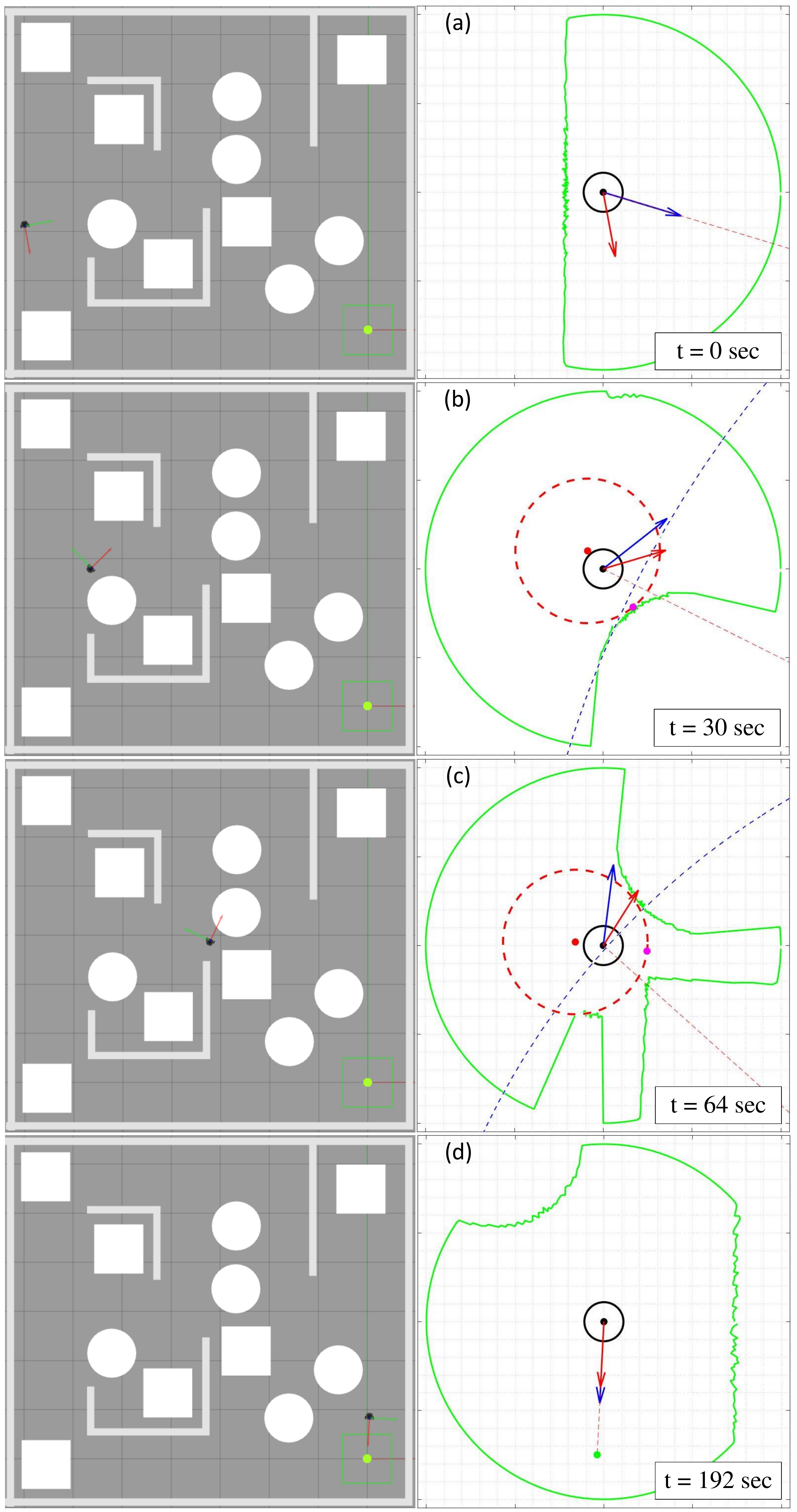}
    \caption{Autonomous navigation of the Turtlebot in non-convex environment.}
    \label{fig:gazebo_display}
\end{figure}

}

\section{Conclusion}\label{sec:conclusion}
We proposed a hybrid feedback controller for safe autonomous navigation in two-dimensional environments with arbitrary-shaped obstacles (possibly non-convex). The obstacles can have non-smooth boundaries and large sizes and can be placed arbitrarily close to each other provided the feasibility requirements stated in Assumptions \ref{assumption:connected_interior} and \ref{Assumption:reach} are satisfied. The proposed hybrid controller, endowed with global asymptotic stability guarantees, relies on an instrumental transformation that virtually modifies the obstacles' shapes such that the uniqueness of the projection of the robot's center onto the closest obstacle is guaranteed - a feature that helps in the design of our \textit{obstacle-avoidance} strategy. 
The obstacle avoidance component of the control law \eqref{definition:vxm} utilizes the projection of the robot's center onto the nearest obstacle. Hence, it is possible to apply the proposed hybrid control scheme in \textit{a priori} unknown environments, as discussed in Section \ref{section:sensor-based_implementation}. 
It should be noted that the trajectories generated by our algorithm may not necessarily correspond to the shortest paths between the initial and final configurations, as shown in Fig. \ref{diagram:multiple_initials}.
Moreover, when the robot switches between the modes, the control input vector \eqref{hybrid_control_input_1} changes value discontinuously. 
Designing a smooth feedback control law that generates robot trajectories with the shortest lengths would be an interesting extension of the present work.
Other interesting extensions consist in considering robots with second-order dynamics and three-dimensional environments with non-convex obstacles.

\begin{appendix}
\subsection{Hybrid basic conditions}
\begin{lemma}
The hybrid closed-loop system \eqref{hybrid_closed_loop_system} with the data $(\mathcal{F}, \mathbf{F}, \mathcal{J}, \mathbf{J})$ satisfies the hybrid basic conditions stated in \cite[Assumption 6.5]{goebel2012hybrid}.\label{nonconvex_hybrid_basic}
\end{lemma}

{
\begin{proof}
The flow set $\mathcal{F}$ and the jump set $\mathcal{J}$, defined in \eqref{bothmodes_flowjumpset} and \eqref{bothmodes_flowjumpset1} are closed subsets of $\mathbb{R}^2\times\mathbb{R}^2\times\mathbb{R}$. The flow map $\mathbf{F}$, given in \eqref{hybrid_closed_loop_system}, is continuous on $\mathcal{F}_0$.
For each $\mathbf{x}\in\mathcal{N}_{\beta}(\mathcal{D}_{r_a}(\mathcal{O}_{\mathcal{W}}^M)), \beta\in[r_a, \alpha),$ according to Lemma \ref{lemma:unique_projection}, the set $\mathcal{PJ}(\mathbf{x}, \mathcal{O}_{\mathcal{W}}^M)$ is a singleton. 
Then, for $\mathbf{x}\in\mathcal{N}_{\beta}(\mathcal{D}_{r_a}(\mathcal{O}_{\mathcal{W}}^M)), \beta\in[r_a, \alpha)$, $\Pi(\mathbf{x}, \mathcal{O}_{\mathcal{W}}^M)$ is continuous with respect to $\mathbf{x}.$
As a result, $\mathbf{F}$ is continuous on $\mathcal{F}_m, m\in\{-1, 1\}$. 
Hence $\mathbf{F}$ is continuous on $\mathcal{F}.$ The jump map $\mathbf{J}$, defined in \eqref{hybrid_closed_loop_system}, is single-valued on $\mathcal{J}_m, m\in\{-1 ,1\}$ \eqref{updatelaw_part2}. Also, $\mathbf{J}$ has a closed graph relative to $\mathcal{J}_0$ \eqref{bothmodes_flowjumpset1}, as it is allowed to be set-valued whenever $\mathbf{x}\in\mathcal{J}_0.$ Hence, according to \cite[Lemma 5.10]{goebel2012hybrid}, $\mathbf{J}$ is outer semi-continuous and locally bounded relative to $\mathcal{J}$.
\end{proof}

\subsection{Proof of Lemma \ref{lemma:reach_property}}
\label{proof:reach_property}

This proof is by contradiction. Let us assume that there exists $\mathbf{x}\in\mathbb{R}^n\setminus\mathcal{A}$ such that $d(\mathbf{x}, \mathcal{A}) = \beta$, where $\beta\in(0, \alpha)$. We further assume that $d(\mathbf{x}, \mathcal{G}) = \eta$ and the set $\mathcal{PJ}(\mathbf{x}, \mathcal{G})$ is not a singleton, where the closed set $\mathcal{G} = (\mathcal{A}\oplus\mathcal{B}_{\alpha}^{\circ}(\mathbf{0}))^{c}$. 

Since we have $\mathbf{reach}(\mathcal{A}) \geq  \alpha$, the set $\mathcal{PJ}(\mathbf{x}, \mathcal{A})$ is a singleton. As a result, according to \cite[Lemma 4.5]{rataj2019curvature}, it follows that $\frac{\mathbf{x} - \Pi(\mathbf{x}, \mathcal{A})}{\|{\mathbf{x} - \Pi(\mathbf{x}, \mathcal{A})}\|}\in\mathbf{N}_{\mathcal{A}}(\Pi(\mathbf{x}, \mathcal{A}))$, where  $\mathbf{N}_{\mathcal{A}}(\Pi(\mathbf{x}, \mathcal{A}))$ denotes the normal cone to the set $\mathcal{A}$ at the point $\Pi(\mathbf{x}, \mathcal{A}).$ 
Furthermore, according to \cite[Lemma 4.5]{rataj2019curvature}, for $\mathbf{q}= \Pi(\mathbf{x}, \mathcal{A}) + \alpha\frac{\mathbf{x} - \Pi(\mathbf{x}, \mathcal{A})}{\|{\mathbf{x} - \Pi(\mathbf{x}, \mathcal{A})}\|}$, the open Euclidean ball $\mathcal{B}_{\alpha}^{\circ}(\mathbf{q})$ does not intersect with the set $\mathcal{A}$ \textit{i.e.},  $\mathcal{B}_{\alpha}^{\circ}(\mathbf{q})\cap\mathcal{A} = \emptyset.$ Hence, as the set $\partial\mathcal{G}$ contains all points in $\mathbb{R}^n$ that are exactly at $\alpha$ distance away from the set $\mathcal{A}$, one can conclude that $\mathbf{q}\in\partial\mathcal{G}\cap\mathcal{L}(\mathbf{x}, \Pi(\mathbf{x}, \mathcal{A}))$. 
Now, we consider two cases depending on the location of the point $\mathbf{q}$ on the line $\mathcal{L}(\mathbf{x}, \Pi(\mathbf{x}, \mathcal{A}))$ as follows:

\textbf{Case 1: } when $d(\mathbf{q}, \Pi(\mathbf{x}, \mathcal{A})) = \alpha \geq \beta + \eta$. Since $\alpha \geq \beta + \eta$, one has $\mathcal{B}_{\eta}(\mathbf{x})\subset\mathcal{B}_{\alpha}(\Pi(\mathbf{x}, \mathcal{A}))\subset\mathcal{D}_{\alpha}(\mathcal{A})$. Additionally, as $\mathbf{card}(\mathcal{PJ}(\mathbf{x},\partial\mathcal{G})) > 1$, there exist $\mathbf{p}_1\in\mathcal{G}$ and $\mathbf{p}_2\in\mathcal{G}$ such that $\{\mathbf{p}_1, \mathbf{p}_2\}\subset\partial\mathcal{B}_{\eta}(\mathbf{x})\cap\partial\mathcal{G}$. Hence, there should be at least two points of contact between the sets $\partial\mathcal{B}_{\eta}(\mathbf{x})$ and $\partial\mathcal{B}_{\alpha}(\Pi(\mathbf{x}, \mathcal{A}))$. Since $\alpha \geq\beta + \eta$, the Euclidean ball $\mathcal{B}_{\eta}(\mathbf{x})$ can only touch the boundary of the Euclidean ball $\mathcal{B}_{\alpha}(\Pi(\mathbf{x}, \mathcal{A}))$ at no more that one point, resulting in a contradiction.

\textbf{Case 2: } when $d(\mathbf{q}, \Pi(\mathbf{x}, \mathcal{A})) = \alpha \in(\beta, \beta+\eta)$. Since $\mathbf{q}\in\partial\mathcal{G}\cap\mathcal{L}(\mathbf{x}, \Pi(\mathbf{x}, \mathcal{A}))$ and $\alpha\in(\beta, \beta + \eta)$, one has $\mathbf{q}\in\mathcal{B}_{\eta}^{\circ}(\mathbf{x})$. This implies that $d(\mathbf{x},\mathcal{G}) < \eta,$ which is a contradiction.

\subsection{Proof of Lemma \ref{lemma:reach_extends}}
\label{proof:reach_extends}
The cases where $\beta = 0$ and $\beta = \alpha$ are trivial. We analyze the case where $\beta \in(0, \alpha)$. This proof is by contradiction. Let us assume that there exists $\mathbf{x}$ such that $d(\mathbf{x}, \mathcal{A}) = \eta$, where $\eta\in(\beta, \alpha)$, and the set $\mathcal{PJ}(\mathbf{x},\mathcal{D}_{\beta}( \mathcal{A}))$ is not a singleton. Therefore, there exist at least two distinct points $\mathbf{p}_1$ and $\mathbf{p}_2$ such that $\mathbf{p}_1, \mathbf{p}_2\subset\partial\mathcal{D}_{\beta}(\mathcal{A})\cap\mathcal{PJ}(\mathbf{x}, \mathcal{D}_{\beta}(\mathcal{A})).$

Since $\mathbf{reach}(\mathcal{A}) \geq \alpha$, the points $\mathbf{p}_1$ and $\mathbf{p}_2$, which belong to the set $\partial\mathcal{D}_{\beta}(\mathcal{A})$, have unique projection on the set $\mathcal{A}$ and $d(\mathbf{p}_1, \Pi(\mathbf{p}_1, \mathcal{A})) = d(\mathbf{p}_2, \Pi(\mathbf{p}_2, \mathcal{A})) = \beta$. 
Moreover, as $\mathbf{x}\in\partial\mathcal{D}_{\eta}(\mathcal{A})$, one has $d(\mathbf{x}, \mathbf{p}_1) = d(\mathbf{x}, \mathbf{p}_2) = \eta - \beta$, and $d(\mathbf{x}, \mathcal{A}) = \eta$. Therefore, one has $d(\mathbf{x}, \Pi(\mathbf{p}_1, \mathcal{A})) = d(\mathbf{x}, \Pi(\mathbf{p}_2, \mathcal{A})) = \eta.$ As a result, $\Pi(\mathbf{p}_1, \mathcal{A})$ and $\Pi(\mathbf{p}_2, \mathcal{A})$ belong to the set $\mathcal{PJ}(\mathbf{x}, \mathcal{A}).$ Since $\eta\in(0, \alpha)$ and $\mathbf{reach}(\mathcal{A}) \geq \alpha$, it is clear that $\mathbf{card}(\mathcal{PJ}(\mathbf{x}, \mathcal{A})) = 1$. Hence, $\Pi(\mathbf{p}_1, \mathcal{A}) = \Pi(\mathbf{p}_2, \mathcal{A}) = \Pi(\mathbf{x}, \mathcal{A}).$ This, by the application of triangular inequality, implies that $\mathbf{p}_1 = \mathbf{p}_2$, which is a contradiction.

\subsection{Proof of Lemma \ref{lemma:unique_projection}}
\label{proof:unique_projection}
Note that the following statement: ``for all locations $\mathbf{x}$ with $d(\mathbf{x}, \mathcal{O}_{\mathcal{W}}^M) <\alpha$, where $\alpha \geq 0$, the set $\mathcal{PJ}(\mathbf{x}, \mathcal{O}_{\mathcal{W}}^M)$ is singleton", is equivalent to having $\mathbf{reach}(\mathcal{O}_{\mathcal{W}}^M) \geq \alpha,$ as defined in Section \ref{section:set_with_positive_reach}.
According to Remark \ref{remark:alternate_modified_obstacle}, $\mathcal{O}_{\mathcal{W}}^M = \mathbf{M}(\mathcal{O}_{\mathcal{W}}, \alpha) = (\mathcal{W}_{\alpha}\oplus\mathcal{B}_{\alpha}^{\circ}(\mathbf{0}))^c$, where the closed set $\mathcal{W}_{\alpha}$ is defined as per \eqref{y_free_workspace} for some $\alpha\geq 0$. If one proves that $\mathbf{reach}(\mathcal{W}_{\alpha}) \geq \alpha$, then, according to Lemma \ref{lemma:reach_property}, one has $\mathbf{reach}(\mathcal{O}_{\mathcal{W}}^M)\geq\alpha$. To that end, we make use of \cite[Proposition 4.14]{rataj2019curvature} to show that Assumption \ref{Assumption:reach} implies that $\mathbf{reach}(\mathcal{W}_{\alpha}) \geq \alpha$.

We know that for $\mathbf{x}\in\mathcal{W}_{\alpha}^{\circ}$,  $\mathbf{T}_{\mathcal{W}_{\alpha}}(\mathbf{x}) = \mathbb{R}^2$, where $\mathbf{T}_{\mathcal{W}_{\alpha}}(\mathbf{x})$ denotes the tangent cone to the set $\mathcal{W}_{\alpha}$ at $\mathbf{x}$. Therefore, for all $\mathbf{p}\in\mathcal{W}_{\alpha}$, one has $\mathbf{p} - \mathbf{x}\in\mathbf{T}_{\mathcal{W}_{\alpha}}(\mathbf{x})$. This implies that for all $\mathbf{x}\in\mathcal{W}_{\alpha}^{\circ}$ and for all $\mathbf{p}\in\mathcal{W}_{\alpha}$, $d(\mathbf{p} - \mathbf{x}, \mathbf{T}_{\mathcal{W}_{\alpha}}(\mathbf{x})) = 0.$ 

Next, we consider the case where $\mathbf{x}\in\partial\mathcal{W}_{\alpha}$ and $\mathbf{p}\in\mathcal{W}_{\alpha}$. If $\mathbf{p}-\mathbf{x}\in\mathbf{T}_{\mathcal{W}_{\alpha}}(\mathbf{x})$, then $d(\mathbf{p} - \mathbf{x}, \mathbf{T}_{\mathcal{W}_{\alpha}}(\mathbf{x})) = 0.$ 
On the other hand, when $\mathbf{p}-\mathbf{x}\notin\mathbf{T}_{\mathcal{W}_{\alpha}}(\mathbf{x})$, we define $\mathcal{T}(\mathbf{x}, \mathbf{p}) := \{\mathbf{t}\in\mathbf{T}_{\mathcal{W}_{\alpha}}(\mathbf{x})|\mathbf{t} = \underset{\mathbf{q}\in\mathbf{T}_{\mathcal{W}_{\alpha}}(\mathbf{x})\setminus\{\mathbf{0}\}}{\arg \min}|\psi(\mathbf{p}-\mathbf{x}, \mathbf{q})|\}$ as the set of all non-zero tangent vectors $\mathbf{t}\in\mathbf{T}_{\mathcal{W}_{\alpha}}(\mathbf{x})$ such that the absolute value of the angle measured from $\mathbf{p}-\mathbf{x}$ to $\mathbf{t}$ is the smallest. 
Since $\mathcal{W}_{\alpha} = \mathcal{W}\setminus\mathcal{D}_{\alpha}(\mathcal{O}_{\mathcal{W}}^{\circ})$, for all $\mathbf{x}\in\partial\mathcal{W}_{\alpha}$, the set $\mathbf{N}_{\mathcal{W}_{\alpha}}(\mathbf{x})\setminus\{\mathbf{0}\}$ is not empty, where $\mathbf{N}_{\mathcal{W}_{\alpha}}(\mathbf{x})$ represents the normal cone to the set $\mathcal{W}_{\alpha}$ at $\mathbf{x}$. Therefore, there exists $\mathbf{n}\in\mathbf{N}_{\mathcal{W}_{\alpha}}(\mathbf{x})$ such that $\|\mathbf{n}\| = 1$ and $\mathbf{n}^\top\mathbf{t} = 0$ for some $\mathbf{t}\in\mathcal{T}(\mathbf{x}, \mathbf{p}).$

Now, one can construct the ball $\mathcal{B}_{\beta}(\mathbf{x} + \beta\mathbf{n})$ for some $\beta > 0$ such that $\mathbf{p}\in\partial\mathcal{B}_{\beta}(\mathbf{x} + \beta\mathbf{n})$, as shown in Fig. \ref{proof_figure_assumption_reach}. It is clear that $\beta\geq \alpha$, otherwise, it will imply that 
$\mathbf{p}\in\mathcal{B}_{\alpha}^{\circ}(\mathbf{x} + \alpha\mathbf{n})$, which does not satisfy Assumption \ref{Assumption:reach}. Therefore, it can be shown that $d(\mathbf{p}-\mathbf{x}, \mathbf{T}_{\mathcal{W}_{\alpha}}(\mathbf{x})) = \|\mathbf{p}- \mathbf{x}\|^2/2\beta \leq \|\mathbf{p}- \mathbf{x}\|^2/2\alpha$ for any $\mathbf{x}, \mathbf{p}\in\mathcal{W}_{\alpha}$. Hence, according to \cite[Proposition 4.14]{rataj2019curvature}, one can conclude that $\mathbf{reach}(\mathcal{W}_{\alpha})\geq\alpha.$

\begin{figure}[ht]
    \centering
    \includegraphics[width = 0.45\linewidth]{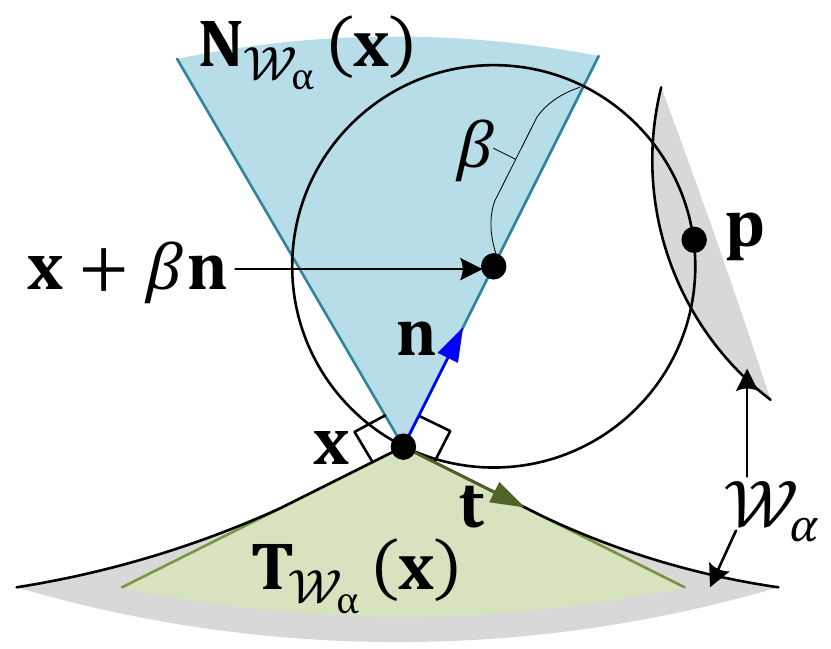}
    \caption{Illustration of the case where $\mathbf{x}\in\partial\mathcal{W}_{\alpha}$ and $\mathbf{p}\in\mathcal{W}_{\alpha}$ such that $\mathbf{p}-\mathbf{x}\notin\mathbf{T}_{\mathcal{W}_{\alpha}}(\mathbf{x}).$}
    \label{proof_figure_assumption_reach}
\end{figure}

\subsection{Proof of Lemma \ref{lemma:pathwise_connected}}
\label{proof:pathwise_connected}
Let us assume that the modified obstacle $\mathcal{O}_{i, \alpha}^M$ is not a connected set, then it implies that there exists two disjoint subsets $\mathcal{M}_1\subset\mathcal{O}_{i, \alpha}^M$ and $\mathcal{M}_2\subset\mathcal{O}_{i, \alpha}^M$ such that $\mathcal{M}_1\cup\mathcal{M}_2 = \mathcal{O}_{i, \alpha}^M$ and $\mathcal{M}_1\cap\mathcal{M}_2 = \emptyset,$ \cite[Definition 16.1]{willard2012general}. By construction in \eqref{obstacle_modification_step}, there exists two non-empty set $\mathcal{O}_{\mathbb{C}}\subset\mathcal{O}_{i, \alpha}$ and $\mathcal{O}_{\mathbb{D}}\subset\mathcal{O}_{i, \alpha}$ such that $\mathcal{O}_{\mathbb{C}}\subset\mathcal{M}_1$, $\mathcal{O}_{\mathbb{D}}\subset\mathcal{M}_2$ and $\mathcal{O}_{\mathbb{C}}\cup\mathcal{O}_{\mathbb{D}} = \mathcal{O}_{i, \alpha}$. Note that, the distance $d(\mathcal{M}_1, \mathcal{M}_2)$ can not be greater than or equal to $2\alpha.$ Otherwise, as the operator $\mathbf{M}$ is extensive, see Remark \ref{remark:features}, it implies that $d(\mathcal{O}_{\mathbb{C}}, \mathcal{O}_{\mathbb{D}}) \geq2\alpha.$ However, according to \eqref{obstacle_alpha_chain}, one has $d(\mathcal{O}_{\mathbb{C}}, \mathcal{O}_{\mathbb{D}})<2\alpha$. Hence, one can conclude that $0<d(\mathcal{M}_1, \mathcal{M}_2)<2\alpha.$ Therefore, there exists $\mathbf{q}_1\in\mathcal{M}_1$ and $\mathbf{q}_2\in\mathcal{M}_2$ such that $0<d(\mathbf{q}_1, \mathbf{q}_2) = d(\mathcal{M}_2, \mathcal{M}_2)<2\alpha$. 
This implies that for $\mathbf{q} = 0.5\mathbf{q}_1 + 0.5\mathbf{q}_2 \in\mathcal{D}_{\alpha}(\mathcal{O}_{i, \alpha}^M)$, one has $\mathbf{card}(\mathcal{PJ}(\mathbf{q}, \mathcal{O}_{i, \alpha}^M)) > 1$, which cannot be the case as per Lemma \ref{lemma:unique_projection}. Hence, the modified obstacle $\mathcal{O}_{i, \alpha}^M$ is a connected set.
}

\subsection{Proof of Lemma \ref{lemma:alpha_existence}}
\label{proof:alpha_existence}
According to Assumption \ref{assumption:connected_interior}, the set $\mathcal{W}_{r_a}^{\circ}$ is a pathwise connected set. Hence, it is evident that there exist a scalar $\bar{\delta}_1 > r_a$ such that for all $\delta\in(r_a, \bar{\delta}_1]$ the $\delta$ eroded obstacle-free workspace $\mathcal{W}_{\delta}$ is pathwise connected. Assumption \ref{assumption:connected_interior}, which assumes that the target location at the origin is in the set $\mathcal{W}_{r_a}^{\circ}$, implies that $d(\mathbf{0}, \mathcal{W}_{r_a}) > 0$. Therefore, it is straight forward to notice that there exists $\bar{\delta}_2>r_a$ such that for any $\delta\in(r_a, \bar{\delta}_2]$ the distance $d(\mathbf{0}, \mathcal{W}_{\delta}) < \delta - r_a.$ As a result, one can choose $\bar{\alpha} = \min\{\bar{\delta}_1, \bar{\delta}_2\}$ to satisfy Lemma \ref{lemma:alpha_existence}.

\subsection{Proof of Lemma \ref{lemma:properties_of_modified_workspace}}
\label{proof:properties_of_modified_workspace}
Since the obstacle reshaping operator $\mathbf{M}$ is \textit{idempotent}, we have $\mathbf{M}(\mathcal{O}_{\mathcal{W}}, \alpha) = \mathbf{M}(\mathcal{O}_{\mathcal{W}}^M, \alpha)$. Therefore, according to \eqref{obstacle_modification_step}, $\mathcal{O}_{\mathcal{W}}\oplus\mathcal{B}_{\alpha}^{\circ}(\mathbf{0}) = \mathcal{O}_{\mathcal{W}}^M\oplus\mathcal{B}_{\alpha}^{\circ}(\mathbf{0})$. As a result, according to \eqref{y_free_workspace} and \eqref{y_modified_free_workspace}, the $\alpha-$eroded obstacle-free workspace $\mathcal{W}_{\alpha}$ is equivalent to the $\alpha-$eroded modified obstacle-free workspace $\mathcal{V}_{\alpha}$. Hence, if one chooses $\alpha$ as per Lemma \ref{lemma:alpha_existence}, then the set $\mathcal{V}_{\alpha}$ is pathwise connected. As a result, the modified obstacle-free workspace $\mathcal{V}_{r_a}$ is also pathwise connected. Moreover, according to Lemma \ref{lemma:alpha_existence}, the distance between the origin and the set $\mathcal{V}_{\alpha}$ is less than $\alpha - r_a$. Since the distance between the sets $\partial\mathcal{V}_{\alpha}$ and $\partial\mathcal{V}_{r_a}$ is $\alpha - r_a$, it is evident that the origin belongs to the interior of the modified obstacle-free workspace $\mathcal{V}_{r_a}.$

\subsection{Proof of Lemma \ref{lemma:bar_epsilon}}
\label{proof:bar_epsilon}
We consider a connected modified obstacle $\mathcal{O}_{i, \alpha}^M, i\in\mathbb{I}$, as stated in Lemma \ref{lemma:pathwise_connected}, and proceed by proving the following two claims:

\textbf{Claim 1:} For every $\mathbf{h}\in\mathcal{J}_0^{\mathcal{W}}\cap\partial\mathcal{D}_{\beta}(\mathcal{O}_{i, \alpha}^M), \beta\in[r_a, r_a + \gamma_s]$ and any $\mathbf{p}\in\mathcal{PJ}(\mathbf{0}, \partial\mathcal{D}_{\beta}(\mathcal{O}_{i, \alpha}^M))$, one has $d(\mathbf{h},\mathcal{B}_{2\delta}(\mathbf{p}))\geq\epsilon_h$, where
\begin{equation}
\epsilon_h = \sqrt{\left(d(\mathbf{0},\mathcal{O}_{\mathcal{W}}^M)^2 - r_a^2\right)} - d(\mathbf{0},\mathcal{D}_{r_a}(\mathcal{O}_{\mathcal{W}}^M)),\nonumber
\end{equation}
and $\delta = \min\{d(\mathbf{0}, \mathcal{O}_{\mathcal{W}}^M) - r_a, \beta - r_a\}.$

\textbf{Claim 2:} For $m\in\{-1, 1\}$, the set $\mathcal{H}_{\mathbf{p}}:=\mathcal{B}_{\delta}(\mathbf{p})\cap\mathcal{N}_{\gamma}(\mathcal{D}_{r_a}(\mathcal{O}_{i, \alpha}^M))\subset\mathcal{R}_a,$ where the location $\mathbf{p}$ and the scalar parameter $\delta$ are defined in claim 1 above.

Claim 1 states that for any connected modified obstacle $\mathcal{O}_{i, \alpha}^M$, the distance between any point $\mathbf{h}$, which is located in the jump set of the \textit{move-to-target} mode associated with this modified obstacle at some distance $\beta\in[r_a, r_a + \gamma_s]$ from it \textit{i.e.,} $\mathbf{h}\in\mathcal{J}_{0}^{\mathcal{W}}\cap\partial\mathcal{D}_{\beta}(\mathcal{O}_{i, \alpha}^M), \beta\in[r_a, r_a + \gamma_s]$, and 
the Euclidean balls of radius $2\delta$ centered at the set of projections of the target onto the set $\partial\mathcal{D}_{\beta}(\mathcal{O}_{i, \alpha}^M)$ is always greater than or equal to $\epsilon_h$ \textit{i.e.}, $d(\mathbf{h}, \mathcal{B}_{2\delta}(\mathbf{p}))\geq\epsilon_h,$ where $\mathbf{p}\in\mathcal{PJ}(\mathbf{0}, \mathcal{D}_{\beta}(\mathcal{O}_{i, \alpha}^M))$ and $\delta = \min\{\beta - r_a, d(\mathbf{0}, \mathcal{O}_{i, \alpha}^M) - r_a\}.$

Claim 2 states that the set $\mathcal{H}_{\mathbf{p}}$, which represents the intersection between the $\delta$-neighbourhood of the set of projections of the target onto the set $\partial\mathcal{D}_{\beta}(\mathcal{O}_{i, \alpha}^M)$ \textit{i.e.}, $\mathcal{B}_{\delta}(\mathbf{p}), \mathbf{p}\in\mathcal{PJ}(\mathbf{0}, \mathcal{D}_{\beta}(\mathcal{O}_{i, \alpha}^M))$ and the $\gamma-$neighbourhood of the dilated modified obstacle \textit{i.e.}, $\mathcal{N}_{\gamma}(\mathcal{D}_{r_a}(\mathcal{O}_{i, \alpha}^M))$ is subset of the \textit{always exit} region $\mathcal{R}_a$ \eqref{always_exit_region}. This implies that if $\bar\epsilon \in(0, \epsilon_h]$, then $\mathcal{H}_{\mathbf{p}}\subset\mathcal{ER}_m^{\mathbf{h}}$, where the set $\mathcal{ER}_m^{\mathbf{h}}$ is defined in \eqref{partition_rm}.
\subsubsection{Proof of claim 1}We aim to obtain an expression for $\norm{\mathbf{h}_{\beta}} - \norm{\mathbf{p}_{\beta}}$, for $\beta\in[r_a, r_a + \gamma_s]$, where
\begin{equation}
    \mathbf{h}_{\beta} = \underset{\mathbf{h}\in\mathcal{J}_{0}^{\mathcal{W}}\cap\partial\mathcal{D}_{\beta}(\mathcal{O}_{i, \alpha}^M)}{\text{argmin }}\norm{\mathbf{h}},\label{location_of_hmin}
\end{equation}
and the location $\mathbf{p}_{\beta}\in\mathcal{PJ}(\mathbf{0}, \partial\mathcal{D}_{\beta}(\mathcal{O}_{i, \alpha}^M))$.

Now, depending on the shape of the obstacle there can be two possibilities as follows:

\textbf{Case A: }When $\mathbf{h}_{\beta}\in\partial\mathcal{D}_{\beta}(\mathcal{O}_{i, \alpha}^M)\cap\mathcal{CH}(\mathbf{0}, \mathcal{O}_{i, \alpha}^M),$ as shown in Fig. \ref{lemma_2_image_1}. Notice that the line segment joining the location $\mathbf{h}_{\beta}$ and the origin passes through the modified obstacle $\mathcal{O}_{i, \alpha}^M$ \textit{i.e.}, $\mathcal{L}_s(\mathbf{h}_{\beta}, \mathbf{0})\cap\mathcal{O}_{i, \alpha}^M\ne \emptyset$. Hence, a part of this line segment belongs to the modified obstacle $\mathcal{O}_{i, \alpha}^M$. In other words, there exist locations $\mathbf{b}$ and $\mathbf{e}$, where $\mathbf{b}, \mathbf{e}\in\partial\mathcal{O}_{i, \alpha}^M$, such that $\mathcal{L}_s(\mathbf{b}, \mathbf{e})\subset\mathcal{L}_s(\mathbf{h}_{\beta}, \mathbf{0})\cap\mathcal{O}_{i, \alpha}^M$, as shown in Fig. \ref{lemma_2_image_1}. We further consider two more sub-cases based on the distance between the target $\mathbf{0}$ and the modified obstacle $\mathcal{O}_{i, \alpha}^M$. Note that, as per Assumption \ref{assumption:connected_interior}, $d(\mathbf{0}, \mathcal{O}_{i, \alpha}^M) > r_a.$

\textbf{Case A1: }When $d(\mathbf{0}, \mathcal{O}_{i, \alpha}^M) \geq \beta,$ as shown in Fig. \ref{lemma_2_image_1}a, one has $ \norm{\mathbf{h}_{\beta}} - \norm{\mathbf{p}_{\beta}} \geq \norm{\mathbf{h}_{\beta}} - \norm{\mathbf{e}} + \norm{\mathbf{b}}- \norm{\mathbf{p}_{\beta}}.$

Since $\mathbf{h}_{\beta}\in\partial\mathcal{D}_{\beta}(\mathcal{O}_{i, \alpha}^M), \{\mathbf{b}, \mathbf{e}\}\subset\partial\mathcal{O}_{i, \alpha}^M$ and $d(\mathbf{0}, \mathcal{O}_{i, \alpha}^M) \geq \beta$, one has $\norm{\mathbf{h}_{\beta}} - \norm{\mathbf{e}} \geq \beta$ and $\norm{\mathbf{b}} \geq \beta + \norm{\mathbf{p}_{\beta}}.$ Hence,
\begin{gather}
\norm{\mathbf{h}_{\beta}} - \norm{\mathbf{p}_{\beta}} \geq2\beta = 2(\beta - r_a) + 2r_a,\nonumber\\
d(\mathbf{h}_{\beta}, \mathcal{B}_{2(\beta - r_a)}(\mathbf{p}_{\beta})) \geq 2r_a.
\label{bound_1}
\end{gather}

\textbf{Case A2: }When $r_a < d(\mathbf{0}, \mathcal{O}_{i, \alpha}^M) < \beta,$ as shown in Fig. \ref{lemma_2_image_1}b, one has $ \norm{\mathbf{h}_{\beta}} - \norm{\mathbf{p}_{\beta}} \geq \norm{\mathbf{h}_{\beta}} - \norm{\mathbf{e}} + \norm{\mathbf{b}}- \norm{\mathbf{p}_{\beta}}.$

Since $\mathbf{h}_{\beta}\in\partial\mathcal{D}_{\beta}(\mathcal{O}_{i, \alpha}^M), \{\mathbf{b}, \mathbf{e}\}\subset\partial\mathcal{O}_{i, \alpha}^M$, $r_a < d(\mathbf{0}, \mathcal{O}_{i, \alpha}^M) < \beta$ and by construction, as shown in Fig. \ref{lemma_2_image_1}b, one has $\norm{\mathbf{h}_{\beta}} - \norm{\mathbf{e}} \geq \beta$ and $\norm{\mathbf{b}} \geq \beta - \norm{\mathbf{p}_{\beta}} = r_a + \norm{\mathbf{p}_{r_a}}.$ Hence, as $\norm{\mathbf{p}_{r_a}} = d(\mathbf{0}, \mathcal{D}_{r_a}(\mathcal{O}_{i, \alpha}^M))$, one has
\begin{gather}
\norm{\mathbf{h}_{\beta}} - \norm{\mathbf{p}_{\beta}} \geq2\beta -2\norm{\mathbf{p}_{\beta}} = 2\norm{\mathbf{p}_{r_a}} + 2r_a,\nonumber\\
d(\mathbf{h}_{\beta}, \mathcal{B}_{2d(\mathbf{0}, \mathcal{D}_{r_a}(\mathcal{O}_{i, \alpha}^M))}(\mathbf{p}_{\beta})) > 2r_a.\label{bound_2}
\end{gather}

\textbf{Case B: }When $\mathbf{h}_{\beta}\in\mathcal{G}$, where the set $\mathcal{G}:=\left(\partial\mathcal{D}_{\beta}(\mathcal{O}_{i, \alpha}^M)\cap\mathcal{CH}(\mathbf{0}, \mathcal{D}_{r_a}(\mathcal{O}_{i, \alpha}^M))\right)\setminus\mathcal{CH}(\mathbf{0}, \mathcal{O}_{i, \alpha}^M)$. In other words, when the line segment joining the locations $\mathbf{h}_{\beta}$ and the target at the origin does not intersect with the interior of the modified obstacle $\mathcal{O}_{i, \alpha}^M$ \textit{i.e.}, $\mathcal{L}_s(\mathbf{h}_{\beta}, \mathbf{0})\cap(\mathcal{O}_{i, \alpha}^M)^{\circ} = \emptyset,$ as shown in Fig. \ref{lemma_2_image_1}. To proceed with the proof, we use the following fact which states that the projection of the target location onto the set $\mathcal{G}$ always belongs to the intersection between the boundary of the conic hull $\mathcal{CH}(\mathbf{0}, \mathcal{D}_{r_a}(\mathcal{D}_{r_a}(\mathcal{O}_{i, \alpha}^M)))$ and the set $G$.

\textbf{Fact 2:} The set $\mathcal{PJ}(\mathbf{0}, \mathcal{G})\subset\mathcal{G}\cap\partial\mathcal{CH}(\mathbf{0},\mathcal{D}_{r_a}(\mathcal{O}_{i, \alpha}^M)).$


\textit{Sketch of the proof}: The proof is by contradiction. We assume that there exists a location $\mathbf{x}$ in the set $\mathcal{PJ}(\mathbf{0}, \mathcal{G})$ that does not belong to the intersection between the set $\mathcal{G}$ and the boundary of the conic hull to the set $\mathcal{D}_{r_a}(\mathcal{O}_{i, \alpha}^M)$ \textit{i.e.}, $\mathbf{x}\in\mathcal{PJ}(\mathbf{0}, \mathcal{G})$ and $\mathbf{x}\notin\mathcal{G}\cap\partial\mathcal{CH}(\mathbf{0}, \mathcal{D}_{r_a}(\mathcal{O}_{i, \alpha}^M))$. We know that the curve $\mathcal{G}$ belongs to the boundary of the $\beta-$dilated modified obstacle \textit{i.e.}, $\mathcal{G}\subset\partial\mathcal{D}_{\beta}(\mathcal{O}_{i, \alpha}^M)$. As a result, there exists a partial section of the boundary of the modified obstacle $\mathcal{O}_{i, \alpha}^M$, let say $\mathcal{M}\subset\partial\mathcal{O}_{i, \alpha}^M$ with $\mathbf{x}\in\mathcal{D}_{\beta}(\mathcal{M})$ such that this curve $\mathcal{M}$ belongs to the relative complement of the conic hull $\mathcal{CH}(\mathbf{0}, \mathcal{O}_{i, \alpha}^M)$ with respect to the conic hull $\mathcal{CH}(\mathbf{0}, \mathcal{D}_{r_a}(\mathcal{O}_{i, \alpha}^M))$ \textit{i.e.}, $\mathcal{M}\in\mathcal{CH}(\mathbf{0}, \mathcal{D}_{r_a}(\mathcal{O}_{i, \alpha}^M))\setminus\mathcal{CH}(\mathbf{0}, \mathcal{O}_{i, \alpha}^M)$. However, by construction of the set $\mathcal{G}$, the intersection between the modified obstacle $\mathcal{O}_{i, \alpha}^M$ and the region $\mathcal{CH}(\mathbf{0}, \mathcal{D}_{r_a}(\mathcal{O}_{i, \alpha}^M))\setminus\mathcal{CH}(\mathbf{0}, \mathcal{O}_{i, \alpha}^M)$ must be an empty set. Therefore, we arrive at a contradiction.

According to Fact 2, $\mathbf{h}_{\beta}\in \mathcal{G}\cap\partial\mathcal{CH}(\mathbf{0}, \mathcal{D}_{r_a}(\mathcal{D}_{r_a}(\mathcal{O}_{i, \alpha}^M)))$, as shown in Fig. \ref{lemma_2_image_2}. Similar to case A, we consider two more sub-cases based on the distance between the target $\mathbf{0}$ and the modified obstacle $\mathcal{O}_{i, \alpha}^M$.

\textbf{Case B1: }When $d(\mathbf{0}, \mathcal{O}_{i, \alpha}^M) \geq \beta, $ as shown in Fig \ref{lemma_2_image_2}a, one has
\begin{equation}
\norm{\mathbf{h}_{\beta}} - \norm{\mathbf{p}_{\beta}}= \norm{\mathbf{h}_{\beta}} - \norm{\mathbf{h}_{r_a}} + \norm{\mathbf{h}_{r_a}} - \norm{\mathbf{p}_{\beta}}.\label{proof_lemma_2_1}
\end{equation}
Since $\mathbf{h}_{\beta}\in\partial\mathcal{D}_{\beta}(\mathcal{O}_{i, \alpha}^M)$ and  $\mathbf{h}_{r_a}\in\partial\mathcal{D}_{r_a}(\mathcal{O}_{i, \alpha}^M)$,  
one has $\norm{\mathbf{h}_{\beta}} - \norm{\mathbf{h}_{r_a}} \geq\beta - r_a$. Let $\mathbf{a}\in\partial\mathcal{O}_{i, \alpha}^M$ be the location such that the lines $\mathcal{L}(\mathbf{a}, \mathbf{h}_{r_a})$ and $\mathcal{L}(\mathbf{h}_{r_a}, \mathbf{0})$ are perpendicular
, as shown in Fig. \ref{lemma_2_image_1}a. Note that, as the line $\mathcal{L}(\mathbf{h}_{r_a}, \mathbf{0})$ is tangent to the set $\mathcal{D}_{r_a}(\mathcal{O}_{i, \alpha}^M)$ at $\mathbf{h}_{r_a}$, one has $d(\mathbf{h}_{r_a}, \mathbf{a}) = r_a.$ Hence, \begin{align}
    \norm{\mathbf{h}_{r_a}} &= \sqrt{\norm{\mathbf{a}}^2 - d(\mathbf{h}_{r_a}, \mathbf{a})^2}
    \geq\sqrt{d(\mathbf{0}, \mathcal{O}_{i, \alpha}^M)^2 - r_a^2}.\label{proof_lemma_2_2}
\end{align}
After substituting \eqref{proof_lemma_2_2} in \eqref{proof_lemma_2_1}, one gets
\begin{equation}
    \norm{\mathbf{h}_{\beta}} - \norm{\mathbf{p}_{\beta}} \geq \beta - r_a + \sqrt{d(\mathbf{0}, \mathcal{O}_{i, \alpha}^M)^2 - r_a^2} - \norm{\mathbf{p}_{\beta}}.\nonumber
\end{equation}
Since $d(\mathbf{0}, \mathcal{O}_{i, \alpha}^M)\geq\beta$, $\norm{\mathbf{p}_{r_a}} - \norm{\mathbf{p}_{\beta}} = \beta - r_a. $ Hence, one has
\begin{gather}
    \norm{\mathbf{h}_{\beta}} - \norm{\mathbf{p}_{\beta}} \geq 2(\beta - r_a) + \sqrt{d(\mathbf{0}, \mathcal{O}_{i, \alpha}^M)^2 - r_a^2} - \norm{\mathbf{p}_{r_a}},\nonumber\\
    d(\mathbf{h}_{\beta}, \mathcal{B}_{2(\beta - r_a)}(\mathbf{p}_{\beta})) \geq \sqrt{d(\mathbf{0}, \mathcal{O}_{i, \alpha}^M)^2 - r_a^2} - \norm{\mathbf{p}_{r_a}}.\label{bound_3}
\end{gather}

\textbf{Case B2: }When $r_a < d(\mathbf{0},\mathcal{O}_{i, \alpha}^M) < \beta, $ as shown in Fig. \ref{lemma_2_image_2}b. According to \eqref{proof_lemma_2_1}, \eqref{proof_lemma_2_2} and the fact that $\norm{\mathbf{h}_{\beta}} - \norm{\mathbf{h}_{r_a}} \geq \beta - r_a$, one has
\begin{equation}
    \norm{\mathbf{h}_{\beta}} - \norm{\mathbf{p}_{\beta}} \geq \beta - r_a + \sqrt{d(\mathbf{0},\mathcal{O}_{i, \alpha}^M)^2 - r_a^2} - \norm{\mathbf{p}_{\beta}}\nonumber,
\end{equation}
Since $r_a < d(\mathbf{0}, \mathcal{O}_{i, \alpha}^M)< \beta$, $\norm{\mathbf{p}_{\beta}} + \norm{\mathbf{p}_{r_a}} = \beta - r_a. $ Hence, one has
\begin{gather}
    \norm{\mathbf{h}_{\beta}} - \norm{\mathbf{p}_{\beta}} \geq 2\norm{\mathbf{p}_{r_a}} + \sqrt{d(\mathbf{0}, \mathcal{O}_{i, \alpha}^M)^2 - r_a^2} - \norm{\mathbf{p}_{r_a}},\nonumber\\
    d(\mathbf{h}_{\beta}, \mathcal{B}_{2\norm{\mathbf{p}_{r_a}}}(\mathbf{p}_{\beta})) \geq \sqrt{d(\mathbf{0}, \mathcal{O}_{i, \alpha}^M)^2 - r_a^2} - \norm{\mathbf{p}_{r_a}}.\label{bound_4}
\end{gather}

Now, considering all the obstacles \textit{i.e.}, for all $i\in\mathbb{I}$, according to \eqref{bound_1}, \eqref{bound_2}, \eqref{bound_3} and \eqref{bound_4}, one has
\begin{equation}
    d(\mathbf{h}_{\beta}, \mathcal{B}_{2\delta}(\mathbf{p}_{\beta})) \geq \epsilon_h,
\end{equation}
where $\delta  = \min\left\{\beta - r_a, \underset{i\in\mathbb{I}}{\min}\{d(\mathbf{0}, \mathcal{O}_{i, \alpha}^M) - r_a\}\right\} = \min\{\beta-r_a, d(\mathbf{0}, \mathcal{O}_{\mathcal{W}}^M)-r_a\}$ and $\epsilon_h$ is evaluated as
\begin{align}
\epsilon_h &= \min\left\{2r_a, \sqrt{d(\mathbf{0}, \mathcal{O}_{\mathcal{W}}^M)^2 - r_a^2} - d(\mathbf{0}, \mathcal{D}_{r_a}(\mathcal{O}_{\mathcal{W}}^M))\right\}.\nonumber
\end{align}
It is clear that for $n > 0$, the function $f(k) = \sqrt{k^2 - n^2} - (k - n), d\in[n, \infty),$ is monotonically increasing and $\underset{k\rightarrow\infty}{\lim}f(k) = n$. Hence, according to Assumption \ref{assumption:connected_interior}, as $d(\mathbf{0}, \mathcal{O}_{\mathcal{W}}^M) > r_a > 0$, one has
\begin{equation}
    \sqrt{\left(d(\mathbf{0},\mathcal{O}_{\mathcal{W}}^M)^2 - r_a^2\right)} - d(\mathbf{0},\mathcal{D}_{r_a}(\mathcal{O}_{\mathcal{W}}^M)) < r_a,\nonumber
\end{equation}
irrespective of the locations of obstacles relative to the target location. As a result,
\begin{align}
\epsilon_h &= \min\{2r_a, \sqrt{d(\mathbf{0}, \mathcal{O}_{\mathcal{W}}^M)^2 - r_a^2} - d(\mathbf{0}, \mathcal{D}_{r_a}(\mathcal{O}_{\mathcal{W}}^M))\}\nonumber\\
&= \sqrt{\left(d(\mathbf{0},\mathcal{O}_{\mathcal{W}}^M)^2 - r_a^2\right)} - d(\mathbf{0},\mathcal{D}_{r_a}(\mathcal{O}_{\mathcal{W}}^M)). \nonumber
\end{align}

\subsubsection{Proof of claim 2}
In this proof, our goal is to show that the set $\mathcal{H}_{\mathbf{p}} = \mathcal{B}_{\delta}(\mathbf{p})\cap\mathcal{N}_{\gamma}(\mathcal{D}_{r_a}(\mathcal{O}_{i, \alpha}^M))$ belongs to the \textit{always exit} region $\mathcal{R}_a$, for any  $\mathbf{p}\in\mathcal{PJ}(\mathbf{0}, \partial\mathcal{D}_{\beta}(\mathcal{O}_{i, \alpha}^M))$ and the parameter $\delta = \min\{d(\mathbf{0}, \mathcal{O}_{\mathcal{W}}^M) - r_a, \beta - r_a\}$, where $\beta\in[r_a, r_a + \gamma_s]$. Based on the distance of the target location from the modified obstacle $\mathcal{O}_{i, \alpha}^M$, we consider two cases as follows:

\textbf{Case 1: }When $d(\mathbf{0}, \mathcal{O}_{i, \alpha}^M) \geq \beta$, as shown in Fig. \ref{lemma_2_image_1}a and Fig. \ref{lemma_2_image_2}a. Since $d(\mathbf{0}, \mathcal{O}_{i, \alpha}^M) \geq \beta$, for any $\mathbf{p}\in\mathcal{PJ}(\mathbf{0}, \mathcal{D}_{\beta}(\mathcal{O}_{i, \alpha}^M))$, one has $\mathcal{B}_{\norm{\mathbf{p}}+ \delta}(\mathbf{0})\subset\mathcal{B}_{d(\mathbf{0},\mathcal{D}_{r_a}(\mathcal{O}_{i, \alpha}^M))}(\mathbf{0})$, where $\delta = \min\{\beta - r_a, d(\mathbf{0}, \mathcal{O}_{i, \alpha}^M) - r_a\} = \beta - r_a.$ Hence, for $\mathbf{p}\in\mathcal{PJ}(\mathbf{0}, \mathcal{D}_{\beta}(\mathcal{O}_{i, \alpha}^M))$, one has $\mathcal{B}_{\delta}(\mathbf{p})\subset\mathcal{B}_{d(\mathbf{0},\mathcal{D}_{r_a}(\mathcal{O}_{i, \alpha}^M))}(\mathbf{0})$. Moreover, $\mathcal{B}_{d(\mathbf{0},\mathcal{D}_{r_a}(\mathcal{O}_{i, \alpha}^M))}(\mathbf{0})\cap\left(\mathcal{D}_{r_a}(\mathcal{O}_{i, \alpha}^M)\right)^{\circ} = \emptyset$. Therefore, according to \eqref{always_exit_region}, one can conclude that $\mathcal{H}_{\mathbf{p}}\subset\mathcal{R}_{a}$.
 
\textbf{Case 2: }When $r_a < d(\mathbf{0}, \mathcal{O}_{i, \alpha}^M) < \beta$, as shown in Fig. \ref{lemma_2_image_1}b and Fig. \ref{lemma_2_image_2}b. According to Lemmas \ref{lemma:reach_extends} and \ref{lemma:unique_projection}, $\mathbf{card}(\mathcal{PJ}(\mathbf{0}, \mathcal{D}_{r_a}(\mathcal{O}_{i, \alpha}^M))) = 1$, \textit{i.e.}, $\Pi(\mathbf{0}, \mathcal{D}_{r_a}(\mathcal{O}_{i, \alpha}^M))$ is unique. Since, as per Lemma \ref{lemma:reach_extends}, $\mathbf{reach}(\mathcal{D}_{r_a}(\mathcal{O}_{i, \alpha}^M)) \geq \alpha - r_a$, according to \cite[Lemma 4.5]{rataj2019curvature}, one has $(\mathcal{B}_{\alpha - r_a}(\mathbf{q}))^{\circ}\cap\mathcal{D}_{r_a}(\mathcal{O}_{i, \alpha}^M)= \emptyset,$ where $\mathbf{q} = \Pi(\mathbf{0}, \mathcal{D}_{r_a}(\mathcal{O}_{i, \alpha}^M)) + (\alpha - r_a)\frac{\mathbf{0}- \Pi(\mathbf{0}, \mathcal{D}_{r_a}(\mathcal{O}_{i, \alpha}^M))}{\norm{\mathbf{0} - \Pi(\mathbf{0}, \mathcal{D}_{r_a}(\mathcal{O}_{i, \alpha}^M))}}.$ Now, notice that, for $\mathbf{p}\in\partial\mathcal{D}_{\beta}(\mathcal{O}_{i, \alpha}^M)\cap\mathcal{L}_s(\mathbf{q}, \Pi(\mathbf{0}, \mathcal{D}_{r_a}(\mathcal{O}_{i, \alpha}^M)))$, one has $\mathcal{B}_{\delta}(\mathbf{p})\subset\mathcal{B}_{\alpha - r_a}(\mathbf{q})$, where $\delta = \min\{\beta - r_a, d(\mathbf{0}, \mathcal{O}_{i, \alpha}^M)-r_a\} = d(\mathbf{0}, \mathcal{O}_{i, \alpha}^M)-r_a < \alpha - r_a$. Hence, if one shows that $\mathbf{p}\in\mathcal{PJ}(\mathbf{0}, \partial\mathcal{D}_{\beta}(\mathcal{O}_{i, \alpha}^M))$ and $\mathbf{card}(\mathcal{PJ}(\mathbf{0}, \partial\mathcal{D}_{\beta}(\mathcal{O}_{i, \alpha}^M))) = 1$, then, according to \eqref{always_exit_region}, it is straightforward to notice that $\mathcal{H}_{\mathbf{p}}\subset\mathcal{R}_a$, where the set $\mathcal{H}_{\mathbf{p}}$ is defined in claims 1 and 2.

To this end, let us assume that $d(\mathbf{0}, \mathcal{D}_{r_a}(\mathcal{O}_{i, \alpha}^M)) = \eta$, where $r_a < \eta < \beta < \alpha.$ Notice that, $\mathbf{0}\in\mathcal{L}(\Pi(\mathbf{0}, \mathcal{D}_{r_a}(\mathcal{O}_{i, \alpha}^M)), \mathbf{q}).$ Since, $\mathbf{p}\in\mathcal{L}_s(\mathbf{q}, \Pi(\mathbf{0}, \mathcal{D}_{r_a}(\mathcal{O}_{i, \alpha}^M)))\cap\partial\mathcal{D}_{\beta}(\mathcal{O}_{i, \alpha}^M)$, one has $d(\mathbf{p}, \Pi(\mathbf{0}, \mathcal{D}_{r_a}(\mathcal{O}_{i, \alpha}^M))) = \beta - r_a.$ Therefore, it is clear that $d(\mathbf{0}, \mathbf{p}) = \beta - \eta.$ Moreover, according to Lemmas \ref{lemma:reach_property} and \ref{lemma:unique_projection}, the set $\mathcal{M} := (\mathcal{O}_{i, \alpha}^M\oplus\mathcal{B}_{\alpha}(\mathbf{0}))^c$ has the reach greater than or equal to $\alpha$, \textit{i.e.}, $\mathbf{reach}(\mathcal{M})\geq\alpha.$ Now, notice that, for $\beta\in[r_a, r_a + \gamma_s], \partial\mathcal{D}_{\beta}(\mathcal{O}_{i, \alpha}^M) = \partial\mathcal{D}_{\alpha-\beta}(\mathcal{M})$. Hence, $\mathbf{0}\in\partial\mathcal{D}_{\alpha - \eta}(\mathcal{M})$ and $\mathbf{p}\in\partial\mathcal{D}_{\alpha - \beta}(\mathcal{M})$. Then, according to Lemma \ref{lemma:reach_extends}, $\mathbf{card}(\mathcal{PJ}(\mathbf{0}, \partial\mathcal{D}_{\alpha - \beta}(\mathcal{M}))) = \mathbf{card}(\mathcal{PJ}(\mathbf{0}, \partial\mathcal{D}_{\beta}(\mathcal{O}_{i, \alpha}^M))) = 1$. Moreover, as $d(\mathbf{0}, \mathbf{p}) = \beta - \eta,$ one can conclude that $\mathbf{p}\in\mathcal{PJ}(\mathbf{0}, \partial\mathcal{D}_{\beta}(\mathcal{O}_{i, \alpha}^M))$.

\begin{figure}
    \centering
    \includegraphics[width = 0.7\linewidth]{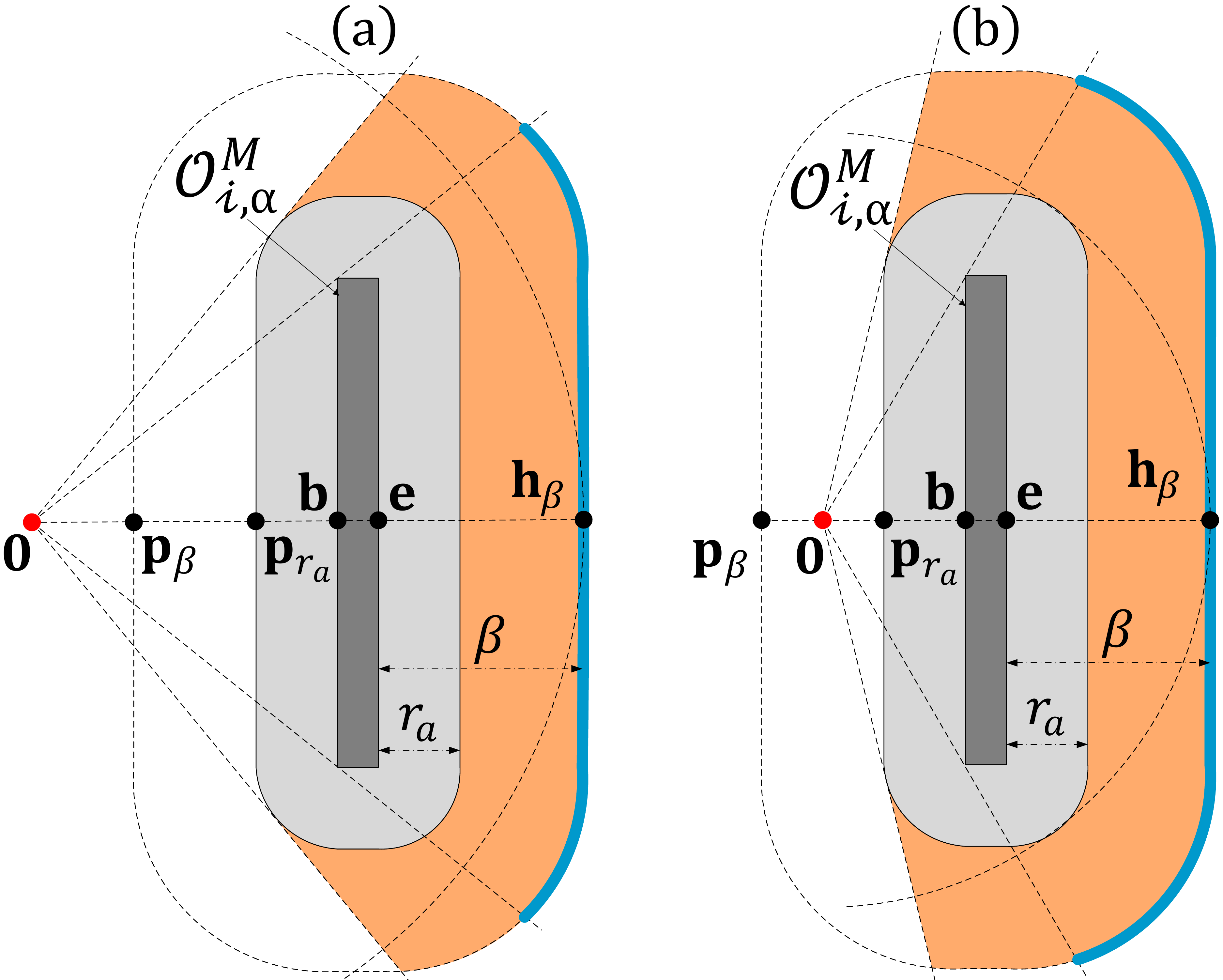}
    \caption{The diagrammatic representation for the Case A in the proof of Lemma \ref{lemma:bar_epsilon}. (a) $d(\mathbf{0}, \mathcal{O}_{i, \alpha}^M)\geq \beta$, (b) $r_a < d(\mathbf{0}, \mathcal{O}_{i, \alpha}^M) < \beta.$}
    \label{lemma_2_image_1}
\end{figure}

\begin{figure}
    \centering
    \includegraphics[width = 1\linewidth]{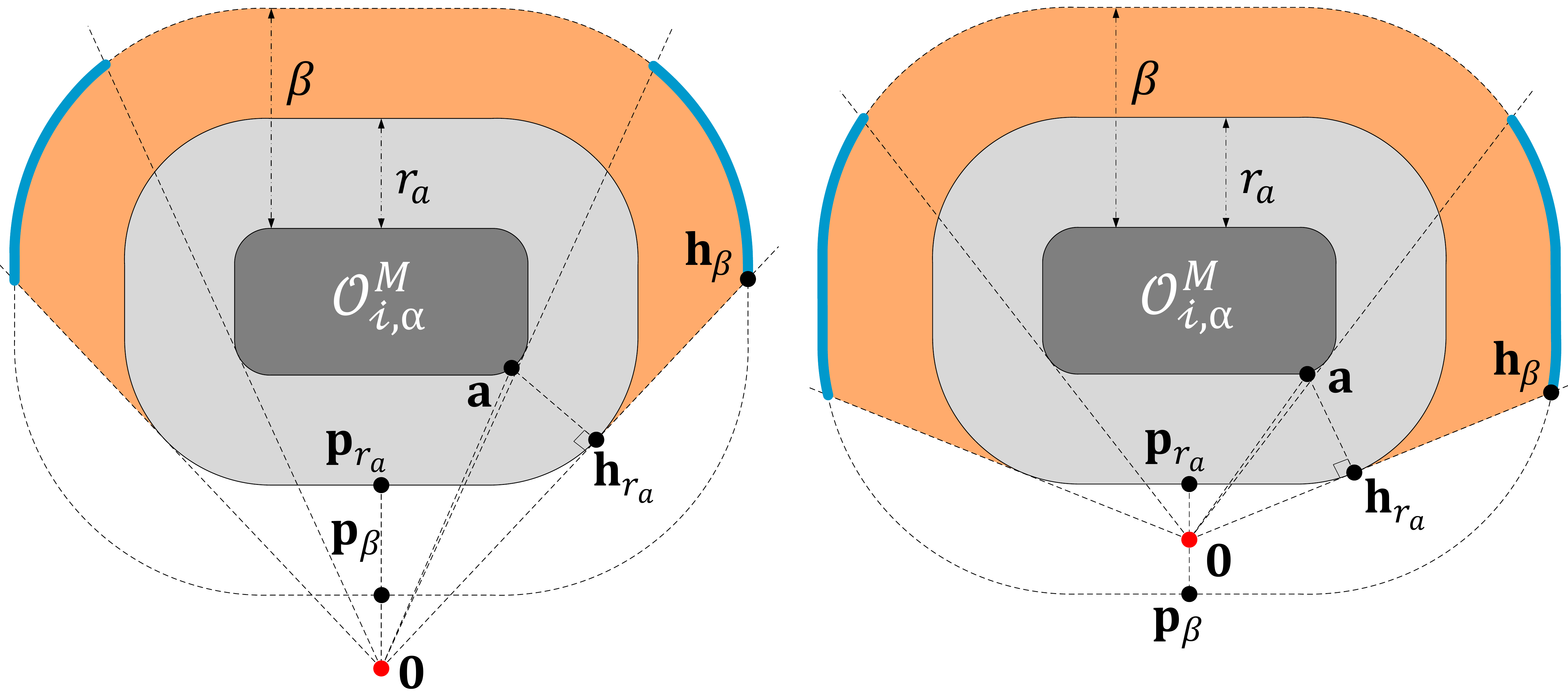}
    \caption{The diagrammatic representation for the Case B in the proof of Lemma \ref{lemma:bar_epsilon}. (a) $d(\mathbf{0}, \mathcal{O}_{i, \alpha}^M)\geq \beta$, (b) $r_a < d(\mathbf{0}, \mathcal{O}_{i, \alpha}^M) < \beta.$}
    \label{lemma_2_image_2}
\end{figure}

\subsection{Proof of Lemma \ref{lemma:forward_invariance}}
\label{proof:forward_invariance}
First we prove that the union of the flow and jump sets covers exactly the obstacle-free state space $\mathcal{K}$. For $m = 0$, according to \eqref{stabilization_flow_set} and \eqref{stabilization_jump_set}, by construction we have $\mathcal{F}_{0}^{\mathcal{W}}\cup\mathcal{J}_{0}^{\mathcal{W}} = \mathcal{V}_{r_a}.$ Similarly, for $m\in\{-1, 1\}$, according to \eqref{avoidance_flow_set} and \eqref{avoidance_jump_set}, by construction we have $\mathcal{F}_m^{\mathcal{W}}\cup\mathcal{J}_m^{\mathcal{W}} = \mathcal{V}_{r_a}$. Inspired by \cite[Appendix 11]{berkane2021arxiv}, the satisfaction of the following equation:
\begin{equation}
    \mathcal{F}_m^{\mathcal{W}}\cup\mathcal{J}_m^{\mathcal{W}}= \mathcal{V}_{r_a}, m\in\mathbb{M},
\end{equation}
along with \eqref{bothmodes_flowjumpset} and \eqref{bothmodes_flowjumpset1} implies $\mathcal{F}\cup\mathcal{J} = \mathcal{K}.$

Now, inspired by \cite[Appendix 1]{berkane2021arxiv}, for the hybrid closed-loop system \eqref{hybrid_closed_loop_system}, with data $\mathcal{H} = (\mathcal{F}, \mathbf{F}, \mathcal{J}, \mathbf{J}),$ define $\mathbf{S}_{\mathcal{H}}(\mathcal{K})$ as the set of all maximal solutions $\xi$ to $\mathcal{H}$ with $\xi(0, 0) \in\mathcal{K}$. Since $\mathcal{F}\cup\mathcal{J} = \mathcal{K}$, each $\xi\in\mathbf{S}_{\mathcal{H}}(\mathcal{K})$ has range $\text{rge } \xi\subset\mathcal{K}.$ Additionally, if every maximal solution $\xi\in\mathbf{S}_{\mathcal{H}}(\mathcal{K})$ is complete, then the set $\mathcal{K}$ will be forward invariant \cite[Definition 3.13]{sanfelice2021hybrid}. Since the hybrid closed-loop system \eqref{hybrid_closed_loop_system} satisfies hybrid basic conditions, as stated in Lemma \ref{nonconvex_hybrid_basic}, one can use \cite[Proposition 6.10]{goebel2012hybrid}, to verify the following viability condition:
\begin{equation}
    \mathbf{F}(\mathbf{x}, \mathbf{h}, m) \cap\mathbf{T}_{\mathcal{F}}(\mathbf{x}, \mathbf{h}, m) \ne \emptyset, \forall(\mathbf{x}, \mathbf{h}, m)\in\mathcal{F}\setminus\mathcal{J},\label{viability_condition}
\end{equation}
which will allow us to establish the completeness of the solution $\xi$ to the hybrid closed-loop system \eqref{hybrid_closed_loop_system}. In \eqref{viability_condition}, $\mathbf{T}_{\mathcal{F}}(\mathbf{x}, \mathbf{h}, m)$ represents the tangent cone to the set $\mathcal{F}$ at $(\mathbf{x}, \mathbf{h}, m)$, as defined in Section \ref{section:normal_and_tangent_cone}. 

Let $(\mathbf{x}, \mathbf{h}, m)\in\mathcal{F}\setminus\mathcal{J}$, which implies by \eqref{bothmodes_flowjumpset} and \eqref{bothmodes_flowjumpset1} that $(\mathbf{x}, \mathbf{h})\in(\mathcal{F}_m^{\mathcal{W}}\setminus\mathcal{J}_m^{\mathcal{W}})\times\mathcal{V}_{r_a}$ for some $m\in\mathbb{M}$. 
 For $\mathbf{x}\in(\mathcal{F}_m^{\mathcal{W}})^{\circ}\setminus\mathcal{J}_m^{\mathcal{W}}, m\in\mathbb{M},$ $\mathbf{T}_{\mathcal{F}}(\xi) = \mathbb{R}^2\times\mathcal{HP}\times\{0\},$ where the set $\mathcal{HP}$ is given by
\begin{equation}
    \mathcal{HP} = \begin{cases}
    \begin{matrix*}[l] \mathbb{R}^2, & \text{if }\mathbf{h}\in(\mathcal{V}_{r_a})^{\circ},\\
    \mathcal{P}_{\geq}(\mathbf{0}, (\mathbf{h} - \Pi(\mathbf{h}, \mathcal{O}_{\mathcal{W}}^M))),& \text{if }\mathbf{h}\in\partial\mathcal{V}_{r_a},\end{matrix*}
    \end{cases}\label{definition_hp}
\end{equation}
where, according to Lemma \ref{lemma:unique_projection}, for $\mathbf{h}\in\partial\mathcal{V}_{r_a}$, the projection $\Pi(\mathbf{h}, \mathcal{O}_{\mathcal{W}}^M)$ is unique. Since, according to \eqref{hybrid_closed_loop_system}, $\dot{\mathbf{h}} = \mathbf{0}$, $\dot{\mathbf{h}}\in\mathcal{HP}$ and \eqref{viability_condition} holds,

For $m = 0$, according to \eqref{stabilization_jump_set} and \eqref{stabilization_flow_set}, one has
\begin{equation}
    \partial\mathcal{F}_0^{\mathcal{W}}\setminus\mathcal{J}_0^{\mathcal{W}} \subset \partial\mathcal{D}_{r_a}(\mathcal{O}_{\mathcal{W}}^M)\cap\mathcal{R}_e,
\end{equation}
and according to Lemma \ref{lemma:unique_projection}, for $\mathbf{x}\in\partial\mathcal{F}_0^{\mathcal{W}}\setminus\mathcal{J}_0^{\mathcal{W}} $, the set $\mathcal{PJ}(\mathbf{x}, \mathcal{O}_{\mathcal{W}}^M)$ is a singleton \textit{i.e.,} the projection $\Pi(\mathbf{x}, \mathcal{O}_{\mathcal{W}}^M)$ is unique. Hence, $\forall\mathbf{x}\in\partial\mathcal{F}_0^{\mathcal{W}}\setminus\mathcal{J}_0^{\mathcal{W}}$,
\begin{equation}
    \mathbf{T}_{\mathcal{F}}(\mathbf{x}, \mathbf{h}, 0) = \mathcal{P}_{\geq}(\mathbf{0}, (\mathbf{x} - \Pi(\mathbf{x}, \mathcal{O}_{\mathcal{W}}^M)))\times\mathcal{HP}\times\{0\},
\end{equation}
where $\mathcal{HP}$ is defined in \eqref{definition_hp}. Also, according to \eqref{hybrid_control_input_1}, for $m =0$, $\mathbf{u}(\mathbf{x}, \mathbf{h}, 0) = -\kappa_s\mathbf{x}, \kappa_s>0$. According to \eqref{definition:exit_region}, $\mathbf{u}(\mathbf{x}, \mathbf{h}, 0)\in\mathcal{P}_{\geq}(\mathbf{0}, (\mathbf{x} - \Pi(\mathbf{x}, \mathcal{O}_{\mathcal{W}}^M)))$ and, according to \eqref{hybrid_closed_loop_system}, $\dot{\mathbf{h}} = \mathbf{0}\in\mathcal{HP}$. As a result, viability condition \eqref{viability_condition} holds for $m =0$.


For $m\in\{-1, 1\}$, according to \eqref{avoidance_flow_set} and \eqref{avoidance_jump_set} one has
\begin{equation}
    \partial\mathcal{F}_{m}^{\mathcal{W}}\setminus\mathcal{J}_m^{\mathcal{W}}\subset\partial\mathcal{D}_{r_a}(\mathcal{O}_{\mathcal{W}}^M),
\end{equation}
and according to Lemma \ref{lemma:unique_projection}, for $\mathbf{x}\in\partial\mathcal{F}_m^{\mathcal{W}}\setminus\mathcal{J}_m^{\mathcal{W}}, m\in\{-1, 1\}$, the set $\mathcal{PJ}(\mathbf{x}, \mathcal{O}_{\mathcal{W}}^M)$ is a singleton \textit{i.e.}, the projection $\Pi(\mathbf{x}, \mathcal{O}_{\mathcal{W}}^M)$ is unique and the circle $\mathcal{B}_{\norm{\mathbf{x} - \Pi(\mathbf{x}, \mathcal{O}_{\mathcal{W}}^M)}}(\mathbf{x})$ intersect with $\partial\mathcal{O}_{\mathcal{W}}^M$ only at the location $\Pi(\mathbf{x}, \mathcal{O}_{\mathcal{W}}^M).$ Hence, $\forall\mathbf{x}\in\partial\mathcal{F}_m^{\mathcal{W}}\setminus\mathcal{J}_m^{\mathcal{W}}, m\in\{-1, 1\},$
\begin{equation}
    \mathbf{T}_{\mathcal{F}}(\mathbf{x}, \mathbf{h}, 0) = \mathcal{P}_{\geq}(\mathbf{0}, (\mathbf{x} - \Pi(\mathbf{x}, \mathcal{O}_{\mathcal{W}}^M)))\times\mathcal{HP}\times\{0\},
\end{equation}
where $\mathcal{HP}$ is defined in \eqref{definition_hp}. Also for $m\in\{-1, 1\}$, according to \eqref{hybrid_control_input_1}, $\mathbf{u}(\mathbf{x}, \mathbf{h}, m) = \kappa_r\mathbf{v}(\mathbf{x}, m), \kappa_r > 0.$ Since, according to \eqref{definition:vxm}, $\mathbf{u}(\mathbf{x}, \mathbf{h}, m)^\intercal(\mathbf{x}- \Pi(\mathbf{x}, \mathcal{O}_{\mathcal{W}}^M)) = 0$ and, according to \eqref{hybrid_closed_loop_system}, $\dot{\mathbf{h}}= \mathbf{0}\in\mathcal{HP},$ the viability condition in \eqref{viability_condition} holds true for $m\in\{-1, 1\}.$

Hence, according to \cite[Proposition 6.10]{goebel2012hybrid}, since \eqref{viability_condition} holds for all $\xi\in\mathcal{F}\setminus\mathcal{J}$, there exists a nontrivial solution to $\mathcal{H}$ for each initial condition in $\mathcal{K}$. Finite escape time can only occur through flow. They can neither occur for $\mathbf{x}$ in the set $\mathcal{F}_{-1}^{\mathcal{W}}\cup\mathcal{F}_1^{\mathcal{W}}$, as this set is bounded as per definition \eqref{avoidance_flow_set}, nor for $\mathbf{x}$ in the set $\mathcal{F}_0^{\mathcal{W}}$ as this would make $\mathbf{x}^{\intercal}\mathbf{x}$ grow unbounded, and would contradict the fact that $\frac{d}{dt}(\mathbf{x}^\intercal\mathbf{x})\leq0$ in view of the definition of $\mathbf{u}(\mathbf{x}, \mathbf{h}, 0)$. Therefore, all maximal solutions do not have finite escape times. Furthermore, according to \eqref{hybrid_closed_loop_system}, $\mathbf{x}^+ = \mathbf{x},$ and from the definition of the update law in \eqref{updatelaw_part1} and \eqref{updatelaw_part2}, it follows immediately that $\mathbf{J}(\mathcal{J})\subset\mathcal{K}$. Hence, the solutions to the hybrid closed-loop system \eqref{hybrid_closed_loop_system} cannot leave $\mathcal{K}$ through jump and, as per \cite[Proposition 6.10]{goebel2012hybrid}, all maximal solutions are complete.

\subsection{Proof of Theorem \ref{theorem:global_stability}}
\label{proof:global_stability}

\textbf{Forward invariance and stability:} The forward invariance of the obstacle-free set $\mathcal{K}$, for the hybrid closed-loop system \eqref{hybrid_closed_loop_system}, is immediate from Lemma \ref{lemma:forward_invariance}. We next prove stability of $\mathcal{A}$ using \cite[Definition 7.1]{goebel2012hybrid}.


According to Lemma \ref{lemma:properties_of_modified_workspace}, the target location $\mathbf{0}\in(\mathcal{V}_{r_a})^{\circ}$. As a result, there exists $\mu_1 > 0$ such that
$\mathcal{B}_{\mu_1}(\mathbf{0})\cap(\mathcal{D}_{r_a}(\mathcal{O}_{\mathcal{W}}^M))^{\circ} = \emptyset.$ According to  \eqref{stabilization_jump_set}, there exists $\mu_2 > 0$ such that $\mathcal{B}_{\mu_2}(\mathbf{0})\cap\mathcal{J}_0^{\mathcal{W}}=\emptyset$. Additionally, as per \eqref{avoidance_jump_set}, there exists $\mu_3>0$ such that $\mathcal{B}_{\mu_3}(\mathbf{0})\subset\mathcal{J}_{m}^{\mathcal{W}}$ for $m\in\{-1, 1\}.$ We define a set $\mathcal{S}_{\mu}:= \mathcal{B}_{\mu}(\mathbf{0})\times\mathcal{V}_{r_a}\times\mathbb{M},$ where $\mu\in(0, \min\{\mu_1, \mu_2, \mu_3\}).$ 
As a result, for all initial conditions $\xi(0, 0)\in\mathcal{S}_{\mu}$, the control input, after at most one jump corresponds to the \textit{move-to-target} mode and the algorithm steers the state $\mathbf{x}$ towards the origin according to the control input vector $\mathbf{u}(\xi) = -\kappa_s\mathbf{x}, \kappa_s> 0$. Hence, for each $\mu\in(0, \min\{\mu_1, \mu_2, \mu_3\}),$ the set $\mathcal{S}_{\mu}=\mathcal{B}_{\mu}(\mathbf{0})\times\mathcal{V}_{r_a}\times\mathbb{M}$ is forward invariant for the hybrid closed-loop system \eqref{hybrid_closed_loop_system}.

Consequently, for every $\rho>0$, one can choose $\sigma\in(0, \min\{\mu_1, \mu_2, \mu_3, \rho\})$ such that for all initial conditions $\xi(0, 0)$ with $d(\xi(0, 0), \mathcal{A})\leq\sigma$, one has $d(\xi(t, j), \mathcal{A})\leq \rho$ for all $(t, j)\in\text{ dom }\xi,$ where $d(\xi, \mathcal{A})^2 = \underset{(\mathbf{0}, \bar{\mathbf{h}}, \bar{m})\in\mathcal{A}}{\inf}(\norm{\mathbf{x}}^2 + \norm{\mathbf{h} - \bar{\mathbf{h}}}^2 + \norm{m - \bar{m}}^2) = \norm{\mathbf{x}}^2.$ Additionally, the target set $\mathcal{A}$, as defined in \eqref{target_set}, is compact.
Hence, according to \cite[Definition 7.1]{goebel2012hybrid}, the target set $\mathcal{A}$ is stable for the hybrid closed-loop system \eqref{hybrid_closed_loop_system}

\textbf{Attractivity:} 
Since the target set $\mathcal{A}$ is obtained by only restricting the $\mathbf{x}$ component of the state space $\mathcal{K}$ to $\mathbf{0}$, one can prove claim 2 in Theorem \ref{theorem:global_stability} by showing that for all solutions $\xi$ to the hybrid closed-loop system \eqref{hybrid_closed_loop_system} initialized in the \textit{move-to-target} mode $(m(0, 0) = 0)$, the $\mathbf{x}$ component of the solution $\xi$ converges to the origin $\mathbf{0}.$
For the closed-loop system \eqref{hybrid_closed_loop_system}, let $\xi(t_0, j_0)\in\mathcal{F}_0$,  for some $(t_0, j_0)\in\text{dom }\xi$.
If $\xi(t, j)\notin\mathcal{J}_0,\forall(t, j)\succeq(t_0, j_0), $ then due to the stabilizing control of the form $\mathbf{u}(\mathbf{x}, \mathbf{h}, 0) = -\kappa_s\mathbf{x}, \kappa_s > 0,$ the state $\mathbf{x}$ will converge to the origin. Now, assume that there exists $(t_1, j_1)\succeq(t_0, j_0)$ such that $\xi(t_1, j_1)\in\mathcal{J}_0$, then the control law switches to the \textit{obstacle-avoidance} mode. According to \eqref{updatelaw_part1}, $\mathbf{h}(t_1, j_1 + 1) = \mathbf{x}(t_1, j_1)$ and $m(t_1, j_1 + 1) \in \{-1, 1\}$ \textit{i.e.,} the state $\mathbf{x}$ will evolve either in the clockwise direction $(m = 1)$ or in the counter-clockwise direction $(m = -1)$. To that end, we use the following lemma to show that if the solution $\xi$ to the hybrid closed-loop system \eqref{hybrid_closed_loop_system}, which has been initialized in the \textit{move-to-target} mode, evolves in the \textit{obstacle-avoidance} mode, then it will always enter in the jump set of that mode only through the set $\mathcal{ER}_m^{\mathbf{h}}\times\mathcal{V}_{r_a}\times\{-1, 1\}$, where, according to \eqref{partition_rm}, the location $\mathbf{x}$ is closer to the target location than the current \textit{hit point}.

\begin{lemma}
Let Assumptions \ref{assumption:connected_interior} and \ref{Assumption:reach} hold, and the parameter $\bar{\epsilon}$, used in \eqref{partition_rm}, is chosen as per Lemma \ref{lemma:bar_epsilon}, then for all solutions $\xi = (\mathbf{x}, \mathbf{h}, m)$ to the hybrid closed-loop system \eqref{hybrid_closed_loop_system} with $\xi(t, j -1)\in\mathcal{J}_0$, $\xi(t, j)\in\mathcal{F}_l$, for some $l\in\{-1, 1\}$ and $(t, j)\in\text{dom }\xi,$ there exists $(p, j)\in\text{dom }\xi$, $(p, j)\succ(t, j)$ such that following conditions hold true:
\begin{enumerate}
\item $\xi(p, j)\in\mathcal{ER}_l^{\mathbf{h}(t, j)}\times\mathcal{V}_{r_a}\times\{l\}\subset\mathcal{J}_l,$ 
\item $\xi(v, j)\in\mathcal{F}_l, \forall (v, j)\in([t, p)\times j)$.
\end{enumerate}
\label{lemma:always_enters_in_move_to_target}
\end{lemma}
\begin{proof}
See Appendix \ref{proof:always_movetotarget}.
\end{proof}

According to Lemma \ref{lemma:always_enters_in_move_to_target}, $\exists(t_2, j_1 + 1)\succeq(t_1, j_1 + 1)$ such that $\xi(t_2, j_1 + 1)\in\mathcal{ER}_{m(t_1, j_1 + 1)}^{\mathbf{h}(t_1, j_1 +1)}\times\mathcal{V}_{r_a}\times\{-1, 1\}\subset\mathcal{J}_{m(t_1, j_1 + 1)}.$ Notice that, according to \eqref{avoidance_jump_set} and \eqref{partition_rm}, $\|{\mathbf{x}(t_2, j_1 + 1)}\|< \|{\mathbf{h}(t_1, j_1 + 1)}\|$, where, according to \eqref{updatelaw_part1}, $\mathbf{h}(t_1, j_1 + 1) = \mathbf{x}(t_1, j_1)$. In other words, according to Lemma \ref{lemma:always_enters_in_move_to_target}, the proposed navigation algorithm ensures that, at the instance where the the control switches from the \textit{obstacle-avoidance} mode to the \textit{move-to-target} mode, the state $\mathbf{x}$ is closer to the origin than from the last point where the control switched to the \textit{obstacle-avoidance} mode. 
Furthermore, when the control input corresponds to the \textit{move-to-target} mode, the algorithm steers the state $\mathbf{x}$ towards the origin under the influence of the stabilizing control vector $\mathbf{u}(\xi) = -\kappa_s\mathbf{x}, \kappa_s > 0.$ Consequently, given that the workspace $\mathcal{W}$ and the obstacles $\mathcal{O}_i, i\in\mathbb{I}\setminus\{0\}$, are  compact, it can be concluded that the solution $\xi$ will contain finite number of jumps and the state $\mathbf{x}$ will converge to the origin. 

\subsection{Proof of Lemma \ref{lemma:always_enters_in_move_to_target}}
\label{proof:always_movetotarget}

We consider a connected modified obstacle $\mathcal{O}_{i, \alpha}^M\subset\mathcal{O}_{\mathcal{W}}^M$, for some $i\in\mathbb{I}$, as stated in Lemma \ref{lemma:pathwise_connected}, where $\alpha$ is selected according to Lemma \ref{lemma:alpha_existence}. Let $\xi(t, j)\in\mathcal{F}_l^{\mathcal{W}}\cap\mathcal{N}_{\gamma}(\mathcal{D}_{r_a}(\mathcal{O}_{i, \alpha}^M))\times\mathcal{V}_{r_a}\times\{l\}$, for some $l\in\{-1, 1\}$. Since $m(t, j-1) = 0$ and $m(t, j) \in\{-1, 1\}$, according to \eqref{stabilization_jump_set}, $\mathbf{x}(t, j) = \mathbf{h}(t, j)\in\partial\mathcal{D}_{\beta}(\mathcal{O}_{i, \alpha}^M)\cap\mathcal{J}_0^{\mathcal{W}},$ where $\beta\in[r_a, r_a + \gamma_s]$. To proceed with the proof, we need the following fact:

\textbf{Fact 3:}
Under Assumptions \ref{assumption:connected_interior} and \ref{Assumption:reach}, consider the following flow-only system:
    \begin{equation}
        \mathbf{\dot{x}} = \kappa_r\mathbf{v}(\mathbf{x}, m), \mathbf{x}\in\mathcal{N}_{\gamma}(\mathcal{D}_{r_a}(\mathcal{O}_{i, \alpha}^M)),\label{flow_only_system}
    \end{equation}
    for some $m\in\{-1, 1\},$ where $\kappa_r > 0$ and $\mathbf{v}(\mathbf{x}, m)$ is given by \eqref{definition:vxm}. If $\mathbf{x}(t_0)\in\partial\mathcal{D}_{\beta_1}(\mathcal{O}_{i, \alpha}^M), $ $\beta_1\in[r_a, r_a + \gamma]$, for some $t_0\geq 0$, then $\mathbf{x}(t)\in\partial\mathcal{D}_{\beta_1}(\mathcal{O}_{i, \alpha}^M)$ for all $t\geq t_0.$\label{lemma:rotational}
    
    \begin{proof}
According to Lemma \ref{lemma:unique_projection}, for $\mathbf{x}\in\partial\mathcal{D}_{\beta_1}(\mathcal{O}_{i, \alpha}^M), \beta_1\in[r_a, r_a + \gamma]$, $\mathbf{card}(\mathcal{PJ}(\mathbf{x}, \mathcal{O}_{i, \alpha}^M)) = 1$. Hence, according to \cite[Lemma 4.5]{rataj2019curvature}, the vector $\mathbf{x} - \Pi(\mathbf{x}, \mathcal{O}_{i, \alpha}^M)$ is normal to the set $\partial\mathcal{D}_{\beta}(\mathcal{O}_{i, \alpha}^M)$ at $\mathbf{x}.$ As a result, the tangent cone to the set $\partial\mathcal{D}_{\beta_1}(\mathcal{O}_{i, \alpha}^M)$ at $\mathbf{x}\in\partial\mathcal{D}_{\beta_1}(\mathcal{O}_{i, \alpha}^M)$ is given by
\begin{equation}
    \mathbf{T}_{\partial\mathcal{D}_{\beta_1}(\mathcal{O}_{i ,\alpha}^M)}(\mathbf{x}) = \mathcal{P}(\mathbf{0}, \mathbf{x} - \Pi(\mathbf{x}, \mathcal{O}_{i, \alpha}^M)).
\end{equation}
Since $\mathbf{v}(\mathbf{x}, m)^{\intercal}(\mathbf{x} - \Pi(\mathbf{x}, \mathcal{O}_{i, \alpha}^M)) = 0$, $\mathbf{v}(\mathbf{x}(t), m)\in\mathbf{T}_{\partial\mathcal{D}_{\beta_1}(\mathcal{O}_{i, \alpha}^M)}(\mathbf{x}(t)),$ for all $t\geq t_0,$ which implies that the solution $\mathbf{x}(t)$ to the flow-only system \eqref{flow_only_system} belongs to the set $\partial\mathcal{D}_{\beta_1}(\mathcal{O}_{i, \alpha}^M)$ for all $t \geq t_0$.
    \end{proof}

According to Fact 3, if $\mathbf{x}$ belongs to the set $\mathcal{N}_{\gamma}(\mathcal{D}_{r_a}(\mathcal{O}_{i, \alpha}^M))$ at some time $t_0 \geq 0$, and is solely influenced by the obstacle-avoidance control vector $\kappa_r\mathbf{v}(\mathbf{x}, m)$, $\kappa_r > 0$, for some $m \in \{-1, 1\}$, then $\mathbf{x}(t)$ will continue to evolve within the set $\mathcal{N}_{\gamma}(\mathcal{D}_{r_a}(\mathcal{O}_{i, \alpha}^M))$ for all $t \geq t_0$, without changing its distance from the set $\mathcal{O}_{i, \alpha}^M$.

According to \eqref{hybrid_control_input_1}, for $l\in\{-1, 1\},$ $\mathbf{u}(\mathbf{x}, \mathbf{h}, l) = \kappa_r\mathbf{v}(\mathbf{x}, l), \kappa_r > 0$. Since $\xi(t, j)\in\mathcal{F}_l,$ for some $l\in\{-1, 1\}$, with $\mathbf{x}(t, j)\in\partial\mathcal{D}_{\beta}(\mathcal{O}_{i, \alpha}^M), \beta\in[r_a, r_a + \gamma_s],$ according to Fact 3, the state $\mathbf{x}$ will evolve along the curve $\partial\mathcal{D}_{\beta}(\mathcal{O}_{i, \alpha}^M)$. Since the workspace $\mathcal{W}$ and the obstacles $\mathcal{O}_i, i\in\mathbb{I}\setminus\{0\},$  $\exists(p, j)\in\text{dom } \xi$, $(p, j)\succ(t, j)$ such that at $(p, j)$, the $\mathbf{x}$ component of the solution $\xi$ will enter the set $\mathcal{H}_{\mathbf{p}} = \mathcal{B}_{\delta}(\mathbf{p})\cap\mathcal{N}_{\gamma}(\mathcal{D}_{r_a}(\mathcal{O}_{i, \alpha}^M))$, where $\delta = \min\{\beta - r_a, d(\mathbf{0}, \mathcal{O}_{\mathcal{W}}^M) - r_a\}$. Since the parameter $\bar{\epsilon}\in(0, \epsilon_h]$, where $\epsilon_h$ is defined in \eqref{epsilon_h}, according to Lemma \ref{lemma:bar_epsilon}, $\mathbf{x}(p, j)\in\mathcal{ER}_{l}^{\mathbf{h}(t, j)}$. Hence, according to \eqref{avoidance_jump_set} and \eqref{bothmodes_flowjumpset}, $\xi(p, t) \in\mathcal{ER}_l^{\mathbf{h}}\times\mathcal{V}_{r_a}\times\{l\}\subset\mathcal{J}_l,$ and condition 1 in Lemma \ref{lemma:always_enters_in_move_to_target} holds true.

Since $\xi(t, j)\in\mathcal{F}_l$, for some $l\in\{-1, 1\},$ with $\mathbf{x}(t, j)\in\partial\mathcal{D}_{\beta}(\mathcal{O}_{i, \alpha}^M), \beta\in[r_a, r_a + \gamma_s],$ according to Fact 3, the state $\mathbf{x}$ evolves in the region $\mathcal{N}_{\gamma}(\mathcal{D}_{r_a}(\mathcal{O}_{i, \alpha}^M))$, along the curve $\partial\mathcal{D}_{r_a + \beta}(\mathcal{O}_{i, \alpha}^M),$ until it enters the set $\mathcal{ER}_l^{\mathbf{h}}$, which guarantees that condition 2 in Lemma \ref{lemma:always_enters_in_move_to_target} holds true.

\subsection{Proof of Proposition \ref{lemma:assumption_equivalence}}
\label{proof:proposition}

Since the obstacles $\mathcal{O}_i, i\in\mathbb{I}\setminus\{0\}$, are convex and compact, and $d(\mathcal{O}_i, \mathcal{O}_j)>2r_a, \forall i, j\in\mathbb{I}, i\ne j$, the $r_a-$dilated versions of the obstacles do not intersect with each other \textit{i.e.}, $d(\mathcal{D}_{r_a}(\mathcal{O}_i), \mathcal{D}_{r_a}(\mathcal{O}_j)) > 0,\forall i, j\in\mathbb{I}, i\ne j$. This ensures the pathwise connectedness of the interior of the obstacle-free workspace $\mathcal{W}_{r_a}^{\circ}$. In addition, since $\mathbf{0}\in\mathcal{W}_{r_a}^{\circ},$ Assumption \ref{assumption:connected_interior} is satisfied. 

Define $\bar{\alpha}  = \underset{i, j\in\mathbb{I}, i\ne j}{\min}d(\mathcal{O}_i, \mathcal{O}_j)/2$. Since the pairwise distance between any two obstacles is greater than $2r_a$, it follows that $\bar{\alpha} > r_a$, and for any $\alpha\in(r_a, \bar{\alpha}]$, the $\alpha-$dilated versions of the obstacles do not intersect with each other \textit{i.e.}, $d(\mathcal{D}_{\alpha}(\mathcal{O}_i), \mathcal{D}_{\alpha}(\mathcal{O}_j)) > 0$. Therefore, $\mathcal{W}_{\alpha}$, defined using \eqref{y_free_workspace}, is not an empty set.
Then, it is straightforward to see that for every $\mathbf{x}\in\mathcal{D}_{\alpha}(\mathcal{O}_{\mathcal{W}})\setminus\mathcal{O}_{\mathcal{W}}^{\circ},$ there is a unique closest point on the boundary of the set $\mathcal{D}_{\alpha}(\mathcal{O}_{\mathcal{W}})$ from $\mathbf{x}$. 
Hence, $\mathbf{reach}(\mathcal{W}_{\alpha}) \geq \alpha$, where $\alpha\in(r_a, \bar{\alpha}]$. Finally, by virtue of \cite[Lemma 4.5]{rataj2019curvature}, Assumption \ref{Assumption:reach} is satisfied.

\end{appendix}
\bibliographystyle{IEEEtran}
\bibliography{reference}

\end{document}